%% file: ICML23_new_score_potential.tex
\documentclass{article}

\usepackage{microtype}
\usepackage{graphicx}
\usepackage{subfigure}
\usepackage{booktabs} % for professional tables

\input{package}
\newcommand{\bM}{{\boldsymbol M}}

\usepackage{tikz}
\usepackage{tikz-qtree}
\usetikzlibrary{shadows,trees}

\usepackage[accepted]{icml2023}

\usepackage[textsize=tiny]{todonotes}

\icmltitlerunning{Provably Convergent  Schr\"{o}dinger Bridge with Applications to Probabilistic Time Series Imputation}

\begin{document}

\twocolumn[
% \icmltitle{A Conditional Schr\"{o}dinger Bridge Method for Probabilistic Time Series Imputation}
\icmltitle{Provably Convergent  Schr\"{o}dinger Bridge with Applications to Probabilistic Time Series Imputation}

% It is OKAY to include author information, even for blind
% submissions: the style file will automatically remove it for you
% unless you've provided the [accepted] option to the icml2023
% package.

% List of affiliations: The first argument should be a (short)
% identifier you will use later to specify author affiliations
% Academic affiliations should list Department, University, City, Region, Country
% Industry affiliations should list Company, City, Region, Country

% You can specify symbols, otherwise they are numbered in order.
% Ideally, you should not use this facility. Affiliations will be numbered
% in order of appearance and this is the preferred way.
\icmlsetsymbol{equal}{*}

\begin{icmlauthorlist}
\icmlauthor{Yu Chen}{equal,msml}
\icmlauthor{Wei Deng}{equal,msml}
\icmlauthor{Shikai Fang}{equal,utah}
\icmlauthor{Fengpei Li}{equal,msml} 
\icmlauthor{Nicole Tianjiao Yang}{emory} \\
\icmlauthor{Yikai Zhang}{msml}
\icmlauthor{Kashif Rasul}{msml}
\icmlauthor{Shandian Zhe}{utah}
\icmlauthor{Anderson Schneider}{msml}
\icmlauthor{Yuriy Nevmyvaka}{msml}
\end{icmlauthorlist}

\icmlaffiliation{utah}{School of Computing, University of Utah (Fang completed part of the work while interning at Morgan Stanley)}
\icmlaffiliation{msml}{Machine Learning Research, Morgan Stanley, NY}
\icmlaffiliation{emory}{Department of Mathematics, Emory University}

\icmlcorrespondingauthor{Wei Deng}{weideng056@gmail.com}
 
\icmlkeywords{Machine Learning, ICML}

\vskip 0.3in
]

\printAffiliationsAndNotice{\icmlEqualContribution} %

\begin{abstract}
The Schr\"{o}dinger bridge problem (SBP) is gaining increasing attention in generative modeling and showing promising potential even in comparison with the score-based generative models (SGMs). SBP can be interpreted as an entropy-regularized optimal transport problem, which conducts projections onto every other marginal alternatingly. However, in practice, only approximated projections are accessible and their convergence is not well understood. To fill this gap,
we present a first convergence analysis of the Schr\"{o}dinger bridge algorithm based on approximated projections. As for its practical applications, we apply SBP to probabilistic time series imputation by generating missing values conditioned on observed data. We show that optimizing the transport cost improves the performance and the proposed algorithm achieves the state-of-the-art result in healthcare and environmental data while exhibiting the advantage of exploring both temporal and feature patterns in probabilistic time series imputation.

\end{abstract}

\section{Introduction}

Time series data is extensively studied in various fields such as finance \citep{Large_Dimensional_Latent_Factor}, healthcare \citep{phy_2012}, and meteorology. However, incomplete or partial observations, equipment failures, and human errors may inevitably lead to the missing value problem, severely limiting the interpretation of the time series. For instance, the inherent illiquidity of certain assets can result in the occurrence of missing values, which in turn impacts our ability to devise reliable trading strategies \citep{Illiquidity_Premia}. In populations of Intensive Care Units (ICUs), predicting mortality rates based on time-series observations of vital signs is essential \citep{phy_2012}. However, the presence of missing data has greatly limited the efficacy of medications and surgical treatments.

One standard approach to tackle such a problem is to leverage score-based generative models (SGMs) \citep{nonequilibrium_thermodynamics_15, DDPM, song_likelihood_training, score_sde}, which propose to recover the data distribution through a backward process that estimates the scores of posterior distributions conditioned on the observed data.  This conditional nature directly motivates the study of Conditional Score-based Diffusion models for Imputation (CSDI) \citep{CSDI}. CSDI is able to learn the temporal-feature patterns well and achieves state-of-the-art performance in probabilistic time series imputation. However, transporting between terminal distributions is often quite expensive for SGMs and CSDI, which requires extensive computations and hyperparameter tuning. As such, a more efficient algorithm is needed to reduce the transport cost.

The Schr\"{o}dinger bridge problem (SBP) was initially proposed to solve problems in quantum mechanics and can be transformed into the entropy-regularized optimal transport (EOT) \citep{DSB, leonard_14, chen_21, Nutz_22_a}. Solving the EOT formulation gives rise to the iterative proportional fitting (IPF) algorithm \citep{Kullback_68, IPF_95}, which provides a principled paradigm to minimize the transport cost and facilitates the estimation of score functions to generate samples of higher quality; % and mitigates the expensive transport costs using a much shorter time compared to SGMs; 
SBP further enables the generalization of linear Gaussian priors to non-linear families with more acceleration potential \citep{Chen21, Gaussian_SB, Pavon_CPAM_21, CSGLD, AWSGLD}. 

Despite the theoretical potential, the existing SBP-based generative models assume we can obtain the exact projections for the IPF algorithm, but in practice, it is often only approximated by deep neural networks \citep{DSB} or Gaussian processes \citep{SBP_max_llk}. In order to fill the gap, we extend the IPF algorithm by allowing for the approximated projections and refer to it as the approximate IPF (aIPF) algorithm; we further conduct theoretical analysis for aIPF based on the optimal transport theory, which deepens the understanding of training budgets in score approximations. Empirically, we apply the SBP-based generative models to probabilistic
time series imputation and demonstrate that minimizing the transport cost improves performance. We summarize our contributions as follows:
\begin{itemize}
    \item We show a first convergence analysis for Schr\"{o}dinger bridge with approximated projections and characterize the relation between training errors and the number of iterations. Our theory motivates future research for devising provably convergent Schr\"{o}dinger bridge (SB) algorithms and paves the way for understanding when SB is faster than SGMs. To bridge the gap between theoretical understanding and practical algorithms, we also draw connections between the aIPF algorithm and the divergence-based likelihood training of forward and backward stochastic differential equations (FB-SDEs).
    \item We apply the Schr\"{o}dinger bridge algorithm to probabilistic time series imputation. We show that optimizing the transport cost visibly improves the performance on synthetic data and achieves the state-of-the-art performance on real-world datasets. 
\end{itemize}

\section{Related Work}

\paragraph{Schrödinger Bridge}
Schr\"{o}dinger bridge problem (SBP) is known for the quantum mechanics formulation and is closely related to stochastic optimal control (SOC) \citep{Chen16, Pavon_CPAM_21, Caluya21} and optimal transport \citep{Compute_OT}. Recent works leverage SBP for generative modeling \citep{DSB, gefei_21} and explore theoretical properties \citep{Nutz_22_a, Nutz_22_func, khrulkov2022understanding, Lavenant_Santambrogio_22}. \citet{Conditional_DSB} apply the amortized formulation to model the conditional SBP for images and state-space models. \citet{forward_backward_SDE} propose likelihood training for SBP approximation based on divergence objectives \citep{Hutchinson89, FFJORD} and forward-backward stochastic differential equations (FB-SDEs) \citep{Ma_FB_SDE}; similar results are shown in \citet{SBP_max_llk}.

\paragraph{Time-series Imputations via Generative Models }

% \Wei{be careful with irrelevant references}
Multivariate time series imputation is challenging because of the temporal-feature dependencies and the irregular locations of missing values. To handle these issues, several recent works use deep generative learning and conditional sampling to achieve competitive performance. Generative techniques in these works include Gaussian processes \citep{multitask_GP}, VAEs \citep{GP_VAE}, neural ODEs \citep{latent_ode, NIPS2020_Emmanuel, control_neural_ode}, neural SDEs \cite{scalable_SDE, Continuous_Latent_Flows}, and GANs \citep{luo2019e2gan}.
Imputation methods based on recurrent networks or attention networks can be found in \citep{che2018recurrent, BRITS, attention_ts}. Curvature flow methods are potentially applicable \citep{Malladi_Sethian}.

\section{Preliminaries}

\subsection{Likelihood Training of SGMs} \label{sec:likelihood_training}
The score-based generative models (SGMs \citep{score_sde}) have become the go-to framework for %simulations tasks of 
generative models. SGMs first inject noise into the data and then recover it from a backward process \citep{Anderson82}
\begin{subequations}
\begin{align}
\mathrm{d} \bx_t &= \bbf(\bx_t, t) \mathrm{d} t+g(t) \mathrm{d} \bm{\mathrm{w}}_t, \label{SGM-SDE-f}\\
    % \quad \bx_0 \sim p_{\text {data }}
\mathrm{d} \bx_t&=\small{\left[\bbf(\bx_t, t)-g(t)^2 \nabla \log p_t\left(\bx_t\right)\right] \mathrm{d} t+g(t) \mathrm{d} \bar{\bm{\mathrm{w}}}_t}, \label{SGM-SDE-b}
    % &\qquad \bx_T \sim p_T \approx p_{\text{prior-} \text{target}}  p_{ \text{cond}} \notag.
\end{align}
\end{subequations}

where $\{\bx_t\}_{t=0}^{T} \in \mathbb{R}^d$ \footnote[4]{$d$ is the data dimension and can be reshaped to other formats.}, $\bx_0\sim p_{\text {data }}$, and $\bx_T\sim p_{\text {prior}}$; 
$\bbf\equiv \bbf\left(\bx_t, t\right)$ is the vector field; $g\equiv g(t)$ is the diffusion term; $\bm{\mathrm{w}}_t$ is the standard Brownian motion; $\bar{\bm{\mathrm{w}}}_t$ is a Brownian motion with time moving backward from $T$ to $0$; $p_t$ is the marginal density of the forward process \eqref{SGM-SDE-f} at time $t$. The score function  $\nabla \log p_t\left(\cdot\right)$ is approximated via a model $s_{\theta}(\cdot, t)$; $p_{\text {data }}$ is simulated via the backward process \eqref{SGM-SDE-b} starting at $\bx_T$. SGMs \citep{DDPM} aim to train $s_{\theta}(\cdot, t)$ by minimizing the mean squared error between the ground-truth score and estimator such that $\mathbb{E} \big[\lambda(t)\|s_{\theta}(\bx_t, t) - \nabla \log p_t\left(\bx_t \mid \bx_0\right)\|_2^2\big]$, where the weight $\lambda(t)$ is set manually. \citet{song_likelihood_training} proposes to maximize the likelihood to learn $s_{\theta}(\bx,t)$ such that 
\begin{align*}
\footnotesize
    &\log p_0^{\mathrm{SDE}}\left(\bx_0\right)  \geq  
    \mathbb{E}_{p_{0T}(\cdot|\bx_0))}\left[\log p_T\left(\bx_T\right)\right]\\
    &-\frac{1}{2} \int_0^T \mathbb{E}_{p_{0t}(\cdot|\bx_0)}\left[g^2\left\|\mathbf{s}_t\right\|_2^2+2\nabla \cdot\left(g^2 \mathbf{s}_t - \bbf \right)\right] \mathrm{d} t, \notag
\end{align*}
where ${s}_t = s_{\theta}(\bx_t, t)$ ,  $p_{0t}(\cdot|\bx_0) = p_{0t}(\bx_{t}|\bx_{0})$ stands for the conditional density of $\bx_{t}$, which evolves with the trajectory of \eqref{SGM-SDE-f}. The 
inequality becomes an equality if the estimator ${s}_t$ exactly matches the score function. Thus, optimizing the lower bound provides an efficient scheme to maximize the data likelihood.

\subsection{Schr\"{o}dinger Bridge Problem}

Even though SGMs have demonstrated success in generative models, they still suffer from transport inefficiency. A \emph{long evolving time} $T$ of the forward process \eqref{SGM-SDE-f} is required to facilitate the score estimation and guarantee that $\bx_t$ will converge close to a prior distribution. Besides, the choice of priors is repeatedly \emph{constrained to Gaussian} and further limits the acceleration potential. To tackle this issue, the dynamical Schr\"{o}dinger Bridge problem (SBP) aims to solve 
\begin{align}\label{dynamic_SBP}
    &\inf_{\mathbb{P}\in \mathcal{D}(\mu_{\star}, \nu_{\star})}\text{KL}(\mathbb{P}|\mathbb{Q}), 
\end{align}
where $\mathcal{D}(\mu_{\star}, \nu_{\star})$ denotes the space of \emph{path measures} with marginal probability measures $\mu_{\star}$ and $\nu_{\star}$ at time $t=0$ and $t=T$, respectively; $\mathbb{Q}$ is the prior measure, usually induced by Brownian motion or Ornstein-Uhlenbeck process; $\text{KL}(\cdot|\mathbb{Q})$ denotes the KL divergence with respect to the measure $\mathbb{Q}$.

The dynamical SBP can be interpreted from stochastic optimal control (SOC) (see section 4.4 in \citet{Chen21})
\begin{align}
\tiny
    &\inf_{\bu\in \mathcal{U}} \E\bigg\{\int_0^T \frac{1}{2}\|\bu(\bx_t,t)\|^2_2 \mathrm{d} t \bigg\} \label{SBP_classical_main}\\
    \text{s.t.} &\  \footnotesize{\mathrm{d} \bx_t=\left[\bbf(\bx_t, t)+g(t)\bu(\bx_t,t)\right]\mathrm{d} t+ \sqrt{2\varepsilon}g(t) \mathrm{d} \bm{\mathrm{w}}_t} \notag\\
    &\ \ \ \ \bx_0\sim \mu_{\star}(\cdot),\ \  \bx_T\sim \nu_{\star}(\cdot)\notag
    ,
\end{align}
where $\mathcal{U}$ is the control set $\bu:\mathbb{R}^d\times [0,T]\rightarrow \mathbb{R}^d$; 
the state-space is $\mathbb{R}^d$ and is sometimes omitted; the expectation is taken w.r.t the joint state PDF $\rho(\bx, t)$; $\varepsilon$ is a regularizer.

\section{Provably Convergent Schr\"{o}dinger Bridge}

Diffusion models have shown superiority in generative models and time series imputation, which motivate interesting theoretical works \citep{lee2022convergence, Sitan_22_sampling_is_easy, DSB, stat_efficiency_SGM}. As a theoretical ideal candidate, Schr\"{o}dinger bridge also has gained tremendous attention \citep{DSB, SBP_max_llk, gefei_21, forward_backward_SDE}, however, the practical theory has not been studied in the literature. 

To bridge the gap between theory and practice, we initiate the convergence study of the practical Schr\"{o}dinger bridge algorithm based on general cost functions and highlight the connections between SBP, EOT, and FB-SDEs.

\subsection{Schr\"{o}dinger Bridge: from Dynamic to Static}
\label{static_SBP}

By applying the disintegration of measures \citep{leonard_14_chain_rule}, the chain rule \citep{DSB} for the KL divergence for the dynamical SBP \eqref{dynamic_SBP} follows
\begin{equation*}
\begin{split}
\small
    \text{KL}(\mathbb{P}|\mathbb{Q})=\text{KL}(\pi|\mathcal{G})+\iint \text{KL}(\mathbb{P}^{ \bx_T}_{\bx_0}|\mathbb{Q}^{\bx_T}_{\bx_0})\dd\pi(\bx_0, \bx_T).
\end{split}
\end{equation*}
where $\pi:=(\mu_{\star},\nu_{\star})$ is a coupling with marginals $\mu_{\star}$ and $\nu_{\star}$; $\mathcal{G}$ is a Gibbs measure: $\mathrm{d}\mathcal{G} \propto e^{-c_{\varepsilon}}\mathrm{d} (\mu_{\star} \otimes \nu_{\star})$; $c_{\varepsilon}$ is a cost function in Eq.\eqref{def_cost_varphi_psi}; $\otimes$ is the product measure; the marginals of $\mathbb{P}$ (or $\mathbb{Q}$) at $t=0$ and $T$ follow from $\mu_{\star}$ and $\nu_{\star}$; $\mathbb{P}^{ \bm{x}_T}_{\bm{x}_0}:=\mathbb{P}(\cdot|\bx_0=\bm{x}_0, \bx_T=\bm{x}_T)$ (or $\mathbb{Q}^{ \bm{x}_T}_{\bm{x}_0}$) denotes a diffusion bridge of $\mathbb{P}$ (or $\mathbb{Q}$) from $\bm{x}_0$ to $\bm{x}_T$. 

Assuming the same bridges for $\mathbb{P}$ and $\mathbb{Q}$, the \emph{static} SBP yields a coupling $\pi_{\star}$ (see Lemma \ref{solution_property} in Appendix \ref{static_SB_property}):
\begin{equation}\label{static_SB}
    \pi_{\star}=\argmin_{\pi\in \Pi(\mu_{\star}, \nu_{\star})} \text{KL}(\pi|\mathcal{G}),
\end{equation}
where $\Pi(\mu_{\star}, \nu_{\star})$ is the set of couplings with marginals $\mu_{\star}$ and $\nu_{\star}$. Moreover, the static SBP yields a structural representation for \eqref{static_SB} \citep{Compute_OT, Nutz22_note}:
\begin{equation*}\label{main_solution_property}
    \mathrm{d}\pi_{\star}(\bx, \by)=e^{{\varphi_{\star}}(\bx) +  \psi_{\star}(\by)-c_{\varepsilon}(\bx, \by)}\mathrm{d} (\mu_{\star} \otimes \nu_{\star}),
\end{equation*}
where ${\varphi_{\star}}$ and $ \psi_{\star}$ are the Schr\"{o}dinger potential functions.

\subsection{From Static SBP to Entropic Optimal Transport}
\label{SBP_EOT}

Next, the equivalence between the static SBP and entropic optimal transport (EOT) follows that:
\begin{equation*}
\begin{split}
\footnotesize
    \text{KL}(\pi|\mathcal{G})&=\iint\log\left(\frac{d\pi}{d (\mu_{\star} \otimes \nu_{\star})} \frac{d (\mu_{\star} \otimes \nu_{\star})}{d \mathcal{G}}\right)d \pi \\
    &\dot{=}\text{KL}(\pi|\mu_{\star} \otimes \nu_{\star}) + \iint  \log  e^{c_{\varepsilon}}d\pi\\
    &=\iint  c_{\varepsilon} d\pi + \text{KL}(\pi|\mu_{\star} \otimes \nu_{\star}), 
\end{split}
\end{equation*}
where $\dot{=}$ denotes an equality that's up to a constant. Problem \eqref{static_SB} is equivalent to the EOT with a 1-regularizer: 
\begin{equation}\label{EOT_problem_supp}
\footnotesize
    \inf_{\pi\in \Pi(\mu_{\star}, \nu_{\star})} \iint  c_{\varepsilon}(\bx, \by)\pi(\mathrm{d}\bx, \mathrm{d}\by) + 1\cdot \text{KL}(\pi|\mu_{\star} \otimes \nu_{\star}).
\end{equation}

\subsection{Approximate Iterative Proportional Fitting (aIPF)}

Recall that the first and second marginal of the coupling $\pi_{\star}$ follow from $\mu_{\star}$ and $\nu_{\star}$, respectively. As detailed in section \ref{derive_schrodinger_eqn}, we can arrive at \emph{Schr\"{o}dinger equations}
\begin{align}
    \int  e^{{\varphi_{\star
    }}(\bx)+ \psi_{\star
    }(\by)-c_{\varepsilon}(\bx,\by)}\mu_{\star} (\mathrm{d}\bx)&=1 \quad \nu_{\star}\text{-}a.s.\label{SE1}\\
    \int  e^{{\varphi_{\star
    }}(\bx)+ \psi_{\star
    }(\by)-c_{\varepsilon}(\bx,\by)}\nu_{\star} (\mathrm{d}\by)&=1 \quad \mu_{\star}\text{-}a.s..\label{SE2}
\end{align}
Notably, the score functions also give rise to a variant of \emph{Schr\"{o}dinger equations}, as shown in Eq.\eqref{schrodinger_eqn_supp}, establishing an inherent link between scoring functions and Schr\"odinger potentials.

% In view of Eq.\eqref{marginal_projections} in Appendix \ref{score_potential_connection}, the above implies that 
% \begin{align*}
%     \text{Eq.}(\ref{SE1})&\Longleftrightarrow \text{the first marginal of $\pi_{\star}$ is $\mu_{\star}$,}\\
%     \text{Eq.}(\ref{SE2})&\Longleftrightarrow \text{the second marginal of $\pi_{\star}$ is $\nu_{\star}$.}
% \end{align*}

To obtain the desired $\pi_{\star}:=(\mu_{\star},\nu_{\star})$, a standard tool is the iterative proportional fitting (IPF) (also known as Sinkhorn algorithm) \citep{IPF_95}. The exact IPF algorithm alternatingly projects the coupling $\pi_k:=(\mu_k, \nu_k)$ at iteration $k$ to every other marginal such that for any $k \in \mathbb{N}$:
$$\mu_{2k+1}= \mu_{\star}, \ \ \nu_{2k}= \nu_{\star}.$$

To wit, we solve every other marginal alternatingly and show the convergence of the marginals to the correct distribution
\begin{equation*}
\footnotesize
    \mu_{2k} \xRightarrow{\text{convergence}}    \mu_{\star}, \ \   \nu_{2k+1} \xRightarrow{\text{convergence}}    \nu_{\star}.
\end{equation*}

However, it is too expensive in practice to obtain the exact marginals $\mu_{\star}$ and $\nu_{\star}$ via Eq.\eqref{SE1} and \eqref{SE2}. %in Eq.\eqref{psi_update} and \eqref{varphi_update}. 
To solve this problem, it is inevitable to approximate the projections (numerically solved via FB-SDE in Eq.\eqref{bs_sde_formulation}) through specific tools, such as deep neural networks \citep{DSB, forward_backward_SDE} or Gaussian process \citep{SBP_max_llk} 
\begin{equation}
\begin{split}
\label{marginal_eqn}
    \mu_{2k+1}&=\mu_{\star, k+1}\approx\mu_{\star}, \quad \nu_{2k}\ \ \ = \nu_{\star, k}\approx\nu_{\star},\\
\end{split}
\end{equation}
where $\mu_{\star, k+1}$ (or $\nu_{\star, k}$) is an approximate measure at iteration $k+1$ (or $k$) that is close to $\mu_{\star}$ (or $\nu_{\star}$). The approximate IPF (aIPF) is presented in Algorithm \ref{sinkhorn} and the comparison to the exact IPF is illustrated in Figure.\ref{fig:sinkhorn}.

\begin{figure}
\centering
\includegraphics[width=70mm]{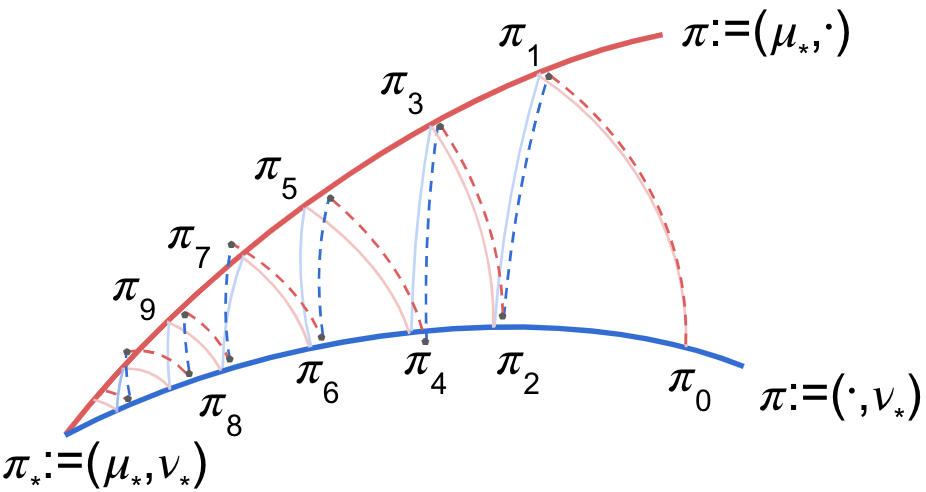}
% \includegraphics[width=70mm]{figures/sinkhorn_green.png}
% \vspace{-0.05in}
\caption{IPF vs aIPF. The light solid lines (IPF) show the iterates of exact projections; the dotted lines (aIPF) present the approximate projections towards $\pi_{\star}:=(\mu_{\star}, \nu_{\star})$.}
% \vspace{-0.05in}
\label{fig:sinkhorn}
\end{figure}

\begin{algorithm}[ht]\caption{One iteration of approximate IPF (aIPF). $ \psi_k$ and $\varphi_k$ denote the estimates of potential functions at iteration $k$. Empirically, the integral is estimated via divergence-based likelihood training of FB-SDEs in section \ref{Sec:SBP-llk}. }\label{sinkhorn}
\begin{align*}
     \psi_k(\by)&\approx-\log \int_{\mathbb{R}^d} e^{ \varphi_k(\bx)-c_{\varepsilon}(\bx,\by)} \mu_{\star}(\mathrm{d}\bx)\\
     \varphi_{k+1}(\bx)&\approx -\log \int_{\mathbb{R}^d} e^{ \psi_k(\by)-c_{\varepsilon}(\bx,\by)} \nu_{\star}(\mathrm{d}\by).
\end{align*}
\end{algorithm}

\subsection{Convergence Analysis}

For the convergence study, the existing (heuristic) finite sample bound  \citep{SBP_max_llk} conjectured that 
\begin{equation*}
\begin{split}
\small
    \|\mu_{2k}-\mu_{\star}\| &\lesssim  \frac{1}{k} +\sum_{i=0}^k\epsilon K^i\\
    \|\nu_{2k-1}-\nu_{\star}\| &\lesssim  \frac{1}{k} +\sum_{i=0}^k\epsilon K^i,
\end{split}
\end{equation*}
where $\|\cdot\|$ is some metric and the aIPF projection operator is assumed to be $K$-Lipchitz. However, ensuring $K<1$ based on general cost functions is not trivial \citep{sinkhorn_exp_general}. As such, a \emph{fundamental question} remains open 
\begin{center}
    {\it \textcolor{darkblue}{Can we ensure the approximation error is less dependent on the number of iterations $k$?}}
\end{center}
Answering this question yields concrete guidance on the computational complexity and tells us when the approximation error doesn't get arbitrarily worse. To achieve this target,  we first lay out the following standard assumptions.

\begin{assump}[Dissipativity]\label{Dissipativity}
$\mu_{\star}$ and $\nu_{\star}$ satisfy the dissipative condition for some constants $m_{\text{ds}}>0$ and $b_{\text{ds}}\geq 0$.
\begin{equation*}\label{dissipative2}
\begin{split}
    &\bigg\langle \bx, -\nabla \log \frac{\mathrm{d} \mu_{\star}}{\mathrm{d}\bx}(\bx)\bigg\rangle\geq m_{\text{ds}}\|\bx\|_2^2 -b_{\text{ds}}\\
    &\bigg\langle \by, -\nabla \log \frac{\mathrm{d} \nu_{\star}}{\mathrm{d}\by}(\by)\bigg\rangle\geq m_{\text{ds}}\|\by\|_2^2 -b_{\text{ds}},
\end{split}
\end{equation*}
where $\frac{\mathrm{d} \mu_{\star}}{\mathrm{d}\bx}$ and $\frac{\mathrm{d} \nu_{\star}}{\mathrm{d}\by}$ are the probability densities for the probability measure $\mu_{\star}$ and $\nu_{\star}$, respectively; $\nabla \log \frac{\mathrm{d} \mu_{\star}}{\mathrm{d}\bx}(\bx)$ and $\nabla \log \frac{\mathrm{d} \nu_{\star}}{\mathrm{d}\by}(\by)$ are the score functions. 
\end{assump}

\emph{Remark:} The above assumption is standard and has been used in \citet{Maxim17}, which allows the densities to be non-convex in a ball with a radius depending on $b_{\text{ds}}$.  
Notably, it also implies the log-Sobolev inequality (LSI) with a bounded constant $C_{LS}$ \citep{lee2022convergence}.
 
\begin{assump}[$\epsilon$-approximation]\label{approx_score}
$\nabla \log \frac{\mathrm{d}\mu_{\star, k}}{\mathrm{d}\bx}(\bx)$ and $\nabla \log \frac{\mathrm{d}\nu_{\star, k}}{\mathrm{d}\by}(\by)$ are the $\epsilon$ approximation of score functions $\nabla \log \frac{\mathrm{d} \mu_{\star}}{\mathrm{d}\bx}(\bx)$ and $\nabla \log \frac{\mathrm{d} \nu_{\star}}{\mathrm{d}\by}(\by)$ at the $k$-th iteration, respectively
    \begin{equation*}
    \begin{split}
        \label{approximate_error_supp}
        &\Big\|\nabla \log \frac{\mathrm{d}\mu_{\star, k}}{\mathrm{d}\bx} - \nabla \log \frac{\mathrm{d} \mu_{\star}}{\mathrm{d}\bx}\Big\|_{\infty}\leq \epsilon\\
        &\Big\|\nabla \log \frac{\mathrm{d}\nu_{\star, k}}{\mathrm{d}\by} - \nabla \log \frac{\mathrm{d} \nu_{\star}}{\mathrm{d}\by}\Big\|_{\infty}\leq \epsilon, 
    \end{split}
    \end{equation*}
\end{assump}

Such an assumption is closely related to the $\epsilon$-accurate score approximation in \citet{DSB, lee2022convergence} except that our focus is the marginals on $\mu_{\star}$ and $\nu_{\star}$ while theirs is the marginal density along the forward SDE \eqref{SGM-SDE-f}. To further extend the score approximation assumption from $L^{\infty}$ (uniformly accurate) to $L^2$ (in expectation), we can leverage the ``bad set'' idea \citep{lee2022convergence} or the Girsanov theorem \citep{Sitan_22_sampling_is_easy} to match the likelihood training framework better. Moreover, the errors in the two marginals do not need to be identical, and a unified $\epsilon$ is employed mainly for the sake of analytical convenience.

\begin{assump}[Lipschitz smoothness] \label{smooth_potential}
The score functions of marginal densities $\nabla\log \frac{\mathrm{d} \mu_{\star}}{\mathrm{d}\bx}$ are $\nabla\log \frac{\mathrm{d} \nu_{\star}}{\mathrm{d}\by}$ are both $L$-Lipschitz smooth.
\end{assump}

To sketch the proof, we first show a summation property of $\sum_{k\geq 1}^n \text{KL}(\pi_k | \pi_{k-1})$ without breaking the cyclical invariance property \citep{Nutz_22_func} in Lemma \ref{corollary_summable} such that
\begin{equation*}
    \sum_{k\geq 1}^n \text{KL}(\pi_k | \pi_{k-1})\leq \text{KL}( \pi_{\star}|\mathcal{G})-\text{KL}(\pi_0|\mathcal{G})+ O(n\epsilon).
\end{equation*}
Next, we prove $\text{KL}(\mu_{2k}|\mu_{\star, k})\leq \text{KL}(\pi_{2k}|\pi_{2k-1})$ and $\text{KL}(\nu_{\star, k}|\nu_{2k-1})\leq \text{KL}(\pi_{2k}|\pi_{2k-1})+O(\epsilon)$, which yields
\begin{equation*}
\begin{split}
    \sum_{k\geq 1}^n \text{KL}(\mu_{2k}|\mu_{\star, k})&\leq \text{KL}( \pi_{\star}|\mathcal{G})-\text{KL}(\pi_0|\mathcal{G})+ O(n\epsilon)\\
    \sum_{k\geq 1}^n \text{KL}(\nu_{\star, k}|\nu_{2k-1})&\leq \text{KL}( \pi_{\star}|\mathcal{G})-\text{KL}(\pi_0|\mathcal{G})+ O(n\epsilon).
\end{split}
\end{equation*}
Moreover, we obtain an approximately monotone-decreasing property in proposition \ref{monotone_property} as follows
\begin{equation*}
\begin{split}
    \text{KL}(\mu_{2k}|\mu_{\star, k}) &\leq \text{KL}(\mu_{2t+2}|\mu_{\star, k+1}) + O(\epsilon)\\
    \text{KL}(\nu_{\star, k}|\nu_{2k-1}) &\leq \text{KL}(\nu_{\star, k+1}|\nu_{2k+1})+O(\epsilon).
\end{split}
\end{equation*}
Finally, combining Lemma \ref{finite} and the fact that $\mu_{\star, k}$ (or $\nu_{\star, k}$) is $\epsilon$-close to $\mu_{\star}$ (or $\nu_{\star}$), our main theorem follows that:
\begin{theorem}[Approximately Sublinear Rate for Marginals]\label{main_theorem}
Given dissipative assumption \ref{Dissipativity}, $\epsilon$-approximate score assumption \ref{approx_score}, and smoothness assumption \ref{smooth_potential}, we have
\begin{equation*}
\begin{split}
    \text{KL}(\mu_{2k}|\mu_{\star}) &\leq \frac{\text{KL}( \pi_{\star}|\mathcal{G})-\text{KL}(\pi_0|\mathcal{G})}{k} + O({\epsilon}^{\frac{1}{2}} +k^{\frac{1}{2}}\epsilon)
    \\ \text{KL}(\nu_{\star}|\nu_{2k-1}) &\leq \frac{\text{KL}( \pi_{\star}|\mathcal{G})-\text{KL}(\pi_0|\mathcal{G})}{k}+O({\epsilon}^{\frac{1}{2}} +k^{\frac{1}{2}}\epsilon),
\end{split}
\end{equation*}
% \Wei{$O(\sqrt{\epsilon}) +\sqrt{k}\epsilon$ v.s. $\sqrt{k\epsilon}$}
where the big-O notations are independent of $k$.
\end{theorem}
The proof is presented in Appendix \ref{main_result_proof}, which provides the first-ever evidence of the convergence of the aIPF algorithm with approximate projections. Our analysis suggests that to achieve an $\epsilon$-accurate target, the iteration should be greater than $\Omega(\frac{1}{\epsilon})$, although early stopping may be necessary to avoid excessive perturbations.  It is worth noting that \emph{order $O(k^{\frac{1}{2}})$ is more preferable than linear-order or expansive upper bounds} in \citet{SBP_max_llk} (when $K\geq 1$).  However, we acknowledge this result is not entirely practical without the bounded-cost-function assumption. We believe the square-root order can be further refined by obtaining a tighter convergence rate \citep{Nutz_sinkhorn_order2, sinkhorn_exp_general}. This refinement can be left as future work.

Moreover, the convergence of the minimized cost \eqref{SBP_classical_main} potentially facilitates the estimation of score functions. However, it involves a trade-off between computation and accuracy. Such a trade-off can be used to establish instances where Schr\"{o}dinger bridge is faster than SGMs.

\subsection{Connections between aIPF and FB-SDE}

The complexity analysis of SBP hinges on the convergence of aIPF based on entropic optimal transport; however, solving the exact integrals in Algorithm \ref{sinkhorn} is far from trivial and heavily relies on score-based sampling techniques. In the next section, we present the likelihood training of SBP to connect with aIPF and build the conditional variant for applications in probabilistic time series imputations. To facilitate reading, we visualize the connections between the convergence analysis of aIPF and the likelihood training of FB-SDE in Figure \ref{fig:diagram} below.

\begin{figure}[ht]
\centering
% \vspace{-0.1in}
\includegraphics[width=0.48\textwidth]{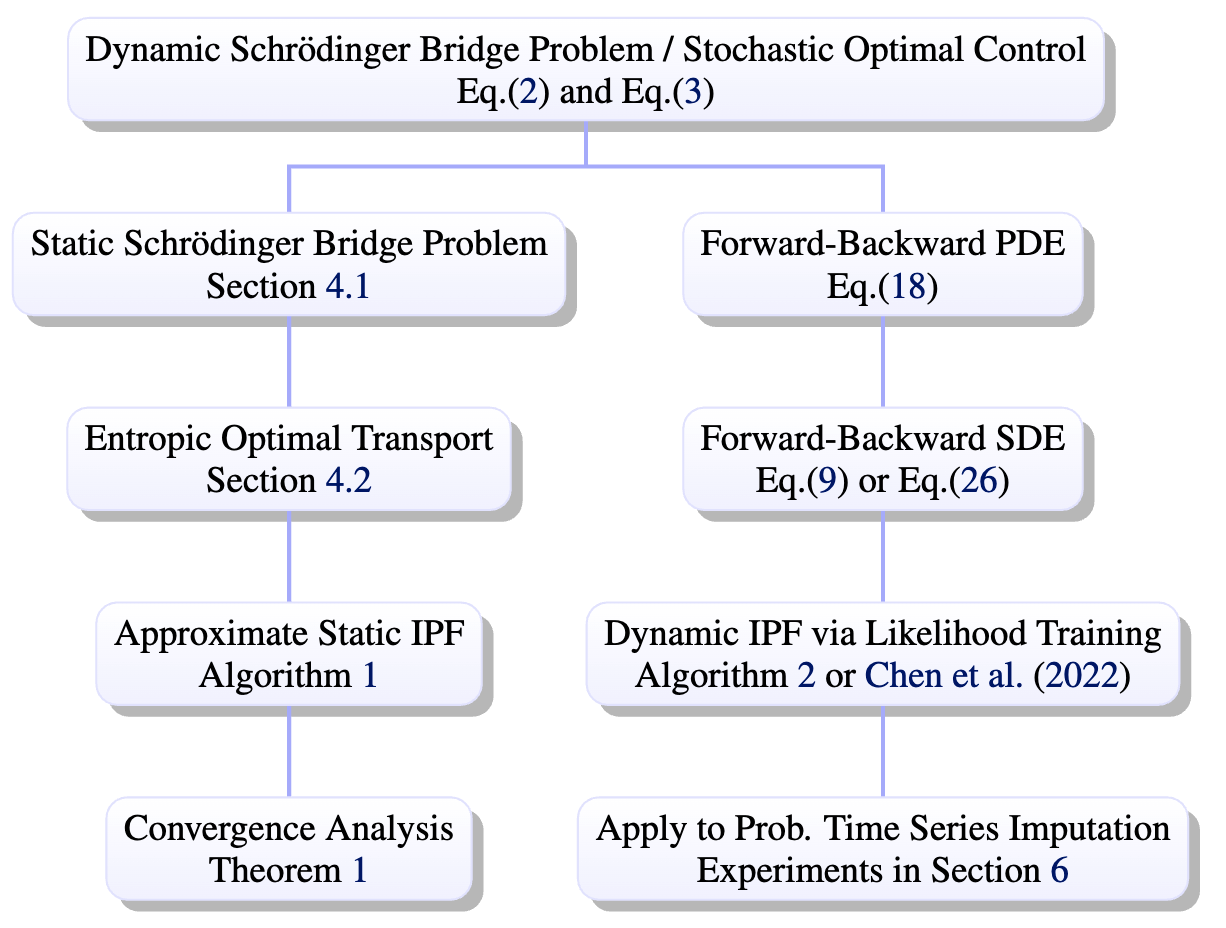}
% \vspace{-0.1in}
\caption{Relation between IPF, SBP, SOC, and FB-SDE.
}
\vspace{-0.1in}
\label{fig:diagram}
\end{figure}

 \begin{figure*}[ht]
\centering
\vspace{0.1in}
\includegraphics[width=\textwidth]{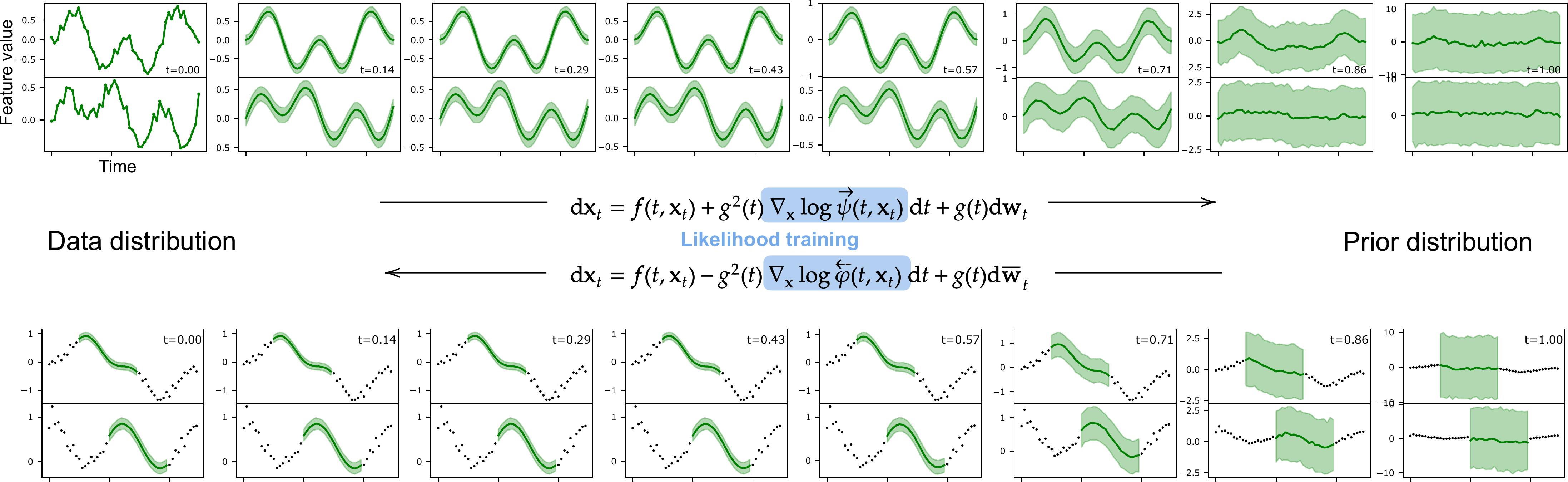}
\vspace{-0.1in}
\caption{Demo of conditional Schr\"{o}dinger bridge for imputation (CSBI). The example shows two correlated time-series features, the dark dots are condition values, and the green band represents an 80\% confidence interval of the imputed missing values, which starts from the prior distribution with very large uncertainty on the right to narrow and smooth band on the left matching the data distribution very well. The score functions $\nabla \log \overrightarrow\psi$ and $\nabla \log \overleftarrow\varphi$ in FB-SDEs \eqref{bs_sde_formulation} are trained via divergence-based likelihood. 
% \Wei{to unify the notations with the nips submission, I changed $\psi$ to $\overrightarrow\psi$ and $\varphi$ to $\overleftarrow\varphi$, you may also change the figure notations here.}
}
\label{fig:intro}
\end{figure*}

\section{Likelihood Training with Applications to Probabilistic Time Series
Imputation}
\label{Sec: method}

In this section, we solve SBP via likelihood training of FB-SDEs and briefly present the conditional Schr\"{o}dinger
bridge method for probabilistic time series
imputation (CSBI).

\subsection{Likelihood Training via Schr\"{o}dinger Bridge}
\label{Sec:SBP-llk}

Solving SBP is often intractable. We show it could be transformed into computation-friendly FB-SDEs. We sketch the main results here and detail derivations in appendix \ref{Preliminaries_SBP}.

Rewrite the SOC perspective of SBP \eqref{SBP_classical_main} with $\varepsilon=\frac{1}{2}$ and constraints \eqref{FKP_eqn_supp} into a Lagrangian \citep{Chen21}:
    \begin{align}
    &\mathcal{L}(\rho, \bu, \phi):=\int_0^T \int_{\mathbb{R}^d} \frac{1}{2}\|\bu(\bx_t,t)\|^2_2\rho(\bx_t, t) \notag\\ 
    &+\phi(\bx_t, t) \bigg(\frac{\partial \rho}{\partial t}+\nabla \cdot(\rho (\bbf+g\bu))-\frac{g^2 \Delta \rho}{2} \bigg) \mathrm{d}\bx_t \mathrm{d}t\notag,
\end{align}
where $\phi(\bx_t, t)$ denotes a Lagrangian multiplier \citep{Chen21}. Further, consider the log transformation based on score functions $(\nabla\log \overrightarrow\psi, \nabla\log \overleftarrow\varphi)$
 \begin{align}
     \overrightarrow\psi(\bx_t, t) &= \exp(\phi(\bx_t, t)/g^2)\notag\\
     \overleftarrow\varphi(\bx_t, t)&=\rho^{\star}(\bx_t, t)/\overrightarrow\psi(\bx_t, t),\notag
 \end{align}
where $\rho^{\star}(\bx, t)$ is the optimal density conditional on the optimal control $\bu^{\star}:= \nabla \phi(\bx, t)$.  
%Together with the joint-space setting, 
Now we obtain the FB-SDEs \citep{forward_backward_SDE}:

\begin{proposition}
The forward-backward SDE (FB-SDE) associated with the problem \eqref{SBP_classical_main} with $\epsilon=\frac{1}{2}$ and conditional constraints follows 
\begin{subequations}\label{bs_sde_formulation}
\begin{equation}
\begin{aligned}
\footnotesize
\mathrm{d} \bx_t=\left[\bbf(\bx_t, t) +g(t)^2\nabla\log\overrightarrow\psi(\bx_t, t)\right]\mathrm{d}t+g(t) \mathrm{d} \bm{\mathrm{w}}_t, \label{bs_sde_formulation_A} \\
\end{aligned}
\end{equation}
\begin{equation}
\begin{aligned}
\footnotesize
\mathrm{d} \bx_t=\left[\bbf(\bx_t, t) -g(t)^2\nabla\log\overleftarrow\varphi(\bx_t, t)\right]\mathrm{d}t+g(t) \mathrm{d} \bar{\bm{\mathrm{w}}}_t  \label{bs_sde_formulation_B}. \\
\end{aligned}
\end{equation}
\end{subequations}
\end{proposition}
where $\bx_0\sim \mu_{\star}$ and $\bx_T \sim \nu_{\star}$.

\begin{algorithm*}[ht]
% \vspace{-0.1in}
\caption{Likelihood training of conditional Schr\"{o}dinger bridge for imputation (CSBI)}\label{alg2_CSBI}
\label{alg-train}
\begin{algorithmic}[1]
\STATE \textbf{Input:} samplers of joint space $p_{\mathrm {prior}}(\bx)$, $ p_{\mathrm {obs}}(\bx)$, and masks.
\STATE Set output of $\overrightarrow\bz_{t}^{\theta}$ as zero and warmup train $\overleftarrow\bz_{t}^{\omega}$ using score matching.
\REPEAT
\STATE  Update cached forward trajectories following \eqref{bs_sde_formulation_A} with certain frequency.
\STATE Compute ${\mathcal{L}}^{\mathrm{SBP}}_{\omega}$ 
    \eqref{eq:llk-ipf-b}  with $\bx_0\sim p_{ \mathrm{obs}}(\bx)$ together with $\bM_{\mathrm{cond}}$, $\bM_{\mathrm{target}}$.
\STATE Take gradient step 
$\nabla {\mathcal{L}}^{\mathrm{SBP}}_{\omega}$ 
    and update $\overleftarrow\bz_{t}^{\omega}$.
\STATE Update cached backward trajectories of \eqref{bs_sde_formulation_B} with certain frequency.
\STATE Compute ${\mathcal{L}}^{\mathrm{SBP}}_{\theta}$ \eqref{eq:llk-ipf-f},
    with $\bx_T\sim p_{\mathrm{prior}}(\bx)$, $\bM_{\mathrm{cond}}=\textbf{0}$, $\bM_{\mathrm{target}}=\textbf{1}$.
\STATE Take gradient step  
$\nabla {\mathcal{L}}^{\mathrm{SBP}}_{\theta}$ 
    and update $\overrightarrow\bz_{t}^{\theta}$.
\UNTIL \textbf{a stopping criterion}
\end{algorithmic}
% \vspace{-0.06in}
\end{algorithm*}

Define $\overrightarrow\bz_{t}=g(t) \nabla \log\overrightarrow\psi(\bx_{t}, t) $ and
$ \overleftarrow\bz_{t}= g(t)\nabla \log\overleftarrow\varphi(\bx, t)$. It\^{o}'s lemma (see Theorem 3 \citep{forward_backward_SDE}) leads to the diffusion of $\overrightarrow\bz_{t}$ and $\overleftarrow\bz_{t}$ in \eqref{bs_sde_formulation}.

We use models $\overrightarrow\bz_{t}^{\theta}$ and $\overleftarrow\bz_{t}^{\omega}$ to learn the forward policy $\overrightarrow\bz_t$ and backward policy $\overleftarrow\bz_t$ and refer to the objective of data likelihood as $\mathcal{L}^{\text{SBP}}_{\theta,\omega}$.  In the context of imputation problems with conditional and target entries, maximizing the likelihood is equivalent to optimizing the backward policy $\overleftarrow\bz_{t}^{\omega}$ and forward policy $\overrightarrow\bz_{t}^{\theta}$ as follows \citep{forward_backward_SDE}:

\begin{subequations}
\begin{align}
% \tiny
&{\mathcal{L}}^{\mathrm{SBP}}_{\omega}(\bx_0)= {
    - \widehat{\E}_{\bx_t \sim \text{\eqref{bs_sde_formulation_A}}} \bigg[\frac{1}{2} 
    {\| \overleftarrow\bz_{t}^{\omega} \circ \bM_{\mathrm{target}} \|}_2^2}+ \label{eq:llk-ipf-b}  \\
 &  g\nabla \cdot \left( \overleftarrow\bz_{t}^{\omega} \circ \bM_{\mathrm{target}} \right)
     {
    + (\overrightarrow\bz_t \circ \bM_{\mathrm{target}})^\intercal 
    (\overleftarrow\bz_{t}^{\omega} \circ \bM_{\mathrm{target}})\bigg]
    } \notag \\
& \small{
    {\mathcal{L}}^{\mathrm{SBP}}_{\theta}(
    \bx_T) = 
    - \widehat{\E}_{\bx_t \sim \text{\eqref{bs_sde_formulation_B}}} \bigg[\frac{1}{2} {\| \overrightarrow\bz_{t}^{\theta}\|}_2^2
    + g\nabla \cdot \overrightarrow\bz_{t}^{\theta}
    + \overleftarrow\bz_t^\intercal \overrightarrow\bz_{t}^{\theta} \bigg]
    }, \label{eq:llk-ipf-f}
\end{align}\label{eq:llk-ipf}
\end{subequations}

where $\widehat{\E}$ denotes the empirical expectations of the sampled trajectories according to the FB-SDEs \eqref{bs_sde_formulation}; $\bM_{\mathrm{target}}$ is the conditional mask to be clarified in section \ref{Sec:joint-space}; $\nabla\cdot$ denotes the divergence (for clarity, $\nabla$ is the gradient.). The masks and conditions are not required in Eq.\eqref{eq:llk-ipf-f}, because the backward SDEs start from the known prior. Since simulating the full sample trajectory is costly, we apply the caching-trajectory strategy \citep{DSB, forward_backward_SDE} to improve the efficiency. Now, we present the practical method in Algorithm \ref{alg-train} and refer to the conditional Schr\"{o}dinger bridge for imputation as CSBI. Similar to \citet{DSB}, this algorithm can be viewed as a dynamic implementation of the IPF algorithm.

During the inference, conditional sampling follows the joint distribution learning by applying the backward policy \citep{score_sde}. %, similar to \citet{rezende2014stochastic}. 
The Langevin corrector \citep{score_sde, forward_backward_SDE} can also be used to improve performance.
See details in Appendix \ref{inference_supp}.

\subsection{Joint Space of SBP for Time Series Imputation }\label{Sec:joint-space}

{
Time series imputation task requires filling missing values in arbitrary entries given partial observations in random positions,
where the condition-target relation usually varies from sample to sample.
This requires the model to capture both temporal and feature-wise dependency at the same time. % for the whole time window. In this context, modelling  with fixed condition-target relation may not be suitable for time series imputation. 
Next, we present our framework based on divergence objectives.
The joint distribution learning of $\bx := (\bx_{ \text{target}},\bx_{ \text{cond}})$ is the following.
 % \Wei{can we polish this paragraph a little bit e.g. reduce 30\%.}
 % Yu: I cut out some words.
\begin{align}\label{eq: joint-constrain}
    &\begin{cases}
    \bx_0\sim \mu_{\star}=p_{ \text{target}| \text{cond}}(\bx)  p_{ \text{cond}}(\bx)\\
    \bx_T\sim \nu_{\star}=p_{\text{prior-} \text{target}}(\bx)  p_{ \text{cond}}(\bx),
    \end{cases}
\end{align}
where $p_{\text{target}| \text{cond}}(\bx) = p(\bx_{ \text{target}}|\bx_{ \text{cond}})$ is the target conditional distribution, $p_{\text {prior-target} }(\bx)$ is the prior distribution of target values, and $p_{\text {cond}}(\bx) = p(\bx_{ \text{cond}})$ is the data-dependent distribution of observations being conditioned on. 

\begin{figure}[ht]
\centering
% \vspace{-0.1in}
\includegraphics[width=0.48\textwidth]{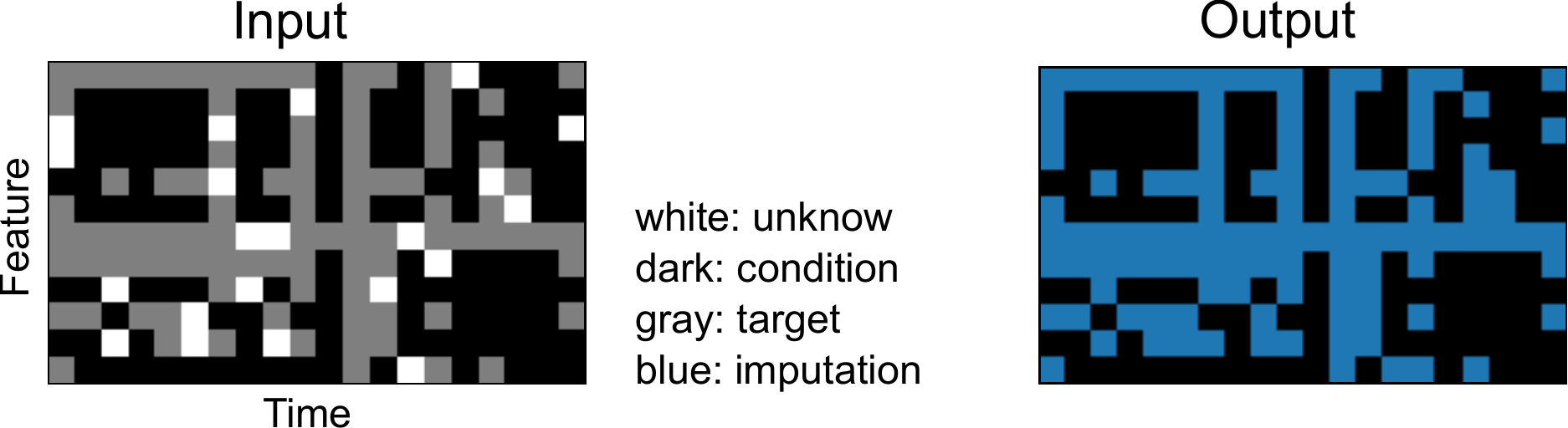}
\vspace{-0.2in}
\caption{Demonstration of entry types and masks. The usage for SBP is described in section \ref{Sec:joint-space}}
% \vspace{-0.1in}
\label{fig:mask}
\end{figure}

\paragraph{Masking for conditional inference } 
The irregular condition-target relation is indicated by observation mask $\bM_{\mathrm{obs}}$, condition mask $\bM_{\mathrm{cond}}$, and target mask $\bM_{\mathrm{target}}$ (Figure \ref{fig:mask}).
$\bM_{\mathrm{obs}}$ covers all ground true values,
unknown entries is complementary to $\bM_{\mathrm{obs}}$ without ground truths,
$\bM_{\mathrm{target}}$ is for the imputation target,
$\bM_{\mathrm{cond}}$ indicates the input condition for the model, which is a subset of $\bM_{\mathrm{obs}}$.
When the model is trained or evaluated, $\bM_{\mathrm{target}}$ is usually set as part of $\bM_{\mathrm{obs}}$ so the performance can be calculated by comparing the imputed values and the ground truths.
When the model is deployed, $\bM_{\mathrm{target}}$ can also cover the unknown entries. See more details on masks in Appendix \ref{subsec:imputation_formulation}.

\vspace{-0.1in}

\begin{table*}[ht]
\caption{Model evaluation: 
The metrics include root mean square error (RMSE), mean absolute error (MAE), and continuous ranked probability score (CRPS).
A smaller metric is better. Results are obtained from 5 folds. 
% \textcolor{blue}{Yu: experiments are running. 
% mainly focusing on physio-0.5, 0.9.
% will keep updating the results before deadline. }
}
% }
\label{tab:results}
\vspace{0.1in}
% \footnotesize
\centering
% \resizebox{\textwidth}{!}{%
\begin{tabular}{l||lll||lll||lll }
\hline
 & \multicolumn{3}{c}{PM2.5} & \multicolumn{3}{c}{PhysioNet 0.1} 
 & \multicolumn{3}{c}{PhysioNet 0.5} \\ \hline
Metrics & 
\multicolumn{1}{l}{RMSE} & \multicolumn{1}{l}{MAE} & CRPS & 
\multicolumn{1}{l}{RMSE} & \multicolumn{1}{l}{MAE} & CRPS & 
\multicolumn{1}{l}{RMSE} & \multicolumn{1}{l}{MAE} & CRPS \\ \hline

V-RIN & \multicolumn{1}{l}{ 40.1 } & \multicolumn{1}{l}{  25.4 } & 0.53  & 
    \multicolumn{1}{l}{ 0.63 } &  \multicolumn{1}{l}{  0.27 } &  0.81 & 
    \multicolumn{1}{l}{ 0.69 } & \multicolumn{1}{l}{ 0.37 } &  0.83  \\ \hline

Multitask GP & \multicolumn{1}{l}{ 42.9 } & \multicolumn{1}{l}{ 34.7 } & 0.27  & 
    \multicolumn{1}{l}{ 0.80 } &  \multicolumn{1}{l}{ 0.46 } & 0.49 & 
    \multicolumn{1}{l}{ 0.84 } & \multicolumn{1}{l}{ 0.51 } &  0.56 \\ \hline

GP-VAE & \multicolumn{1}{l}{ 43.1  } & \multicolumn{1}{l}{ 26.4 } &  0.41 & 
    \multicolumn{1}{l}{ 0.73 } &  \multicolumn{1}{l}{ 0.42 } & 0.58 & 
    \multicolumn{1}{l}{ 0.76 } & \multicolumn{1}{l}{ 0.47 } & 0.66 \\ \hline

CSDI  & \multicolumn{1}{l}{ 19.3 } & \multicolumn{1}{l}{9.86} & \textbf{0.11}  &   %pm25 % 0.108
    \multicolumn{1}{l}{ 0.57} &  \multicolumn{1}{l}{ 0.24} &  0.26 &  % physio 0.1
    \multicolumn{1}{l}{ 0.65} & \multicolumn{1}{l}{ 0.32} & \textbf{0.35}  \\ \hline  %physio 0.5
%
%%%%%%%%%%%%%%%%%%%%% ours %%%%%%%%%%%%%%%%%%%%%%%%%%%
\textbf{CSBI (ours)} & \multicolumn{1}{l}{ \textbf{19.0} } & \multicolumn{1}{l}{\textbf{9.80} } & \textbf{0.11} & % 10.25 & 0.112
    \multicolumn{1}{l}{ \textbf{0.55}} &  \multicolumn{1}{l}{ \textbf{0.23} } &  \textbf{0.25} & 
    \multicolumn{1}{l}{ \textbf{0.63}} & \multicolumn{1}{l}{ \textbf{0.31}} &  \textbf{0.35} \\ \hline

\end{tabular}%
% \vspace{-0.2in}
\end{table*}

\section{Experiments}\label{main_experiments}
In this section, we evaluate the performance of CSBI through one synthetic data and two real datasets. \footnote[2]{ \url{https://github.com/morganstanley/MSML/tree/main/papers/Conditional_Schrodinger_Bridge_Imputation}.}

%(details in Appendix \ref{sec:appendix_dataset}).
% \Wei{add appropriate appendix links}
\subsection{Datasets}
\paragraph{Synthetic Data}

We first test our algorithm on a simple synthetic dataset. The time series data has $K=8$ features and $L=50$ time steps. The data is created by adding signals, sinusoidal curves of various frequencies, and random noise.
Next, data entries are removed randomly mimicking the missed observed values (unknown entries). 
The observed entries are split into conditions and artificial targets. 
20 consecutive time points of each feature are selected as the artificial targets.
More details can be found in Appendix \ref{sec:appendix_dataset}.
Examples are shown in Figure \ref{fig:sinusoid}.

\begin{figure}[ht]
\centering
% \vspace{-0.15in}
\includegraphics[width=0.48\textwidth]{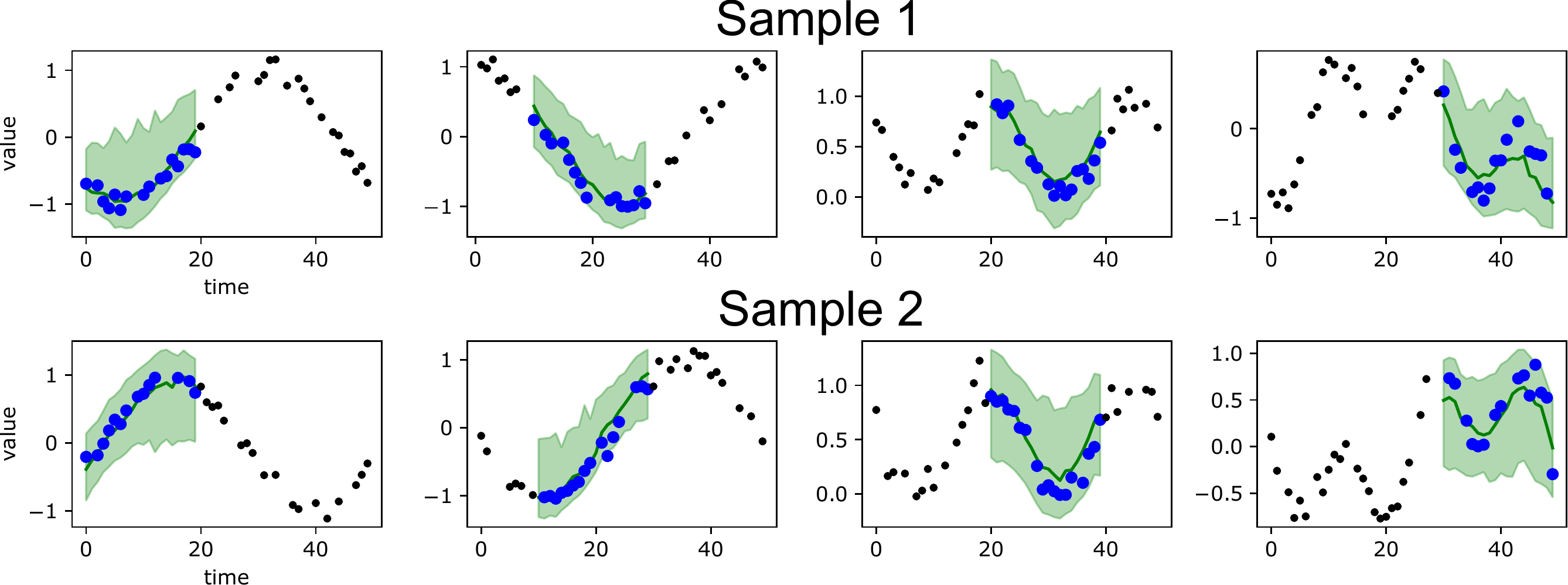}
\caption{A demo with 4 features and 50 time steps. Dark dots are conditions, blue dots are the ground truth of the missing values, and the error bar shows the 80\% confidence. The imputation is conducted for the targets (blue dots) and the unknowns (gaps between blue dots). The target mask of the last feature is at the end of the time window which is equivalent to the prediction task. 
}
\label{fig:sinusoid}
\vspace{-0.2in}
\end{figure}

\paragraph{Environmental Data} It is composed of the hourly sampled PM2.5 air quality index (with unit $\mu g/m^3$) from 36 monitoring stations for 12 months \citep{U_Air}. The time window has $K=36$ features and $L=36$ time points. The raw data has 13\% missing values (the portion of unknown entries).
The target entries only come from the observed entries.
See demonstrations in Appendix \ref{sec:appendix_imputation_eg}.

\paragraph{Healthcare Data}
Another dataset widely used in time series imputation literature is the PhysioNet Challenge 2012 \citep{phy_2012}. It has 4000 clinical time series with $K=35$ features and $L=48$ time points for 48 hours from the intensive care unit (ICU).
The raw data is sparse with 80\% missing values (the portion of unknown entries) making the imputation very challenging.
We further randomly select 10\% and 50\% out of observed values as the targets.
Demonstrations are shown in Appendix \ref{sec:appendix_imputation_eg}.

\subsection{Model Pipeline}
In this section, we briefly describe the pipeline of the framework.
More details about the neural networks, training procedure, inference, baseline models, and evaluation can be found in Appendix \ref{appendix:experiment}.

As described in section \ref{Sec:SBP-llk} and algorithm \ref{alg-train}, 
we use two separate neural networks to model the forward or backward policy 
$\overrightarrow{\bz}_t$ and $\overleftarrow{\bz}_t$.
The backward network needs to handle partially observed input and conduct conditional inference. More specifically, the backward policy has format $\overleftarrow{\bz}_t(t, \bX_{t,\mathrm{cond}}, \bM_{\mathrm{cond}})$ which takes in diffusion time, conditions, and outputs the policy of the whole time window (its outputs at condition positions are usually ignored).
While the forward network, as an assistant for training the backward policy, does not need to process partial input, and we use a modified U-Net as the neural network \citep{ronneberger2015u}.
In both networks, the diffusion time is incorporated through embedding.
Similar to the design \citep{CSDI}, the backward policy handles the input with irregular conditions based on the transformer, where the condition information is encoded through channel concatenation, feature index embedding, and time index embeddings (using the time point index of the time window, not the actual time of the time series to have a fair comparison with baseline models).

Our baseline models include V-RIN \citep{v_rin}, multitask Gaussian process (multitask GP) \citep{multitask_GP}, GP-VAE \citep{GP_VAE}, and the state-of-the-art model CSDI \citep{CSDI}. Our model is evaluated using 100 samples. We report the mean absolute error (MAE) and root mean square error (RMSE); in addition, we include the continuous ranked probability score (CRPS) %\citep{doi:10.1198/016214506000001437} 
to measure the quality of the imputed distribution that is calculated using all samples.

\subsection{Evaluation}
\paragraph{Synthetic Data} We try to answer a key question here: 
\begin{center}
    {\it \textcolor{darkblue}{Does lower transport costs facilitate estimations of score functions and yield samples of higher quality?}}
\end{center}

To answer the question in a consistent framework, we compare CSBI with a CSBI variant by forcing $\nabla\log \overrightarrow\psi\equiv 0$ in Eq.\eqref{bs_sde_formulation}, where the latter (denoted by CSBI$_0$) is, in theory, equivalent to the SGM-based CSDI. %, although it suffers from a large variance issue via Hutchinson estimators. 
We observe in Figure \ref{fig:dsm_vs_sb_empirical} that all criteria of CSBI converge to small errors, however, CSBI$_0$ yields a rather crude performance when the terminal time $T$ in Eq.\eqref{bs_sde_formulation} is small, which helps answer the question affirmatively. See details in Appendix \ref{sec:dsm_vs_sb_empirical}.

\begin{figure}[ht]
\centering
\includegraphics[width=0.4\textwidth]{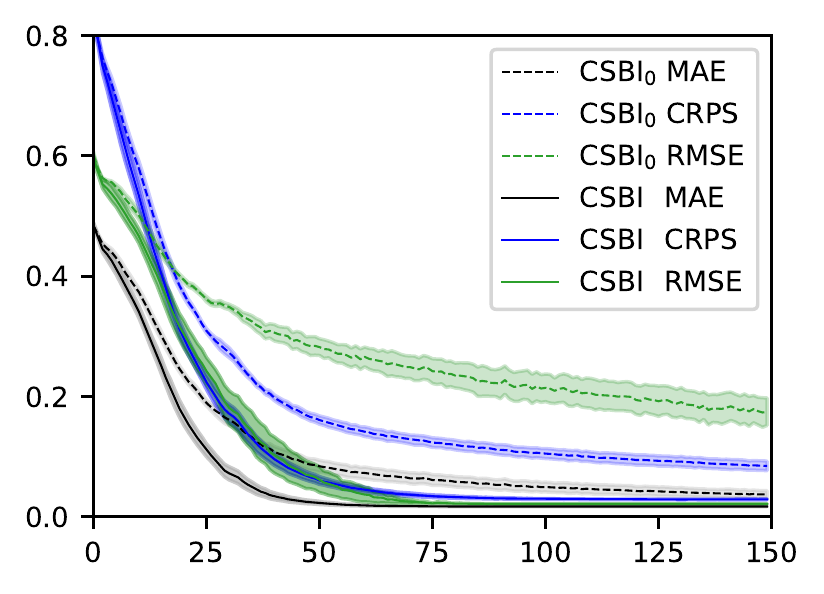}
% \vspace{-0.25in}
\caption{Comparison between CSBI$_0$ (equivalent to the SGM-based CSDI in theory) and CSBI in a consistent framework.
Both models are trained with 150 iterations under the same settings.
}
\label{fig:dsm_vs_sb_empirical}
% \vspace{-0.17in}
\end{figure}

\paragraph{Healthcare and Environmental Data} We observe in Table \ref{tab:results} that the recurrent imputation networks in V-RIN achieve the worst performance; the GP-based methods, such as  multitask GP and GP-VAE, can model the uncertainty accurately, however, it is still inferior to CSDI in these two datasets; we suspect one reason is that GP-based methods rely more on distributional assumptions; in addition, the VAE structure needs to take the whole window (conditions and dummy missing values) as the input to generate the latent states, as such, the dummy values at the target entries may affect the performance. By contrast, our CSBI model achieves state-of-the-art performance and slightly outperforms the CSDI model, although the latter is already highly optimized. 
% \textcolor{blue}{(Yu: The performance is improved by a little bit. We can say achieve sota and outperform in some sub-datasets. )} \Wei{currently not so superior}
Despite the optimized transport cost, we didn't achieve a significant improvement. One reason is due to the large variance issue via the trace estimator \citep{Hutchinson89}. See discussions in appendix \ref{limitations} for details.

\subsection{Time Series Forecasting} 

Time series prediction is a natural application of our current framework, where the condition mask is substituted with the context window and the target mask with the future window. Our method accommodates missing values in the context window during both training and inference, eliminating the need to fill these gaps with dummy values.
% Our method allows missing values in the context window during both training and inference. Unlike the method in \citep{salinas2019high}, we do not need to fill the missing entries in the context window with dummy values. 
The training and inference remain the same as the imputation task.
Table \ref{tab:prediction} shows the results using two public datasets: Solar and Exchange \citep{lai2018modeling}.
Baseline models include GP-copula \citep{salinas2019high}, Vec-LSTM-low-rank-Copula (Vec) \citep{salinas2019high}, TransMAF \citep{rasul2021autoregressive}.
Our method achieves competitive performance compared to other baseline models. For a comprehensive understanding, we provide detailed information in Appendix \ref{sec:appendix_prediction}.
\vspace{-0.23in}

\begin{table}[H]
\caption{Time series prediction task evaluated by CRPS.}
% (lower is better). }
\vspace{-0.1in}
\label{tab:prediction}
% \vskip 0.1in
\begin{center}
% \begin{small}
% \begin{sc}
{\footnotesize
\begin{tabular}{lcccc}
\toprule
Methods & GP-copula	 & Vec & TransMAF & \textbf{CSBI (ours)} \\
\midrule
Exchange &  \textbf{0.008} & 0.009 & 0.012 & \textbf{0.008} \\
Solar    &  0.371 & 0.384 & 0.368 & \textbf{0.365} \\
\bottomrule
\end{tabular}
}
% \end{sc}
% \end{small}
\end{center}
% \vskip -0.1in
\vspace{-0.2in}
\end{table}

\section{Conclusion}

Schr\"{o}dinger bridge algorithm is gaining popularity in generative models, however, to the best of our knowledge, there is no prior work studying the convergence based on the IPF algorithm with approximate projections. To bridge the gap between theoretical analysis and empirical training, we provide the first approximation analysis and motivate future research to obtain a tighter upper bound. We also draw connections to demystify the connections between IPF and FB-SDEs, which sheds light on the complexity analysis for future algorithm development. 

For applications to probabilistic time series imputation, we  propose a conditional Schr\"{o}dinger bridge algorithm based on divergence-based likelihood training and the method is able to tackle missing values in random positions. Empirically, the proposed algorithm is tested on multiple datasets; the great performance indicates the effectiveness of our algorithm in time series imputation.  The flexible formulation further shows a potential to extend the linear Gaussian prior to more general priors \citep{CSGLD, AWSGLD} to yield more efficient algorithms.

% \textcolor{red}{remove irrelevant references.}

\bibliography{mybib, mybib_other}
\bibliographystyle{icml2023}

\newpage
\appendix
\onecolumn
\raggedbottom
 
\begin{Large}
\begin{center}
    \textbf{Supplimentary Material for \textit{``Provably Convergent Schr\"{o}dinger Bridge with Applications to Probabilistic Time Series Imputation''}}
\end{center}
\end{Large}
$\newline$
In section \ref{Preliminaries}, we lay out the preliminary knowledge of Schr\"{o}dinger Bridge; In section \ref{convergence}, we  establish the main convergence result for Schr\"{o}dinger Bridge based on approximated scores; In section \ref{appendix:experiment}, we provide the experimental details based on a synthetic dataset, PM2.5, and PhysioNet data.

\section{Preliminaries}
\label{Preliminaries}
\subsection{From Schr\"{o}dinger Bridge problem (SBP) to FB-SDE}
\label{Preliminaries_SBP}
The stochastic-optimal-control perspective of SBP (see section 4.4 in \citet{Chen21} and \citep{Pavon_CPAM_21, Caluya21}) proposes to minimize 
\begin{align}
    &\inf_{\bu\in \mathcal{U}} \E\bigg\{\int_0^T \frac{1}{2}\|\bu(\bx,t)\|^2_2 \mathrm{d}t \bigg\} \notag\\
    \text{s.t.} &\ \ \mathrm{d} \bx_t=\left[\bbf(\bx, t)+g(t)\bu(\bx,t)\right]\mathrm{d}t+\sqrt{2\varepsilon} g(t)\mathrm{d} \mathbf{w}_t \label{control_diffusion}\\
    &\ \ \bx_0\sim  \mu_{\star} ,\ \  \bx_T\sim  \nu_{\star} \notag
    ,
\end{align}
where $\mathcal{U}$ is a set of control variables $\bu:\mathbb{R}^d\times [0,T]\rightarrow \mathbb{R}^d$; the state-space of $\bx$ is $\mathbb{R}^d$ by default; $\bbf: \mathbb{R}^d\times [0,T]\rightarrow \mathbb{R}^d$ is the drift or vector field; $\mathbf{w}_t$ is the standard Brownian motion in $\mathbb{R}^d$. The expectation is taken w.r.t the joint state PDF $\rho(\bx, t)$ given initial and terminal conditions; $\varepsilon$ is a scalar and is also related to the regularizer in the EOT formulation. Rewrite SBP into a variational formulation \citep{Chen21}, we have
\begin{align}
    &\inf_{\bu\in \mathcal{U}, \rho} \int_0^T \int_{\mathbb{R}^d} \frac{1}{2}\|\bu(\bx,t)\|^2_2 \rho(\bx, t)\mathrm{d}\bx \mathrm{d}t \label{variational_form_supp} \\
     \text{s.t.} &\ \ \frac{\partial \rho}{\partial t}+\nabla \cdot(\rho (\bbf+g \bu))=\varepsilon g^2\Delta \rho\label{FKP_eqn_supp}\\
     &\ \ \rho(\bx, 0)= \frac{\dd\mu_{\star}}{\dd\bx},\ \  \rho(\by, T)=\frac{\dd\nu_{\star}}{\dd\by},\notag
\end{align}
Note that Eq.\eqref{FKP_eqn_supp} is the Fokker-Planck equation for the corresponding controlled diffusion process \eqref{control_diffusion} based on decision variables $(\rho, \bu)\in \mathcal{P}(\mathbb{R}^d)\times \mathcal{U}$ and $\mathcal{P}(\mathbb{R}^d)$ is the set of probability measures on $\mathbb{R}^d$.

Consider the Lagrangian of \eqref{variational_form_supp} and introduce  $\phi(\bx, t): \mathbb{R}^d\times [0,T]\rightarrow \mathbb{R}$ as a Lagrangian multiplier \citep{chen_21}
\begin{align}
    \mathcal{L}(\rho, \bu, \phi):&=\int_0^T \int_{\mathbb{R}^d} \frac{1}{2}\|\bu(\bx,t)\|^2_2\rho(\bx, t) + \phi(\bx, t) \bigg(\frac{\partial \rho}{\partial t}+\nabla \cdot(\rho (\bbf+g\bu))-\varepsilon g^2\Delta \rho\bigg) \mathrm{d}\bx \mathrm{d}t\notag,\\
    &=\int_0^T \int_{\mathbb{R}^d} \bigg( \frac{1}{2}\|\bu(\bx,t)\|^2_2\ - \frac{\partial \phi}{\partial t}-\langle \nabla \phi, \bbf+g\bu \rangle -\varepsilon g^2\Delta \phi  \bigg)\rho(\bx, t) \mathrm{d}\bx \mathrm{d}t + \underbrace{\int_{{\mathbb{R}}^d} \rho \phi \dd \bx |_{t=0}^{T}}_{\text{constant w.r.t $\bu$}} \label{un_contrained_},
\end{align}
where the second equation is obtained through integration by parts with respect to $t$ and $\bx$, respectively.

Minimizing with respect to $\bu$, we get the optimal control $\bu^{\star}$ as follows
\begin{equation}\label{score_primal}
    \bu^{\star}(\bx, t)=g(t)\nabla\phi(\bx, t).
\end{equation}

Plugging Eq.\eqref{score_primal} into Eq.\eqref{un_contrained_}, we have
\begin{align*}
    \mathcal{L}(\rho, \bu, \phi)=\int_0^T \int_{\mathbb{R}^d} \bigg(- \frac{\partial \phi}{\partial t}-\frac{1}{2}\|g(t)\nabla\phi(\bx, t)\|^2_2 -\langle \nabla \phi, \bbf \rangle -\varepsilon g^2\Delta\phi \bigg)\rho(\bx, t) \mathrm{d}\bx \mathrm{d}t.
\end{align*}

Minimizing the control cost means that $\frac{\partial \phi}{\partial t}+\varepsilon g^2\Delta\phi + \langle \nabla \phi, \bbf \rangle=-\frac{1}{2}\|g(t)\nabla\phi(\bx, t)\|^2_2$.

Given the optimal control $u^{\star}$, the above PDE is known as the \emph{Hamilton–Jacobi–Bellman} (HJB) PDE. Since the HJB PDE is non-linear due to the presence of $\frac{1}{2}\|g(t)\nabla\phi(\bx, t)\|^2_2$, we make it linear through the Cole-Hopf transformation
\begin{subequations}\label{CH_transform}
\begin{align}
    \overrightarrow\psi(\bx, t)&=\exp\bigg(\frac{\phi(\bx, t)}{2\varepsilon}\bigg)\label{CH_transform_a}\\
    \overleftarrow\varphi(\bx, t)&=
    % \rho^{\star}(\bx, t)\exp\bigg(-\frac{\phi(\bx, t)}{2\varepsilon}\bigg)=
    \rho^{\star}(\bx, t)/\overrightarrow\psi(\bx, t),\label{CH_transform_b}
\end{align}
\end{subequations}
where $\rho^{\star}(\bx, t)$ is the optimal density of \eqref{variational_form_supp} conditional on the optimal control $\bu^{\star}$. We now can verify that the transformed variables $(\overrightarrow\psi, \overleftarrow\varphi)$ solve the \emph{Schr\"{o}dinger system} and obtain the backward-forward Kolmogorov equations
\begin{equation}
\begin{split}
    \label{heat_forward_backward}
\frac{\partial \overrightarrow\psi}{\partial t}&+\langle \nabla \overrightarrow\psi, \bbf \rangle +\varepsilon g^2\Delta\overrightarrow\psi=0, \\
\frac{\partial \overleftarrow\varphi}{\partial t}& + \nabla \cdot (\overleftarrow\varphi \bbf) -\varepsilon g^2 \Delta \overleftarrow\varphi=0,
\end{split}
\end{equation}
under the constraint that 
\begin{equation}
\begin{split}\label{product_measures}
\overrightarrow\psi(\bx, t=0)\overleftarrow\varphi(\bx, t=0)&=\frac{\dd\mu_{\star}}{\dd\bx}, \ \ \overrightarrow\psi(\by, t=T)\overleftarrow\varphi(\by, t=T)=\frac{\dd\nu_{\star}}{\dd\by}.
\end{split}
\end{equation}

Next, applying  the solution of Eq.\eqref{heat_forward_backward} leads to the following equations by the Chapman-Kolmogorov equations 
% \textcolor{red}{double check if Feyman Kac is involved}
\begin{equation}
\begin{split}\label{Schrodinger_system}
    \overrightarrow\psi(\bx, t)&=\int_{\mathbb{R}^d} \text{Ker}_{\varepsilon}(t, \bx, T, \by)\overrightarrow\psi(\by, T) \mathrm{d}\by, \ \ \ \overleftarrow\varphi(\by, t)=\int_{\mathbb{R}^d} \text{Ker}_{\varepsilon}(0, \bx, t, \by)\overleftarrow\varphi(\bx, 0) \mathrm{d}\bx,
\end{split}
\end{equation}
where $\text{Ker}_{\varepsilon}(s, \bx, t, \by)$ is the Markov kernel associated with the diffusion $\mathrm{d}\bx_t = \bbf(\bx_t, t) \mathrm{d}t + \sqrt{2\varepsilon}g(t) \mathrm{d} \mathbf{w}_t$; closed-forms of $\text{Ker}_{\varepsilon}$ is in general intractable; some concrete cases follow that
\begin{equation}  
\label{kernel_}
\text{Ker}_{\varepsilon}(s, \bx_s, t, \bx_t)=\left\{  
             \begin{array}{lr}  
             (2\pi \varepsilon(t-s))^{-n/2} \exp\bigg(-\frac{\big\|\bx_t-e^{-\gamma (t-s)} \bx_s\big\|_2^2}{2\varepsilon (1-e^{-2\gamma (t-s)})}\bigg)  & \text{VP-SDE: $\bbf(\bx_t, t)=-\gamma \bx_t$}, g(t)\equiv 1, \\  
              & \\
             (4\pi \varepsilon(t-s))^{-n/2} \exp\big(-\frac{\|\bx_t-\bx_s\|_2^2}{4\varepsilon (t-s)}\big) & \text{VE-SDE: $\bbf(\bx_t, t)=0, g(t)\equiv 1.\quad\ \ \  $} 
             \end{array}
\right.  
\end{equation}

In view of Eq.\eqref{CH_transform}, the optimal decision variables $(\rho, \bu)$ can be obtained as follows 
\begin{align}
    &\rho^{\star}(\bx, t)=\overrightarrow\psi(\bx, t)\overleftarrow\varphi(\bx, t), \ \ \quad \ \ \bu^{\star}(\bx, t)=2\varepsilon g(t)\nabla \log \overrightarrow\psi(\bx, t).\label{optimal_u}
\end{align}

Combining Eq.\eqref{product_measures} and Eq.\eqref{Schrodinger_system}, we solve a variant of \emph{Schr\"{o}dinger equations}  as follows
\begin{equation}
\begin{split}\label{schrodinger_eqn_supp}
\overleftarrow\varphi(\bx, 0)\int_{\mathbb{R}^d} e^{-c_{\varepsilon}(\bx, \by)}\overrightarrow\psi(\by, T) \mathrm{d}\by&=\frac{\dd\mu_{\star}}{\dd\bx}, \ \ \ \overrightarrow\psi(\by, T)\int_{\mathbb{R}^d} e^{-c_{\varepsilon}(\bx, \by)}\overleftarrow\varphi(\bx, 0) \mathrm{d}\bx=\frac{\dd\nu_{\star}}{\dd\by}, \\
\end{split}
\end{equation}
where 
$c_{\varepsilon}(\cdot, \cdot)$ is the potential energy of the kernel $K_{\varepsilon}(0, \bx, T, \by)$ defined as follows
\begin{equation}
\begin{split}\label{def_cost_varphi_psi}
    &\qquad c_{\varepsilon}(\bx, \by)=-\log \text{Ker}_{\varepsilon}(0, \bx, T, \by).
\end{split}
\end{equation}

Combining Eq.\eqref{score_primal} and Eq.\eqref{CH_transform} and replacing $\bu$  with $\overrightarrow\psi$ (Eq.\eqref{optimal_u}) into the forward diffusion process \eqref{control_diffusion}, we have
\begin{align}\label{forward_diffusion_simplified}
    \mathrm{d} \bx_t&=\left[\bbf(\bx_t, t)+2\varepsilon g(t)^2\nabla\log\overrightarrow\psi(\bx_t, t)\right]\mathrm{d}t+\sqrt{2\varepsilon}g(t) \mathrm{d} \mathbf{w}_t.
\end{align}

Following \citet{Anderson82, forward_backward_SDE} to reverse the forward diffusion \eqref{forward_diffusion_simplified}, we obtain the backward diffusion:
\begin{equation*}
\begin{split}
    \mathrm{d} \bx_t&=\big[\bbf(\bx_t, t)+2\varepsilon g(t)^2\nabla\log\overrightarrow\psi(\bx_t, t)- 2\varepsilon g(t)^2\nabla \log \rho^{\star}(\bx_t, t)\big]\mathrm{d}t+\sqrt{2\varepsilon}g(t)\mathrm{d} \mathbf{w}_t\\
    &=\left[\bbf(\bx_t, t)-2\varepsilon g(t)^2\nabla\log\overleftarrow\varphi(\bx_t, t)\right]\mathrm{d}t+\sqrt{2\varepsilon}g(t) \mathrm{d} \bar{\mathbf{w}}_t,
\end{split}
\end{equation*}
where the second equality follows since $\log \overrightarrow\psi(\cdot, t) + \log\overleftarrow\varphi(\cdot, t)=\log\rho^{\star}(\cdot, t)$ by Eq.\eqref{CH_transform_b}.

\begin{proposition}
Given the score functions $(\overrightarrow\psi, \overleftarrow\varphi)$ that solve the \emph{Schr\"{o}dinger system} 
 \begin{align*}
    \begin{cases}
    \frac{\partial \overrightarrow\psi}{\partial t}+\langle \nabla \overrightarrow\psi, \bbf \rangle +\varepsilon g^2\Delta\overrightarrow\psi=0 \\[3pt]
    \frac{\partial \overleftarrow\varphi}{\partial t}+\nabla\cdot (\overleftarrow\varphi \bbf)-\varepsilon g^2 \Delta \overleftarrow\varphi=0.
    \end{cases}
    \text{\ \ s.t. } \overrightarrow\psi(\bx, 0) \overleftarrow\varphi(\bx, 0) = \frac{\dd\mu_{\star}}{\dd\bx},\ \ ~\overrightarrow\psi(\by, T) \overleftarrow\varphi(\by, T) = \frac{\dd\nu_{\star}}{\dd\by}.
\end{align*}
    
\emph{Schr\"{o}dinger system} yields the desired forward-backward stochastic differential equation (FB-SDE)
\begin{subequations}
\begin{align}
\mathrm{d} \bx_t&=\left[\bbf(\bx_t, t) + 2\varepsilon g(t)^2\nabla\log\overrightarrow\psi(\bx_t, t)\right]\mathrm{d}t+\sqrt{2\varepsilon}g(t) \mathrm{d} \mathbf{w}_t, \ \  \bx_0\sim \mu_{\star}\\
\mathrm{d} \bx_t&=\left[\bbf(\bx_t, t) -2\varepsilon g(t)^2 \nabla\log\overleftarrow\varphi(\bx_t, t)\right]\mathrm{d}t+\sqrt{2\varepsilon} g(t) \mathrm{d} \bar{\mathbf{w}}_t,\ \  \bx_T \sim \nu_{\star}.
\end{align}\label{FB-SDE}
\end{subequations}
Setting $\varepsilon=\frac{1}{2}$ recovers the FB-SDE \eqref{bs_sde_formulation}. Part of the above derivation is standard and we present it here for the sake of self-containedness.

\end{proposition}

\subsection{An Important Property of Static Schr\"{o}dinger Bridge Problem (SBP)}
\label{static_SB_property}
\begin{lemma}[Structure Property of Static SBP \citep{Compute_OT, Nutz22_note}]\label{solution_property}
Let $\mathcal{G}\backsim \mu_{\star}\otimes \nu_{\star}$, where $\backsim$ is the cyclical invariant property \citep{Nutz_22_func} and $\mathrm{d}\mathcal{G} \propto e^{-c_{\varepsilon}}\mathrm{d} (\mu_{\star} \otimes \nu_{\star})$. %defined in section 2.2 \citep{Nutz22_note}. 
Suppose there is a unique coupling $\pi_{\star}$ for the static SBP $$\pi_{\star} = \argmin_{\pi\in \Pi(\mu_{\star}, \nu_{\star})} \text{KL}(\cdot| \mathcal{G}).$$

\begin{itemize}
    \item There exist measurable functions $ \varphi_{\star}$ and $ \psi_{\star}$ such that 
    \begin{equation}\label{formulation}
    \frac{\mathrm{d}\pi_{\star}}{\mathrm{d} \mathcal{G}}=e^{ \varphi_{\star}\oplus \psi_{\star}},
    \end{equation}
    where $ \varphi_{\star},  \psi_{\star}: \mathbb{R}^d\rightarrow [-\infty, \infty)$ are known as the Schr\"{o}dinger potentials. The $\oplus$ operator is defined as $( \varphi_{\star}\oplus  \psi_{\star})(\bx, \by)= \varphi_{\star}(\bx)+ \psi_{\star}(\by)$ for functions $ \varphi_{\star}$ and $ \psi_{\star}$. The summation of potentials is unique.
    
    \item Suppose there is a solution $\widehat{\pi}_0$ that admits a density formulation
    \begin{equation*}
    \frac{\mathrm{d}\widehat{\pi}_0}{\mathrm{d} \mathcal{G}}=e^{ \varphi_{\star}\oplus \psi_{\star}},
    \end{equation*}
    for functions $ \varphi_{\star}: \mathbb{R}^d\rightarrow [-\infty, \infty)$ and $ \psi_{\star}: \mathbb{R}^d\rightarrow [-\infty, \infty)$, it follows that $\widehat{\pi}_0$ is the Schr\"{o}dinger bridge.
\end{itemize}

\end{lemma}
 
% \subsection{Connections between Score Functions in FB-SDE and Potential Functions in EOT}
% \label{score_potential_connection}
% To build connections between optimal transport and the SOC problem, we consider the following transformations \Wei{double check here}
% \begin{subequations}
% \begin{align}
% \overleftarrow\varphi(\bx, 0)=e^{ \varphi_{\star}(\bx)}\frac{\dd \mu_{\star}}{\dd \bx}\label{OT_sbp_transform_1}\\
% \quad \overrightarrow\psi(\by, T)=e^{ \psi_{\star
%     }(\by)} \frac{\dd \nu_{\star}}{\dd \by},\label{OT_sbp_transform_2_supp}
% \end{align}\label{OT_sbp_transform}
% \end{subequations}
% where $\nabla\log\overrightarrow\psi$ and $\nabla\log\overleftarrow\varphi$ are the forward and backward score functions and $\psi_{\star}$ and $\varphi_{\star}$ are the Schr\"{o}dinger potentials.

\subsection{Schr\"odinger Equations}\label{derive_schrodinger_eqn}

For any set $A\subset\mathbb{R}^d$, take the integral for the coupling $\Pi$ on with respect to the marginals
\begin{equation*}
\begin{split}
    \label{marginal_projections}
\int_A \mu_{\star}(\mathrm{d}\bx)=\iint_{A\times \mathbb{R}^d}  \pi_{\star}(\mathrm{d}\bx, \mathrm{d}\by)=\iint_{A\times \mathbb{R}^d} \underbrace{e^{\varphi_{\star}(\bx)+ \psi_{\star}(\by)}\mathcal{G}(\mathrm{d}\bx, \mathrm{d}\by)}_{\text{by Eq.\eqref{formulation}}}=\int_{A}\underbrace{\int_{ \mathbb{R}^d}e^{\varphi_{\star}(\bx)+ \psi_{\star}(\by)-c_{\varepsilon}(\bx, \by)}\nu_{\star}(\mathrm{d}\by)}_{ \text{the first marginal of $\pi_{\star}$ is $\mu_{\star}$}}\mu_{\star}(\mathrm{d}\bx)\\
\int_A \nu_{\star}(\mathrm{d}\by)=\iint_{\mathbb{R}^d\times A}  \pi_{\star}(\mathrm{d}\bx, \mathrm{d}\by)=\iint_{\mathbb{R}^d\times A} \underbrace{e^{\varphi_{\star}(\bx)+ \psi_{\star}(\by)}\mathcal{G}(\mathrm{d}\bx, \mathrm{d}\by)}_{\text{by Eq.\eqref{formulation}}}=\int_{A}\underbrace{\int_{ \mathbb{R}^d}e^{\varphi_{\star}(\bx)+ \psi_{\star}(\by)-c_{\varepsilon}(\bx, \by)}\mu_{\star}(\mathrm{d}\bx)}_{\text{the second marginal of $\pi_{\star}$ is $\nu_{\star}$}}\nu_{\star}(\mathrm{d}\by),
\end{split}
\end{equation*}
which implies the \emph{Schr\"{o}dinger equations}:
\begin{subequations}
\begin{align}
\int_{\mathbb{R}^d} e^{ \varphi_{\star}(\bx)+ \psi_{\star}(\by)-c_{\varepsilon}(\bx,\by)}\nu_{\star} (\mathrm{d}\by)=1 \quad \mu_{\star}\text{-}a.s., \qquad \int_{\mathbb{R}^d} e^{ \varphi_{\star}(\bx)+ \psi_{\star}(\by)-c_{\varepsilon}(\bx,\by)}\mu_{\star} (\mathrm{d}\bx)=1 \quad \nu_{\star}\text{-}a.s. 
\end{align}\label{SE12_supp}
\end{subequations}

In other words, if $( \varphi_{\star},  \psi_{\star})$ are Schr\"{o}dinger potentials, then $( \varphi_{\star},  \psi_{\star})$ is a solution of Eq.(\ref{SE12_supp}).

\section{Convergence Analysis for the Marginals}
\label{convergence}
\paragraph{Notations} $\pi_k:=(\mu_k, \nu_k)$ is the coupling at the $k$-th iteration with marginals $\mu_k$ and $\nu_k$; $\pi_{\star}$ is the optimal coupling with target marginals $\mu_{\star}$ and $\nu_{\star}$;  $\mu_{\star, k}$ and $\nu_{\star, k}$ denote the $\epsilon$-approximation of $\mu_{\star}$ and $\nu_{\star}$ via approximations. $ \varphi_{k_1}$ and $ \psi_{k_2}$ denote the potential functions and the coupling can be represented as $\pi_{k_1+k_2}:=\pi( \varphi_{k_1},  \psi_{k_2})$ by Lemma \ref{solution_property}.

We are interested in the convergence of the marginals. However, computing the integrals in Algorithm \ref{sinkhorn} is too expensive. To handle this issue, we first present the exact formulation of the approximate IPF algorithm in Algorithm \ref{sinkhorn_supp}. 

\begin{algorithm}[ht]\caption{One iteration of approximate IPF (aIPF) based on approximate measures $\mu_{\star, k}$ and $\nu_{\star, k}$. In view of Eq.\eqref{marginal_eqn}, it differs from Algorithm \ref{sinkhorn} in that $\mu_{\star}$ (or $\nu_{\star}$) is changed to $\mu_{\star, k}$ (or $\nu_{\star, k}$) for an exact equality.}\label{sinkhorn_supp}
\begin{align}
     \psi_k(\by)&=-\log \int_{\mathbb{R}^d} e^{ \varphi_k(\bx)-c_{\varepsilon}(\bx,\by)}\mu_{\star, k}(\mathrm{d}\bx),\qquad
     \varphi_{k+1}(\bx)=-\log \int_{\mathbb{R}^d} e^{ \psi_k(\by)-c_{\varepsilon}(\bx,\by)}\nu_{\star, k}(\mathrm{d}\by)\notag.
\end{align}
\end{algorithm}

In such a case, by Lemma \ref{solution_property}, the approximate couplings $\pi_{2k}$ and $\pi_{2k-1}$ follows that
\begin{equation*}
\begin{split}
\label{approx_marginal_def}
    \pi_{2k}:=\pi( \varphi_k,  \psi_k), \quad \pi_{2k-1}:=\pi( \varphi_k, \psi_{k-1}), \quad \forall k\geq 0,
\end{split}
\end{equation*}
where the approximate potential functions $ \varphi_k$ and $ \psi_k$ (and $ \psi_{k-1}$) are associated with the couplings $\pi_{2k}=(\mu_{2k}, \nu_{2k})$ and $\pi_{2k-1}=(\mu_{2k-1}, \nu_{2k-1})$ as follows 
\begin{align}
    \mathrm{d}\pi_{2k}=e^{\varphi_k\oplus \psi_k-c_{\varepsilon}}\mathrm{d}(\mu_{\star, k}\otimes \nu_{\star, k}), \quad  \mathrm{d}\pi_{2k-1}=e^{\varphi_k\oplus \psi_{k-1}-c_{\varepsilon}}\mathrm{d}(\mu_{\star, k}\otimes \nu_{\star, k-1}),\label{approx_potential}
\end{align}
where $\varphi_{-1}= \psi_0:=0$ and thus $\pi_{-1}=\mathcal{G}$.

\begin{lemma}\label{iterative_lemma}
For all $k \geq 0$ and $n \geq 0$, we have

\text{(i)} $\text{KL}(\pi_{2k}|\pi_{2k-1})$ and $\text{KL}(\pi_{2k+1}|\pi_{2k})$ satisfy the following equations
\begin{equation*}
\begin{split}
    \text{KL}(\pi_{2k}|\pi_{2k-1})&=\nu_{\star, k}( \psi_k+ \log d {\nu_{\star, k}}-\psi_{k-1} - \log d {\nu_{\star, k-1}})\\
    \text{KL}(\pi_{2k+1}|\pi_{2k})&=\mu_{\star, k+1}( \varphi_{k+1}+\log d {\mu_{\star, k+1}}- \varphi_k - \log d {\mu_{\star, k}}),
\end{split}
\end{equation*}

\text{(ii)} the summation of $\text{KL}(\pi_{2k}|\pi_{2k-1})$ and $\text{KL}(\pi_{2k+1}|\pi_{2k})$ follows that
\begin{equation}\label{key_inequality}
   \sum_{k=0}^{n-1} \text{KL}(\pi_{2k+1}|\pi_{2k})\leq  \mu_{\star}(\varphi_n) + O(n\epsilon), \quad  \sum_{k=0}^{n} \text{KL}(\pi_{2k}|\pi_{2k-1})\leq  \nu_{\star}(\psi_n)+ O(n\epsilon).
\end{equation} 
\end{lemma}

\begin{proof}
(i) In view of Eq.\eqref{approx_potential}, we have
\begin{equation}
\begin{split}
\label{diff_mu_nu}
    \text{KL}(\pi_{2k}|\pi_{2k-1})=\iint_{\mathbb{R}^d \times \mathbb{R}^d} \log \frac{\mathrm{d}\pi_{2k}}{\mathrm{d}\pi_{2k-1}} \mathrm{d}\pi_{2k}&= \iint_{\mathbb{R}^d \times \mathbb{R}^d} \big( \psi_k+ \log d {\nu_{\star, k}}-\psi_{k-1} -\log d {\nu_{\star, k-1}}  \big)\mathrm{d}\pi_{2k}\\
    &=\int_{\mathbb{R}^d} \big( \psi_k+ \log d {\nu_{\star, k}}-\psi_{k-1} - \log d {\nu_{\star, k-1}} \big) \mathrm{d}\nu_{\star, k}.
\end{split}
\end{equation}

Similarly, we can show  $\text{KL}(\pi_{2k+1}|\pi_{2k})=\int_{\mathbb{R}^d} \big( \varphi_{k+1}+\log d {\mu_{\star, k+1}}- \varphi_k - \log d {\mu_{\star, k}} \big) \mathrm{d}\mu_{\star, k+1}$.

(ii) The non-negative property is clear; Summing up items in Eq.\eqref{diff_mu_nu}, we have
\begin{equation*}
\begin{split}
\small
\label{diff_mu_nu_sum}
    \sum_{k=0}^{n}\text{KL}(\pi_{2k}|\pi_{2k-1})&=\sum_{k=0}^{n}\int_{\mathbb{R}^d} \bigg( \psi_k+ \log d {\nu_{\star, k}}-\psi_{k-1} - \log d {\nu_{\star, k-1}} \bigg) \mathrm{d}\nu_{\star, k}\\
    &= \int_{\mathbb{R}^d} \bigg( \psi_n+ \log d {\nu_{\star, n}}-\psi_{-1} - \log d {\nu_{\star, -1}} \bigg) \mathrm{d} \nu_{\star} + \\
    &\qquad +\sum_{k=0}^{n}\int_{\mathbb{R}^d} \bigg( \psi_k+ \log d {\nu_{\star, k}}-\psi_{k-1} - \log d {\nu_{\star, k-1}} \bigg) \mathrm{d}(\nu_{\star, k}-\nu_{\star})\\
    &\leq \nu_{\star}(\psi_n) + O(n\epsilon),\\
\end{split}
\end{equation*}
where the $O(\epsilon)$ approximation follows from Lemma \ref{control_diff_measure} and \ref{control_diff_measure_v2} based on Assumption \ref{Dissipativity}, Assumption \ref{approx_score}, and Assumption \ref{smooth_potential}. The other half can be shown similarly. \qed.

\end{proof}

In the following, we present an important result for the convergence analysis.
\begin{proposition}
\begin{equation}\label{summable_iter}
     \sum_{k=0}^n \text{KL}(\pi_k|\pi_{k-1}) \leq \text{KL}( \pi_{\star}|\mathcal{G}) - \text{KL}( \pi_{\star}|\pi_n) + O(n\epsilon).
\end{equation}
\end{proposition}

\begin{proof}
By Lemma \ref{iterative_lemma} and Eq.\eqref{key_inequality}, we can easily verify that 
\begin{equation}
\begin{split}\label{upper_bound_1}
    \sum_{k=0}^{2n}\text{KL}(\pi_k|\pi_{k-1})\leq  \mu_{\star}(\varphi_n) + \nu_{\star}(\psi_n)+O(n\epsilon).
\end{split}
\end{equation}
From another perspective, we know that 
\begin{equation}\label{eqn_11}
    \text{KL}( \pi_{\star}|\mathcal{G}) -  \text{KL}( \pi_{\star}|\pi_{2n})=\E^{ \pi_{\star}} [\log (\mathrm{d}\pi_{2n} / \mathrm{d}\mathcal{G})]=\E^{ \pi_{\star}}[\varphi_n + \psi_n]+ O(\epsilon)= \mu_{\star}(\varphi_n) + \nu_{\star}(\psi_n) + O(\epsilon).
\end{equation}
Combining \eqref{eqn_11} and \eqref{upper_bound_1} concludes the proof for $2n$. Similarly, we can also show the proof for $2n-1$.\qed
\end{proof}

The above result shows $\text{KL}(\pi_k|\pi_{k-1})$ decays (approximately) fast as $k\rightarrow \infty$, which implies a convergence of the marginals.

\begin{lemma}\label{corollary_summable}
\begin{equation}
\begin{split}
\label{one_key_part}
    \text{KL}(\mu_k| \mu_{\star,\lfloor \frac{k+1}{2}\rfloor})&+\text{KL}(\nu_k| \nu_{\star,\lfloor \frac{k}{2}\rfloor}) \leq \text{KL}(\pi_k | \pi_{k-1}),
\end{split}
\end{equation}
where $\lfloor \cdot \rfloor$ is the floor function; the sum of RHS from $1$ to $n$ is upper bounded by a fixed constant
\begin{equation*}
    \sum_{k\geq 1}^n \text{KL}(\pi_k | \pi_{k-1})\leq \text{KL}( \pi_{\star}|\mathcal{G}) - \text{KL}(\pi_0|\mathcal{G})+O(n\epsilon) \leq \text{KL}( \pi_{\star}|\mathcal{G})+ O(n\epsilon).
\end{equation*}
\end{lemma}

\begin{proof} The first inequality holds by the data processing inequality in Lemma \ref{data_processing} for both even and odd $k$.

The second one follows by Eq.\eqref{summable_iter} and $\pi_{-1}=\mathcal{G}$.\qed
\end{proof}

The approximate IPF algorithm also yields other important theoretical properties.
\begin{lemma}\label{lemma_RN_derivitive}

\begin{equation*}
\begin{split}
    &\frac{\mathrm{d} \mu_{2k}}{\mathrm{d}\mu_{\star, k}}=e^{ \varphi_k -  \varphi_{k+1}}, \quad \frac{\mathrm{d} \nu_{2k-1}}{\mathrm{d}\nu_{\star, k-1}}=e^{\psi_{k-1} -  \psi_k}, \quad \frac{\mathrm{d} \mu_{2k+1}}{\mathrm{d}\mu_{\star, k+1}}=1, \quad \frac{\mathrm{d} \nu_{2k}}{\mathrm{d}\nu_{\star, k}}=1.\\
\end{split}
\end{equation*}
Moreover, we have
\begin{subequations}
\begin{align}
&\text{KL}(\mu_{2k}|\mu_{\star, k})= \mu_{2k}( \varphi_k- \varphi_{k+1})=\text{KL}(\pi_{2k}|\pi_{2k+1})+O(\epsilon)\label{eqn_relation_a}\\
&\text{KL}(\nu_{2k-1}|\nu_{\star, k-1})= \nu_{2k-1}( \varphi_{k-1}-\varphi_k)=\text{KL}(\pi_{2k-1}|\pi_{2k})+O(\epsilon)\label{eqn_relation_b}\\
&\text{KL}(\mu_{\star, k+1}|\mu_{2k})= \mu_{\star, k+1}( \varphi_{k+1}- \varphi_k)+O(\epsilon)= \text{KL}(\pi_{2k+1}|\pi_{2k})+O(\epsilon)\label{eqn_relation_c}\\
&\text{KL}(\nu_{\star, k}|\nu_{2k-1})= \nu_{\star, k}( \psi_k-\psi_{k-1})+O(\epsilon)= \text{KL}(\pi_{2k}|\pi_{2k-1})+O(\epsilon)\label{eqn_relation_d}\\
&\text{KL}(\mu_{2k+2}|\mu_{\star, k+1})-\text{KL}(\mu_{2k+2}|\mu_{2k})=\mu_{2k+2}(\varphi_k- \varphi_{k+1})+O(\epsilon)\label{eqn_relation_e}\\
&\text{KL}(\nu_{2k+1}|\nu_{\star, k+1})-\text{KL}(\nu_{2k+1}|\nu_{2k-1})=\nu_{2k+1}(\psi_{k-1}-\psi_{k+1})+O(\epsilon).\label{eqn_relation_f}
\end{align}\label{eqn_relation}
\end{subequations}
\end{lemma}

\begin{proof} By Eq.\eqref{approx_potential}, we take the integral with respect to the second marginal to obtain the marginal density
\begin{equation}
\begin{split}\label{1st_relation_supp_1st}
    \frac{\mathrm{d} \mu_{2k}}{\mathrm{d}\mu_{\star, k}}&=\int_{\mathbb{R}^d} \frac{\mathrm{d}\pi_{2k}}{\mathrm{d}(\mu_{\star, k}\otimes \nu_{\star, k})}(\bx,\by) \nu_{\star, k}(\mathrm{d}\by)=\int_{\mathbb{R}^d}e^{ \varphi_k(\bx)+ \psi_k(\by)-c_{\epsilon}(\bx, \by)} \nu_{\star, k}(\mathrm{d}\by) \\
    &=e^{ \varphi_k (\bx)}\int_{\mathbb{R}^d}e^{ \psi_k(\by)-c_{\epsilon}(\bx, \by)} \nu_{\star, k}(\mathrm{d}\by)=e^{ \varphi_k (\bx)} e^{- \varphi_{k+1}(\bx)},
\end{split}
\end{equation}
where the last equality follows by Algorithm \ref{sinkhorn_supp}.  $\frac{\mathrm{d} \mu_{2k+1}}{\mathrm{d}\mu_{\star, k+1}}=1$ follows directly due to the definition in Eq.\eqref{marginal_eqn}. Next, by Eq.\eqref{1st_relation_supp_1st} and Assumption \ref{approx_score}, we show that 
\begin{equation*}
    \text{KL}(\mu_{\star, k+1}|\mu_{2k})=-\int_{\mathbb{R}^d} \log \frac{\mathrm{d} \mu_{2k}}{\mathrm{d}\mu_{\star, k}}\frac{\mathrm{d} \mu_{\star, k}}{\mathrm{d}\mu_{\star, k+1}}\mathrm{d}\mu_{\star, k+1}=\mu_{\star, k+1}( \varphi_{k+1}- \varphi_k)+O(\epsilon),
\end{equation*}
where the last item follows by Assumption \ref{approx_score} and Lemma \ref{control_diff_measure_v2}.

By Eq.\eqref{diff_mu_nu} and Assumption \ref{approx_score}, we can easily show the inequality in Eq.\eqref{eqn_relation_d}. The rest can be proved similarly.\qed
\end{proof}

Before we finally present the final result, we provide some elementary entropy calculations. 
\begin{proposition}\label{monotone_property}
For any $k\geq 0$, 
\begin{align}
    &\text{KL}(\nu_{\star, k+1}|\nu_{2k+1})\geq \text{KL}(\mu_{2k+2}|\mu_{\star, k})+\text{KL}(\pi_{2k+2}|\pi_{2k})-\text{KL}(\mu_{2k+2}|\mu_{2k})+O(\epsilon)\geq \text{KL}(\mu_{2k+2}|\mu_{\star, k})+O(\epsilon),\label{1st_relation_supp}\\
    &\text{KL}(\nu_{\star, k+1}|\nu_{2k+1})\leq \text{KL}(\mu_{2k}|\mu_{\star, k})-\text{KL}(\pi_{2k}|\pi_{2k+2})+O(\epsilon)\leq \text{KL}(\mu_{2k}|\mu_{\star, k})+O(\epsilon)\label{2nd_relation}.
\end{align}

Moreover, an approximately monotone-decreasing property is shown as follows
\begin{equation*}
\begin{split}
\label{monotone_dec_a}
    \text{KL}(\mu_{2k}|\mu_{\star, k}) +O(\epsilon) &\geq \text{KL}(\nu_{\star, k+1}|\nu_{2k+1})\\
    \text{KL}(\nu_{\star, k+1}|\nu_{2k+1})+O(\epsilon)&\geq \text{KL}(\mu_{2k+2}|\mu_{\star, k+1})\\
    \text{KL}(\mu_{2k+2}|\mu_{\star, k+1})+O(\epsilon)&\geq \text{KL}(\nu_{\star, k+2}|\nu_{2k+3})\\
    \cdots
\end{split}
\end{equation*}
which further implies that $\text{KL}(\mu_{2k}|\mu_{\star, k})$ and $\text{KL}(\nu_{\star, k+1}|\nu_{2k+1})$ are approximately monotone decreasing as $k\rightarrow \infty$. 
\end{proposition}

\begin{proof}

To prove Eq.\eqref{2nd_relation}, recall the definitions of $\pi_{2k}$ and $\pi_{2k+2}$ in Eq.\eqref{approx_potential}, we first observe that
\begin{align}
    \text{KL}(\pi_{2k}|\pi_{2k+2})&=\iint_{\mathbb{R}^d\times \mathbb{R}^d} \log \frac{\pi_{2k}}{\pi_{2k+2}}\mathrm{d}\pi_{2k}\notag\\
    &=\iint_{\mathbb{R}^d\times \mathbb{R}^d} \log \frac{e^{ \varphi_k+ \varphi_k-c_{\epsilon}} \mathrm{d}\mu_{\star, k}\otimes \nu_{\star, k}}{e^{ \varphi_{k+1}+ \varphi_{k+1}-c_{\epsilon}} \mathrm{d}\mu_{\star, k+1}\otimes \nu_{\star, k+1}}\mathrm{d}\pi_{2k}\notag\\
    &= \iint_{\mathbb{R}^d\times \mathbb{R}^d} ( \varphi_k- \varphi_{k+1} +  \psi_k-\psi_{k+1})\mathrm{d}\pi_{2k} +\iint_{\mathbb{R}^d\times \mathbb{R}^d} \log \frac{\mathrm{d}\mu_{\star, k}\otimes \nu_{\star, k}}{\mathrm{d}\mu_{\star, k+1}\otimes \nu_{\star, k+1}}\mathrm{d}\pi_{2k}\notag\\
    &\leq \mu_{2k}( \varphi_k- \varphi_{k+1})+\nu_{\star, k}( \psi_k-\psi_{k+1})+O(\epsilon)\notag\\
    &\leq\mu_{2k}( \varphi_k- \varphi_{k+1})+\nu_{\star, k+1}( \psi_k-\psi_{k+1})+O(\epsilon)\notag\\
    &\leq \text{KL}(\mu_{2k}|\mu_{\star, k})-\text{KL}(\nu_{\star, k+1}|\nu_{2k+1})\label{u_2t_larger_nu}+O(\epsilon),
\end{align}
where the first inequality follows by Lemma \ref{control_diff_measure_v2}; the second inequality follows by Lemma \ref{control_diff_measure}; the last inequality follows by Eq.\eqref{eqn_relation_a} and Eq.\eqref{eqn_relation_d} in Lemma \ref{lemma_RN_derivitive} and the approximate error $O(\epsilon)$ follows from Assumption \ref{approx_score}. 

To show Eq.\eqref{1st_relation_supp}, apply Eq.\eqref{eqn_relation_e} and Eq.\eqref{eqn_relation_d} in Lemma \ref{lemma_RN_derivitive} again
\begin{align}
    \text{KL}(\pi_{2k+2}|\pi_{2k})&=\iint_{\mathbb{R}^d\times \mathbb{R}^d} \bigg( \varphi_{k+1}-\varphi_k + \psi_{k+1}-\psi_k +\log \frac{\mathrm{d}\mu_{\star, k+1}\otimes \nu_{\star, k+1}}{\mathrm{d}\mu_{\star, k}\otimes \nu_{\star, k}}\bigg)\mathrm{d}\pi_{2k+2}\notag\\
    &\leq \mu_{2k+2}( \varphi_{k+1}-\varphi_k) + \nu_{\star, k+1}(\psi_{k+1}-\psi_k) + O(\epsilon)\notag\\
    &\leq\text{KL}(\mu_{2k+2}|\mu_{2k})-\text{KL}(\mu_{2k+2}|\mu_{\star, k+1})+\text{KL}(\nu_{\star, k+1}|\nu_{2k+1}) + O(\epsilon)\label{nu_above_nu_2t}.
\end{align}
 
Moreover, by the data processing inequality \citep{Nutz22_note}, we have $\text{KL}(\pi_{2k+2}|\pi_{2k})\geq \text{KL}(\mu_{2k+2}|\mu_{2k})$. Combining Eq.\eqref{nu_above_nu_2t} shows $\text{KL}(\nu_{\star, k+1}|\nu_{2k+1})+O(\epsilon)\geq \text{KL}(\mu_{2k+2}|\mu_{\star, k+1})$. Eq.\eqref{u_2t_larger_nu} naturally leads to $\text{KL}(\mu_{2k}|\mu_{\star, k})+O(\epsilon)\geq \text{KL}(\nu_{\star, k+1}|\nu_{2k+1})$. The approximate decreasing property and the inequalities in Eq.\eqref{1st_relation_supp} and Eq.\eqref{2nd_relation} are proved. \qed

\end{proof}

\subsection{Proof of Theorem \ref{main_theorem}}\label{main_result_proof}
Finally, we are able to prove the main result, that is, the sublinear convergence for the marginals.

\begin{proof}
Recall in Lemma \ref{corollary_summable} that $\sum_{k \geq 1}^n \text{KL}(\pi_k|\pi_{k-1})\leq \text{KL}( \pi_{\star}|\mathcal{G})-\text{KL}(\pi_0|\mathcal{G})+O(n\epsilon)$. By Eq.\eqref{one_key_part}, we have $\text{KL}(\mu_{2k}|\mu_{\star, k})\leq \text{KL}(\pi_{2k}|\pi_{2k-1})$; by Eq.\eqref{eqn_relation_d}, $\text{KL}(\nu_{\star, k}|\nu_{2k-1})\leq \text{KL}(\pi_{2k}|\pi_{2k-1})+O(\epsilon)$. It follows that
\begin{equation*}
    \sum_{k \geq 1}^n\text{KL}(\mu_{2k}|\mu_{\star, k})\leq \text{KL}( \pi_{\star}|\mathcal{G})-\text{KL}(\pi_0|\mathcal{G})+O(n\epsilon),\quad \sum_{k \geq 1}^n\text{KL}(\nu_{\star, k}|\nu_{2k-1})\leq \text{KL}( \pi_{\star}|\mathcal{G})-\text{KL}(\pi_0|\mathcal{G})+O(n\epsilon).
\end{equation*}
Combining the approximately monotone decreasing property in Proposition \ref{monotone_property} and Lemma \ref{unbounded_cost}, we have 
\begin{equation*}
\begin{split}
    \text{KL}(\mu_{2k}|\mu_{\star, k})\leq \frac{\text{KL}( \pi_{\star}|\mathcal{G})-\text{KL}(\pi_0|\mathcal{G})}{k} + O(\epsilon^{\frac{1}{2}}+{k^{\frac{1}{2}}\epsilon}),\ \  \text{KL}(\nu_{\star, k}|\nu_{2k-1})\leq \frac{\text{KL}( \pi_{\star}|\mathcal{G})-\text{KL}(\pi_0|\mathcal{G})}{k} + O(\epsilon^{\frac{1}{2}}+{k^{\frac{1}{2}}\epsilon}).
\end{split}
\end{equation*}

Similar results hold for $\text{KL}(\mu_{\star, k}|\mu_{2k})$ and $\text{KL}(\nu_{2k-1}|\nu_{\star, k-1})$ by Eq.\eqref{eqn_relation}. Further combining Lemma \ref{control_diff_measure} and Lemma \ref{control_diff_measure_v2}, we can complete the first half of the proof as follows
\begin{equation*}
    \text{KL}(\mu_{2k}|\mu_{\star})=\int \log \frac{\mathrm{d} \mu_{2k}}{\mathrm{d} \mu_{\star}}\mathrm{d} \mu_{2k}\leq \int \log \frac{\mathrm{d} \mu_{2k}}{ \mathrm{d}\mu_{\star, k}} + \log \frac{\mathrm{d}\mu_{\star, k}}{\mathrm{d} \mu_{\star}} \mathrm{d} \mu_{2k}\leq \text{KL}(\mu_{2k}|\mu_{\star, k}) + O(\epsilon).
\end{equation*}

The rest can be proved similarly.
\qed
\end{proof}

\subsection{Auxiliary Results}

\begin{lemma}
\label{control_diff_measure}
Assume we have a probability density $\rho(\bx)= C_{\text{Norm}} e^{-U(\bx)}$ defined on $\mathbb{R}^d$ and an approximate density  $\widetilde \rho(\bx)= C_{\text{Norm}} e^{-\widetilde U(\bx)}$, where the energy functions $U$ and $\widetilde U$ are differentiable and satisfy
    \begin{equation}\label{grad_l_inf_norm}
    \begin{split}
        \|\nabla \widetilde U - \nabla U\|_{\infty} \leq \epsilon.
    \end{split}
    \end{equation}
    Moreover, the density $\rho$ satisfies the dissipative assumption \ref{Dissipativity}, then for an Lipschitz smooth function $f$, we have that 
    \begin{equation*}
        \bigg|\int_{\mathbb{R}^d} f(\bx)\widetilde \rho(\bx) \mathrm{d} \bx - \int_{\mathbb{R}^d} f(\bx)\rho(\bx) \mathrm{d} \bx\bigg|\leq O(\epsilon),
    \end{equation*}
    where the big-O notation mainly depends on $m_{\text{ds}}$ in assumption \ref{Dissipativity}, the smoothness assumption \ref{smooth_potential}, and the dimension $d$.
\end{lemma}

\begin{proof}
Recall that for any $\bx, \by\in \mathbb{R}^d$,  $ U(\bx) - U(\by)=\int_0^1 \frac{\dd}{\dd t} U(t\bx + (1-t)\by) = \int_0^1 \langle \bx -\by, \nabla U(t\bx +(1-t) \by) \rangle \dd t.$
% \begin{equation*}
%     U(\bx) - U(\by)=\int_0^1 \frac{\dd}{\dd t} U(t\bx + (1-t)\by) = \int_0^1 \langle \bx -\by, \nabla U(t\bx +(1-t) \by) \rangle \dd t.
% \end{equation*}
Since $\rho$ and $\widetilde \rho$ are probability densities, there is a $\bx_0$ such that $U(\bx_0)=\widetilde U(\bx_0)$. By the differentiability of $U$ and $\widetilde U$, we have 
\begin{equation}
    \begin{split}\label{diff_U}
    |\widetilde U(\bx)- U(\bx)|&=\bigg|\int_{0}^{1} \langle \bx-\bx_0, \nabla \widetilde U(\cdot) - \nabla U(\cdot) \mathrm{d} t \bigg| \\
    &\leq \int_0^1 \|\bx -\bx_0\|_2\cdot \big \|\nabla \widetilde U(\cdot) - \nabla U(\cdot)\big \|_2 \mathrm{d}t\\
    &\leq \epsilon (\|\bx\|_2 + \|\bx_0\|_2),
    \end{split}
    \end{equation}
where the first inequality follows by Cauchy Schwarz inequality; the second inequality follows by Eq.\eqref{grad_l_inf_norm}; we use $\nabla U(\cdot)$ and $\nabla \widetilde U(\cdot)$ for convenience because the value holds for any $t$. As such, we have
    \begin{equation*}
    \begin{split}
        \bigg|\int_{\mathbb{R}^d} f(\bx)  \widetilde \rho(\bx) \mathrm{d} \bx - \int_{\mathbb{R}^d} f(\bx) \rho(\bx) \mathrm{d} \bx\bigg|&=C_{\text{Norm}}\bigg|\int_{\mathbb{R}^d} e^{- U(\bx)}f(\bx) \big(e^{U(\bx)-\widetilde U(\bx)} -1\big) \mathrm{d} \bx\bigg|\\
        &\leq C_{\text{Norm}}\int_{\mathbb{R}^d}  e^{- U(\bx)}|f(\bx)|\cdot \big |e^{U(\bx)-\widetilde U(\bx)} -1\big | \mathrm{d} \bx\\
        &\leq C_{\text{Norm}}\int_{\mathbb{R}^d}e^{- U(\bx)} |f(\bx)| \cdot \big |e^{\epsilon (\|\bx\|_2+\|\bx_0\|_2)} -1\big | \mathrm{d} \bx.
    \end{split}
    \end{equation*}

Given the dissipativity assumption \ref{Dissipativity}, following  Lemma 3.1 in \cite{Maxim17} we have that
\begin{equation}\label{energy_lower_bound}
    U(\bx)\geq \frac{m_{\text{ds}}}{3}\|\bx\|_2^2-\frac{b_{\text{ds}}}{2}\log 3.
\end{equation}

Consider a large enough compact set $\mathbb{C}$ that contains $\bx_0$, we have that
\begin{equation*}
    \begin{split}
        &\ \ \ \bigg|\int_{\mathbb{R}^d} f(\bx) \widetilde \rho(\bx) \mathrm{d} \bx - \int_{\mathbb{R}^d} f(\bx)\rho(\bx) \mathrm{d} \bx\bigg|\\
        &\leq C_{\text{Norm}}\int_{\mathbb{R}^d} e^{- U(\bx)}|f(\bx)|\cdot\big |e^{\epsilon (\|\bx\|_2+\|\bx_0\|_2)} -1\big | \mathrm{d} \bx\\
        &=C_{\text{Norm}}\bigg(\int_{\mathbb{C}} e^{- U(\bx)}|f(\bx)|\cdot\big |e^{\epsilon (\|\bx\|_2+\|\bx_0\|_2)} -1\big | \mathrm{d} \bx+\int_{\mathbb{R}^d/\mathbb{C}} e^{- U(\bx)}|f(\bx)|\cdot\big |e^{\epsilon (\|\bx\|_2+\|\bx_0\|_2)} -1\big | \mathrm{d} \bx\bigg)\\
        &\leq C_{\text{Norm}}\bigg(\underbrace{\int_{\mathbb{C}} e^{- U(\bx)}|f(\bx)|\cdot\big |e^{\epsilon (\|\bx\|_2+\|\bx_0\|_2)} -1\big | \mathrm{d} \bx}_{\text{I}}+\underbrace{\int_{\mathbb{R}^d/\mathbb{C}} |f(\bx)|\cdot e^{- \frac{m_{\text{ds}}}{3} \|\bx\|_2^2 + 2\epsilon \|\bx\|_2+\frac{b_{\text{ds}}}{2}\log 3} \mathrm{d} \bx\bigg)}_{\text{II}},
    \end{split}
    \end{equation*}
where the last inequality follows by Eq.\eqref{energy_lower_bound}.

Recall that the quadratic growth of a Lipschitz continuous function $f$ (Assumption \ref{smooth_potential}) is much slower than the decay speed of an exponential function. As such, we can first upper bound $\text{II}$ by the tail of a Gaussian density:
\begin{equation*}
\begin{split}
\text{II}=3e^{\frac{b_{\text{bs}}}{2}}\int_{\mathbb{R}^d/\mathbb{C}} |f(\bx)|\cdot e^{- \frac{m_{\text{ds}}}{3} \|\bx\|_2^2 + 2\epsilon \|\bx\|_2} \mathrm{d} \bx\leq 3e^{\frac{3 d \epsilon^2}{m_{\text{ds}}} +\frac{b_{\text{bs}}}{2}}\int_{\mathbb{R}^d/\mathbb{C}}  |f(\bx)|e^{2\epsilon \|\bx\|_2-2\epsilon\bx}\cdot e^{- \frac{m_{\text{ds}}}{3} \big\|\bx-\frac{3 \epsilon}{m_{\text{ds}}} \bm{1}\big\|_2^2} \mathrm{d} \bx\leq O(\epsilon),
\end{split}
\end{equation*}
where the last inequality holds given a large enough compact set $\mathbb{C}$.

For the first term $\text{I}$ with small enough $\epsilon$ and a fixed $\mathbb{C}$, applying Taylor expansion completes the proof.
\begin{equation*}
\begin{split}
\text{I}=\int_{\mathbb{C}} e^{- U(\bx)}|f(\bx)| \cdot \big |e^{\epsilon (\|\bx\|_2+\|\bx_0\|_2)} -1\big | \mathrm{d} \bx\leq \sup_{\bx\in\mathbb{C}}\epsilon (\|\bx\|_2+\|\bx_0\|_2) \cdot |f(\bx)| \cdot \int_{\mathbb{C}} e^{- U(\bx)} \mathrm{d} \bx\leq O(\epsilon).
\end{split}
\end{equation*}
\qed    
\end{proof}

\begin{lemma}
\label{control_diff_measure_v2}
Suppose we have probability densities $\rho(\bx)$ and $\widetilde\rho(\bx)$ that satisfy $\rho(\bx)= e^{-U(\bx)} / C$ and $\widetilde \rho(\bx)= e^{-\widetilde U(\bx)}/\widetilde C $ with $C$ and $\widetilde C$ being the normalizing constants. Moreover, the energy functions $U$ and $\widetilde U$ follow dissipative Assumption \ref{Dissipativity} and the smoothness Assumption \ref{smooth_potential} and $\|\nabla \widetilde U - \nabla U\|_{\infty} \leq \epsilon$. Then we have that $$|\log\rho(\bx)-\log\widetilde \rho(\bx)| \leq O(\epsilon \|\bx\|_2+\epsilon).$$

Further, given a probability density $\widehat\rho(\bx)$ that satisfies the dissipative Assumption \ref{Dissipativity}, we have $$\int_{\dbR^d}|\log\rho(\bx)-\log\widetilde \rho(\bx)| \widehat \rho(\bx)\dd \bx\leq O(\epsilon).$$
\end{lemma}

\begin{proof} (i) Similar to Eq.\eqref{diff_U}, we have $|\widetilde U(\bx)- U(\bx)|\leq \epsilon (\|\bx\|_2 + \|\bx_0\|_2)$. Combining $e^a\leq 1+2a$ for $a\leq 1$, it follows
\begin{equation*}
\begin{split}
    |\widetilde C - C| &\leq \int_{\dbR^d} e^{-U(\bx)}\big| e^{-\widetilde U(\bx)+U(\bx)} - 1\big| \dd \bx \leq \epsilon \underbrace{\int_{\dbR^d} e^{-U(\bx)}(2\big\| \bx\|_2+2\|\bx_0 \big\|_2) \dd \bx}_{\text{integrable}},
\end{split}
\end{equation*}
where the last item is integrable due to the fast tail decay by Assumption \ref{Dissipativity}. We can easily show that $|\frac{\widetilde C}{C}-1|\leq O(\epsilon)$ for small enough $\epsilon$. It concludes that
\begin{equation*}
\begin{split}
    |\log\rho(\bx)-\log\widetilde \rho(\bx)| \leq |\widetilde U(\bx)- U(\bx)| + \big|\log C - \log \widetilde C\big|\leq O(\epsilon \|\bx\|_2)+O(\epsilon)=O(\epsilon\|\bx\|_2+\epsilon).
\end{split}
\end{equation*}

(ii) Similar to Lemma \ref{control_diff_measure}, the second result holds directly due to the fast tail decay induced by Assumption \ref{Dissipativity}.
\qed    
\end{proof}

The following lemma is a restatement of Lemma 1.6 in \cite{Nutz22_note}.
\begin{lemma}[Data processing inequality]\label{data_processing}
    Let $P, Q\in \mathcal{P}(\Omega)$ and $K: \Omega\rightarrow  \mathcal{P}(\Omega')$ a Markov kernel. Assume $P'\in \mathcal{P}(\Omega')$  and $Q'\in \mathcal{P}(\Omega')$ are the second marginals %\nicole{second marginals. (first marginal is P and Q, $\text{KL}(P|Q)= \text{KL}(P\otimes K|Q\otimes K)$.)} 
    of $P\otimes K \in \mathcal{P}(\Omega\otimes \Omega')$ and $Q\otimes K \in \mathcal{P}(\Omega\otimes \Omega')$, respectively. Then we have
    \begin{equation*}
        \text{KL}(P'|Q')\leq \text{KL}(P\otimes K|Q\otimes K).
    \end{equation*}
\end{lemma}

\begin{lemma}\label{finite}
Given a non-negative sequence $\{x_i\}_{0\leq i\leq N}$ such that $\sum_{i=0}^N x_i\leq C$ and $x_{i+1}\leq x_i+\epsilon$, we have 
\begin{equation*}
    x_i \leq \frac{C}{i+1}+\sqrt{2C\epsilon}+\epsilon.
\end{equation*}
\end{lemma}

\begin{proof}
Fix $0\leq i^* \leq N$, consider the optimization problem 
\begin{equation}\label{opt1}
\begin{aligned}
\max_{\boldsymbol x} \quad & x_{i^*}\\
\textrm{s.t.} \quad & \sum_{i=0}^N x_i\leq C \text{ and } x_{i}\geq 0 \text{ for $0\leq i\leq N$ }\\
  &x_{i+1}\leq x_i+ \epsilon \text{ for $0\leq i\leq N-1$ }.    \\
\end{aligned}
\end{equation}
The optimal solution exists as this is a linear programming with a bounded feasible region. Let  $x^*_{i^*}$ be the optimal value. Then, we must have 

\begin{enumerate}[label=\upshape(\Roman*)]
    \item\label{itfirst}$x^*_{i} = x^*_{i-1} +\epsilon$ for any $i\leq i^*$ where $x^*_{i-1} > 0$. 
    \item \label {itsecond}$x^*_{i}=0$ for $i>i^*$. 
\end{enumerate}

 To see \ref{itfirst}, suppose $x^*_{i-1}>0$ for some $i\leq i^*$ and $x^*_{i}< x^*_{i-1}+\epsilon$. Then we can decrease $x^*_{i-1}$ and increase each entry of $\{x^*_j\}_{i\leq j \leq i^*}$. Now the solution is still feasible but the objective value is increased, thus contracting the optimality of \eqref{opt1}.
 
 To see \ref{itsecond}, if $x^*_i=a>0$ for some $i>i^*$, we can set $x^*_i=0$ and increase each element of $\{x_j\}_{0\leq j \leq i^*}$ by $\frac{a}{i^*+1}$. Again, this would not violate any constraints.
 
Define $i^*_0\triangleq\min\{0\leq i\leq N: x^*_i>0\}$, the analysis can be broken down in two scenarios:

 \paragraph{Scenario 1} When $i^*_0=0$: by \ref{itfirst} and \ref{itsecond}, we must have  $x^*_0=c_0$ for some $c_0>0$, $x^*_i=c_0+\epsilon i$ for all $i\leq i^*$ and $x^*_i=0$ for all $i>i^*$. It follows that $\sum_{i=0}^N x^*_i=(i^*+1)c_0+\frac{i^*(1+i^*)\epsilon}{2}\leq C$,
 which implies $c_0\leq \frac{C}{i^*+1}-\frac{i^*\epsilon}{2}$ and $i^*\leq \sqrt{\frac{2C}{\epsilon}}.$
 As a result, the optimal value of $x^*_{i^*}$ satisfies
  \begin{equation*}
     x^*_{i^*}= (c_0+i^*\epsilon) \leq \frac{C}{i^*+1}+\frac{(i^*)\epsilon}{2} \leq \frac{C}{i^*+1}+\sqrt{C\epsilon}.
 \end{equation*}

\paragraph{Scenario 2} When $i^*_0>0$: by \ref{itfirst} and \ref{itsecond}, we have  $x^*_i=0$ for $i\notin [i_0^*, i^*]$ and $x^*_i=c_0+\epsilon (i-i^*_0)$ for all $i^*_0\leq i\leq i^*$ and some $0<c_0\leq \epsilon$. Define $I\triangleq i^*-i^*_0+1$, we have $\sum_{i=0}^N x^*_i=c_0 I+\frac{(I-1)I\epsilon}{2}\leq C$, which implies $(I-1)\leq \sqrt{\frac{2C}{\epsilon}}$. The optimal value of $x^*_{i^*}$  satisfies
\begin{equation*}
    x^*_{i^*}=c_0+(I-1)\epsilon \leq \epsilon+\sqrt{2C\epsilon}.
\end{equation*}

Combining the results of Scenario 1 and Scenario 2 completes the proof.\qed
\end{proof}

\begin{lemma}\label{unbounded_cost}
Given a non-negative sequence $\{x_i\}_{i\geq 0}$ such that $\sum_i x_i=C<\infty$ and $x_{i+1}\leq x_i+\epsilon$. We have 
\begin{equation*}
    x_i\leq \frac{C}{i+1}+\sqrt{2C\epsilon}+\epsilon.
\end{equation*}
\end{lemma}
\begin{proof}
1) Since $\sum_i x_i=C<\infty$, we have $N$ such that $\sum_{i> N} x_i<\epsilon$ which implies $x_i\leq \epsilon$ for $i> N$; 2) For $0\leq i \leq N$, we have $\sum_{i=1}^Nx_i\leq C$, applying Lemma \ref{finite} shows  
$x_i\leq\frac{C}{i+1}+\sqrt{2C\epsilon}+\epsilon $ for $0\leq i \leq N$.  \qed
\end{proof}

\section{Experimental Details} \label{appendix:experiment}

\subsection{Conditional-inference framework}
\label{subsec:imputation_formulation}

Consider an arbitrary window of a \emph{multivariate} time series of some  fixed length $L$ and $K$ features (variates): ${\bx}_{\mathrm{data}} \in\mathbb{R}^{K\times L}$ from the full training dataset.
The entries of this window are labeled by \textit{observation}, \textit{condition},  \textit{target}, and \textit{unknown} (one entry can have multiple labels). Observations represent all known values from the raw data;
in many cases, the raw data has missing values, so the complementary of \textit{observation} is \textit{unknown};
the condition entries are presented to the model as partial information of the window and are part of the observations.
% the goal of the imputation task is to infer the target missing entries preferably together with the uncertainty estimates.

To evaluate the model, the target entries are randomly selected from the observations, but these are hidden from the model as artificial ``missing" values. The performance metrics are calculated by comparing the imputed values and the known observations as ground truth. The locations of observation, condition, and target missing values in a time series window are indicated by binary masks
$\bM_{\mathrm{obs}}, \bM_{\mathrm{cond}}, \bM_{\mathrm{target}} \in\{0,1 \}^{K\times L}$.
Their values can thus be obtained through Hadamard product
$\bx_{\mathrm{obs}} ={\bx}_{\mathrm{data}} \circ \bM_{\mathrm{obs}}$, 
$\bx_{\mathrm{cond}} ={\bx}_{\mathrm{data}} \circ \bM_{\mathrm{cond}}$ and $\bx_{\mathrm{target}} ={\bx}_{\mathrm{data}} \circ \bM_{\mathrm{target}}$, respectively. The masks may change from window to window.
Note that $\bM_{\mathrm{cond}} \circ \bM_{\mathrm{target}} \equiv \textbf{0}$ and  $\bM_{\mathrm{cond}} \circ \bM_{\mathrm{obs}} \equiv \bM_{\mathrm{cond}}$.
$\bM_{\mathrm{cond}} + \bM_{\mathrm{target}}$ is not necessarily equal to $\bM_{\mathrm{obs}}$ or $\textbf{1}^{K\times L}$.
The unknown entries do not have ground true values, as shown in Figure. \ref{fig:mask}. Having formulated the general multivariate time series imputation task as a probabilistic model, the imputation task is treated as a conditional generative model and the goal is to sample according to $p(\bx_{\mathrm{target}} | \bx_{\mathrm{cond}}, \bM_{\mathrm{cond}}, \bM_{\mathrm{target}})$.
% \begin{equation*}
% p(\bx_{\mathrm{target}} | \bx_{\mathrm{cond}}, \bM_{\mathrm{cond}}, \bM_{\mathrm{target}}).
% \label{eq:imputation_goal}
% \end{equation*}

\subsection{Datasets} \label{sec:appendix_dataset}
% \Wei{if Type 1 data is only used for debugging, can we remove it?}
\paragraph{Synthetic dataset}
Each sample has $K=8$ features and $L=50$ time points. The signal has a simple temporal and feature structure.
% In \emph{type one}, sample $i$ follows that $\bx_{\mathrm{data},i}(k, t) = \mathrm{signal}_k(t) + \sigma_{\mathrm{noise}} \cdot \varepsilon_{i,k,t},
% \quad k=1,...,K,\quad t=1,...,L$
% where $\varepsilon_{i,k,t} \stackrel{i.i.d.}{\sim} N(0, 1)$, $\sigma_{\mathrm{noise}}$ is a constant.
% In type one, the signals are fixed for all samples, so the entries are independent from each other 
% $p(\bx_{\mathrm{target}} | \bx_{\mathrm{cond}}, \bM_{\mathrm{cond}}, \bM_{\mathrm{target}}) = p(\bx_{\mathrm{target}} | \bM_{\mathrm{target}})$. 
% Learning the joint distribution of the time series is equivalent to learning the marginal distribution of each entry.
% Performing the imputation does not require the knowledge of covariation between features. 
The signals are a mixture of sinusoidal curves.
\begin{align*}
\mathrm{signal}_1(t) =& \sin(2\pi t) \qquad\qquad\qquad\quad \mathrm{signal}_5(t) = \sin^2(2\pi t) \cos(2\pi t) + 0.3 t  \\
\mathrm{signal}_2(t) =& \cos(2\pi t) \qquad\qquad\qquad\quad \mathrm{signal}_6(t) = \sin^3(2\pi t) -0.3t\\
\mathrm{signal}_3(t) =& \sin^2(2\pi t) \qquad\qquad\qquad\ \ \  \mathrm{signal}_7(t) = \cos^2(2\pi t) e^{-0.1t} - 0.2 t\\
\mathrm{signal}_4(t) =& 2 \sin^2(2\pi t) \cos(2\pi t) \qquad\ \  \mathrm{signal}_8(t) = \cos^2(2\pi t) \sin(2\pi t) e^{0.4t} + 0.2t.
\end{align*}

The noisy data is created by randomly shifting the phases and adding Gaussian noise,
\begin{equation*}
\bx_{\mathrm{data},i}(k, t) = \mathrm{signal}_k(t + \omega_i) + \sigma_{\mathrm{noise}} \cdot \varepsilon_{i,k,t},
\quad k=1,...,K,\quad t=1,...,L
\label{eq:sinusoid_dynamic}
\end{equation*}
where $i$ is the sample index, $\varepsilon_{i,k,t} \stackrel{i.i.d.}{\sim} N(0, 1)$,
$\omega_i \stackrel{i.i.d.}{\sim} \mathrm{Uniform}(0, 1)$.
The phase of each sample is random $\omega_i \stackrel{i.i.d.}{\sim} \mathrm{Uniform}[0, 1]$, and all features in a sample share the same phase shift.
Imputing the missing values requires inferring the phase of the signal based on a partially observed noisy signal which imposes the learning of the dependency between conditions and targets to handle the imputation task.
Once raw data is created, some time points are randomly removed, mimicking the missed observed values (unknown entries). 
The observed entries are split into conditions and artificial missing values (targets). 
20 consecutive time points of each feature are selected as artificial missing values (targets).

\paragraph{Real datasets}
The model is applied to real datasets such as air quality PM2.5 \citep{U_Air} and PhysioNet \citep{phy_2012}.
The air quality data has $K=36$ features and $L=36$ time points. The raw data has 13\% missing values (the portion of the unknown entries).
The PhysioNet data has 4000 clinical time series with $K=35$ features and $L=48$.
The raw data is sparse with 80\% missing values. We further randomly select 10\% and 50\% out of observed values as the targets.
The preprocess and time window splitting follow the previous work \citep{CSDI}.
Both real datasets have large dimensions (in terms of $K\times L$) than the synthetic data.

\subsection{SDEs} \label{sec:appendix_sde}

\paragraph{VESDE}
The forward SDE is $\mathrm{d} \bx_t 
= \sqrt{\frac{\mathrm{d} \sigma^2(t) }{\mathrm{d} t} }
= g(t) \mathrm{d} \mathbf{w}_t
= \sigma_{\min} \left(\frac{ \sigma_{\max} }{\sigma_{\min}} \right)^t
\sqrt{2\log \frac{ \sigma_{\max} }{\sigma_{\min}} }
\mathrm{d} \mathbf{w}_t.$
The variance term is $\sigma\left(t \right)
= \sigma_{\min} \left(\frac{ \sigma_{\max} }{\sigma_{\min}} \right)^t,
\quad t \in (0, 1]$. $\sigma_{\min}$ is usually set as a very small value close to zero.
$\sigma_{\max}$ is set as a much larger value than the variance of the data so 
$p(\textbf{x}_t | \textbf{x}_0)$ is closer to normal distribution as $t$ approximates $T=1$. %VESDE does not suffer much from discretization errors.  

\paragraph{VPSDE}
The forward SDE is $\mathrm{d} \bx_t = -\frac{1}{2}\beta(t) \bx_t \mathrm{d}t
+ \sqrt{\beta(t)} \mathrm{d} \mathbf{w}_t$, where $\beta(t) = \beta_{\min} + t (\beta_{\max} - \beta_{\min}),
\quad t \in (0, 1].$
A straightforward numerical scheme follows that
\begin{align*}
\bx_{i+1}
\approx& \left(1 - \frac{1}{2}\beta^{\text{VPSDE}}(i/N)\Delta \right) \bx_i
    + \sqrt{\beta^{\text{VPSDE}}(i/N)\Delta } \cdot \boldsymbol{\varepsilon}.
\end{align*}
which is adopted in \citet{score_sde}.

\subsection{Model details} \label{sec:appendix_models}
% The implementation of model structures can be found in our public code repository. 
The key structure is the backward policy for generating imputed values; The forward policy aims to reduce the transport cost.

% \Wei{To Yu: can we simplify the Transformer subsection by around 30-40\%? .e.g. cite CSDI's work and only mention the most important part of the structures.}

\paragraph{Transformer for the backward policy} The diagram below shows the major transformations of the neural network.
The tuple $(B,C,K,L)$ represents the shape of a tensor, where $B$ is the batch size, $C$ is the number of channels, $K$ is the number of features, $L$ is the number of time points. The backward policy takes $\bx_{\mathrm{cond}}, \bM_{\mathrm{cond}}$ as the input.
\begin{equation*}
\begin{aligned}
& \bx_{\mathrm{cond}}  \xrightarrow{\mathrm{input}}
(B,K,L) \xrightarrow[]{\mathrm{unsqueeze}}
(B,1,K,L) \xrightarrow[]{\mathrm{stem}} 
(B,C-1,K,L) \xrightarrow[]{\mathrm{concatenate}\; \bM_{\mathrm{cond}} } \\
& (B,C,K,L) \xrightarrow[\text{step 1}]{\mathrm{add\; embedding} }
(B,C,K,L) \xrightarrow[\times N_{\mathrm{layer}}]{\mathrm{transformer\; blocks} }
(B,C,K,L) \xrightarrow[]{\mathrm{output \; projection}} 
(B,K,L).
\end{aligned}
\end{equation*}
In step 1, the embedding is a concatenation of feature index, time index, and the condition mask with shape $(B, C_{\mathrm{feature}} + C_{\mathrm{time}} + 1, K, L)$.
The time index is for the time series not for the SDE diffusion time.
The feature embedding is the same for all batches and all time point, and the time embedding is the same for all batches and all features.
Then the embedding is projected to $C$ channels.
The diffusion time embedding is added in the transformer blocks.
The model stacks $N_{\mathrm{layer}}$ transformer blocks with residual connections.
The diagram of the main component is the following,
\begin{equation*}
\begin{aligned}
\xrightarrow{\mathrm{input}}
& (B,C,K,L) \xrightarrow[\text{step 1}]{\mathrm{reshape}}
(BK, L,C) \xrightarrow[\text{step 2}]{\mathrm{time\; transformer}}
(BK, L,C) \xrightarrow[\text{step 3}]{\mathrm{reshape}} \\
&(B,C,K,L) \xrightarrow[\text{step 4}]{\mathrm{reshape}}
(BL, K,C) \xrightarrow[\text{step 5}]{\mathrm{feature\; transformer}}
(BL, K,C)  \xrightarrow[]{\mathrm{reshape}}
(B,C,K,L)
\end{aligned}
\end{equation*}
Each block has two transformers, one is for temporal information and the other is for feature information.
In step 2, the time transformer encoder performs along the $L$ dimension as the sequence and treats $C$ as the embedding.
The size of the attention matrix is $L\times L$.
In step 1, the feature dimension is reshaped into batch dimension meaning all features share the same transformer function. 
Similarly, in step 5, the feature transformer performs along features $K$ as the sequence and uses $C$ as the embedding. This is the reason why step 4 reshapes time dimension $L$ into batch dimension.
The size of the attention matrix is $K\times K$.
Our model has $N_{\mathrm{layer}}=4$ transformers blocks. Each transformer block has 64 channels, 8 attention heads.
Totally it has 414 thousand parameters.

% Such temporal-feature transformer carefully balances the trade-off between memory usage and flexibility. One alternative design is the following,
% \begin{equation*}
% \xrightarrow{\mathrm{input}}
% (B,C,K,L) \xrightarrow[]{\mathrm{reshape}}
% (B, KL, C) \xrightarrow[\times N_{\mathrm{layer}}]{\mathrm{Transformer}}
% (B, KL, C) \xrightarrow[]{\mathrm{reshape}}
% (B,C,K,L)
% \end{equation*}
% The whole window is reshaped into a sequence with length $KL$.
% The order of the sequence does not matter for transformers.
% The transformer performs along $KL$ dimension as the sequence with $C$ as the embedding.
% Such a design is more flexible than the previous one.
% In each layer, the attention is applied to the whole window.
% The size of the attention matrix is $KL\times KL$.
% As a comparison, in the previous design, each layer applies attention to either the time dimension or feature dimension, which might be short-sighted, but significantly reduces the memory space from $KL\times KL$ to $L\times L$ or $K\times K$.

\paragraph{U-Net for the forward policy} The forward policy does not hand missing values.
We use the U-Net for the forward policy \citep{ronneberger2015u, song2019generative}. It has skip connections from the down-scaling branch to the up-scaling branch on each scale. Our model has 3 down-scaling layers and 3 up-scaling layers with 32 channels and 664K parameters.

\paragraph{Activation functions} It is important to note that the loss function of the model involves the calculation of the divergence with respect to the data. We use \emph{SiLU} instead of \emph{ReLU} to avoid vanishing gradients.

\subsection{Training}
Compared to the denoising score matching method, the Schr\"odinger bridge method involves optimizing both forward and backward SDEs and sampling non-linear forward SDE,
which makes it harder to train, perform inference, debug, and tune the model.
In both approaches, the inference only requires the backward SDE and the procedure is similar.

\paragraph{Hyperparameters}
The model is warmed up using SGM for about 6000 iterations with batch size 64. We use AdamW as the optimizer. The alternative training has 40 stages with each stage running 480 iterations. The trajectories are sampled every 80 iterations. The learning rates for the forward and backward steps are $2\times10^{-6}$ and $2\times10^{-5}$, respectively. The exponential decay scheduler is adopted to improve stability.  We use VESDE as the base SDE as introduced in section \ref{sec:appendix_sde}.
$\sigma_{\max} = 20, \sigma_{\min}=0.001$.
We use 100 discretization steps. 
The prior distribution for VESDE is $N(0, \sigma_{\max}^2 \textbf{I})$.

\subsection{Inference }\label{inference_supp}
The imputation task requires conditional sampling $p(\bx_{\mathrm{target}} | \bx_{\mathrm{cond}}, \bM_{\mathrm{cond}}, \bM_{\mathrm{target}})$.
% \begin{equation*}
% p(\bx_{\mathrm{target}} | \bx_{\mathrm{cond}}, \bM_{\mathrm{cond}}, \bM_{\mathrm{target}})
% \end{equation*}
The conditional inference model needs to process partially observed information $\bx_{\mathrm{obs}}$ and the condition mask $\bM_{\mathrm{obs}}$. 

% These all add complexity of the pipeline.
% So we test the model using unconditional sampling first.
% This is the reason we design the synthetic dataset \emph{type one}, where 
% $p(\bx_{\mathrm{target}} | \bx_{\mathrm{cond}}, \bM_{\mathrm{cond}}, \bM_{\mathrm{target}})
% =p(\bx_{\mathrm{target}} | \bM_{\mathrm{target}})$ in section \ref{sec:appendix_dataset}, so performing the unconditional inference is equivalent to the conditional inference in this special case.

% \Wei{how about we remove unconditional inference and remove type 1 in synthetic data? it seems like it is only used for debugging and is not mentioned in the main body.}

% \paragraph{Unconditional inference}
% The algorithm for unconditional is 
% \begin{algorithm}[H]
% \caption{Unconditional inference based on VESDE }

% \begin{algorithmic}[1]
% \STATE \textbf{Input:} trained backward policy $\overleftarrow\bz$,
% hyperparameters $\sigma_{\min}$, $\sigma_{\max}$.
% \STATE Draw sample $\bx_T \sim N(\textbf{0}, \sigma_{\max}^2 \textbf{I})$
% \FOR{$i \gets T$ to $1$}
%     \STATE $t = i/T$, 
%         $\varepsilon \sim N(0, \textbf{I})$,
%         $\Delta=1/T$,
%         $g = \sigma_{\min} \left(\frac{ \sigma_{\max} }{\sigma_{\min}} \right)^t
%         \sqrt{2\log \frac{ \sigma_{\max} }{\sigma_{\min}}}$
%     \STATE $\bx_{i-1} = \bx_{i}
%         + [- \bbf(t)
%         + g \overleftarrow\bz(t, \bx_{i})]\Delta
%         + g \sqrt{\Delta} \varepsilon $
% \ENDFOR
% \end{algorithmic}
% \label{alg-unconditiona_sampling}
% \end{algorithm}

\paragraph{Conditional inference}
The conditional inference, more specifically the imputation or inpainting in our case \citep{score_sde, CSDI}, is the following,

\begin{algorithm}[H]
\caption{Conditional inference based on VESDE}

\begin{algorithmic}[1]
\STATE \textbf{Input:} trained backward policy $\overleftarrow\bz$, $\bx_{\mathrm{cond}}$, $\bM_{\mathrm{cond}}$,
hyperparameters $\sigma_{\min}$, $\sigma_{\max}$.
\STATE Draw sample $\bx_T \sim N(\textbf{0}, \sigma_{\max}^2 \textbf{I})$
\FOR{$i \gets T$ to $1$}
    \STATE $t = i/T$, 
        $\varepsilon \sim N(0, \textbf{I})$,
        $\Delta=1/T$,
        $g = \sigma_{\min} \left(\frac{ \sigma_{\max} }{\sigma_{\min}} \right)^t
        \sqrt{2\log \frac{ \sigma_{\max} }{\sigma_{\min}}}$
    \STATE $\bx_i = \bx_{\mathrm{cond}} \circ \bM_{\mathrm{cond}} 
        + \bx_i \circ (\textbf{1} - \bM_{\mathrm{cond}})$
    \STATE $\bx_{i-1} = \bx_{i}
        + [- \bbf(t)
        + g \overleftarrow\bz(t, \bx_{i}, \bM_{\mathrm{cond}})]\Delta
        + g \sqrt{\Delta} \varepsilon $
\ENDFOR
\end{algorithmic}
\label{alg:predictor_corrector}
\end{algorithm}

% \paragraph{Predictor-corrector }
% Similar to SGM and \citet{forward_backward_SDE}, SBP can also use predictor-corrector during the sampling. In practice, we choose $r_{\mathrm{snr}}=0.08$. 
% Predictor-corrector can be helpful if the missing ratio is small.
% It does not improve the sample quality significantly if the missing ratio is large.

\subsection{Limitations }\label{limitations}
The marginal improvement doesn't mean minimizing the transport cost via the control variable $\bu$ is not promising; by contrast, our model is limited by other complications such as the divergence approximations. How to reduce the variance and computation workload of the Hutchinson estimators \citep{Hutchinson89, FFJORD} will be essential to improve the performance; other interesting updates can be seen in \citet{Chen_divergence_free}.

\subsection{Empirical verification} \label{sec:dsm_vs_sb_empirical}
In this section, we empirically compare the convergence of CSBI and CSBI$_0$ using the synthetic data as described in Appendix \ref{sec:appendix_dataset}. To make a fair comparison, CSBI$_0$ is trained following Eq.\eqref{eq:llk-ipf} by \emph{forcing $\overrightarrow\bz_{t} \equiv 0$}, which is equivalent in theory to score matching loss \citep{forward_backward_SDE}.
Two models share the same settings with a constant learning rate except that our method trains the forward policy in each iteration.
Since CSBI$_0$ needs more forward diffusion steps to converge to the ideal prior distribution, it may have poor performance when the number of diffusion steps is insufficient or the variance of the diffusion is small.
As a comparison, SBP can overcome such an issue by minimizing the transport cost through the forward SDE \citep{DSB}.
% \textbf{(wei, where do you say SGM needs long run to converge to proper prior?)}. 
In this experiment, the number of diffusion steps is 20, and the variance of the forward diffusion is small $\sigma_{\mathrm{max}}=0.3$ using VESDE as described in Appendix \ref{sec:appendix_sde}.

\subsection{Time series prediction } \label{sec:appendix_prediction}
Our model can be easily adopted for time series prediction task by simply manipulating the masks. The condition mask in our model corresponds to the context window in the prediction task, and the target mask is equivalent to the future window.
Our method allows missing values in the context window during both training and inference. 
The training and inference procedures remain the same as the imputation task.

We evaluate the performance of the model using two public datasets: Solar and Exchange \citep{lai2018modeling}. Details of the datasets are shown in Table \ref{tab:prediction_dataset}.
Baseline models include GP-copula \citep{salinas2019high}, Vec-LSTM-low-rank-Copula (Vec) \citep{salinas2019high}, TransMAF \citep{rasul2021autoregressive}. The performance of the baseline models is from the reference therein.

\begin{table}[H]
\caption{Properties of the datasets.}
\centering
\begin{tabular}{lccccc}
\toprule
Datasets & Dimension & Frequency & Total time points & Context length & Prediction length \\
\midrule
Exchange &  8 & Daily & 6,071 &  48  & 30 \\
Solar    &  137 & Hourly & 7,009 &  80 & 24 \\
\bottomrule
\end{tabular}
\label{tab:prediction_dataset}
\end{table}

\subsection{Imputation examples} \label{sec:appendix_imputation_eg} 
In this section, we demonstrate an example of the diffusion process from the prior distribution to the final data distribution using the PM2.5 dataset as shown in Figure \ref{fig:pm25_diff_demo}.
Figure \ref{fig:pm25_eg} and \ref{fig:physio_eg_1} present more examples of the imputed data distributions.
Figure \ref{fig:pm25_eg}, \ref{fig:pm25_eg1}, and \ref{fig:pm25_eg_2} illustrate the irregularity of the time series imputation, where the missing values can location anywhere in the window.
In Figure \ref{fig:pm25_eg}, the top left feature has much more missing values than the feature in row 3 column 1.
As a comparison, Figure \ref{fig:pm25_eg1} provides a different layout of missing values and conditions.
All these imputations are handled by one model not by models trained separately with different masks.

\begin{figure}[H]
\centering
\includegraphics[width=\textwidth]{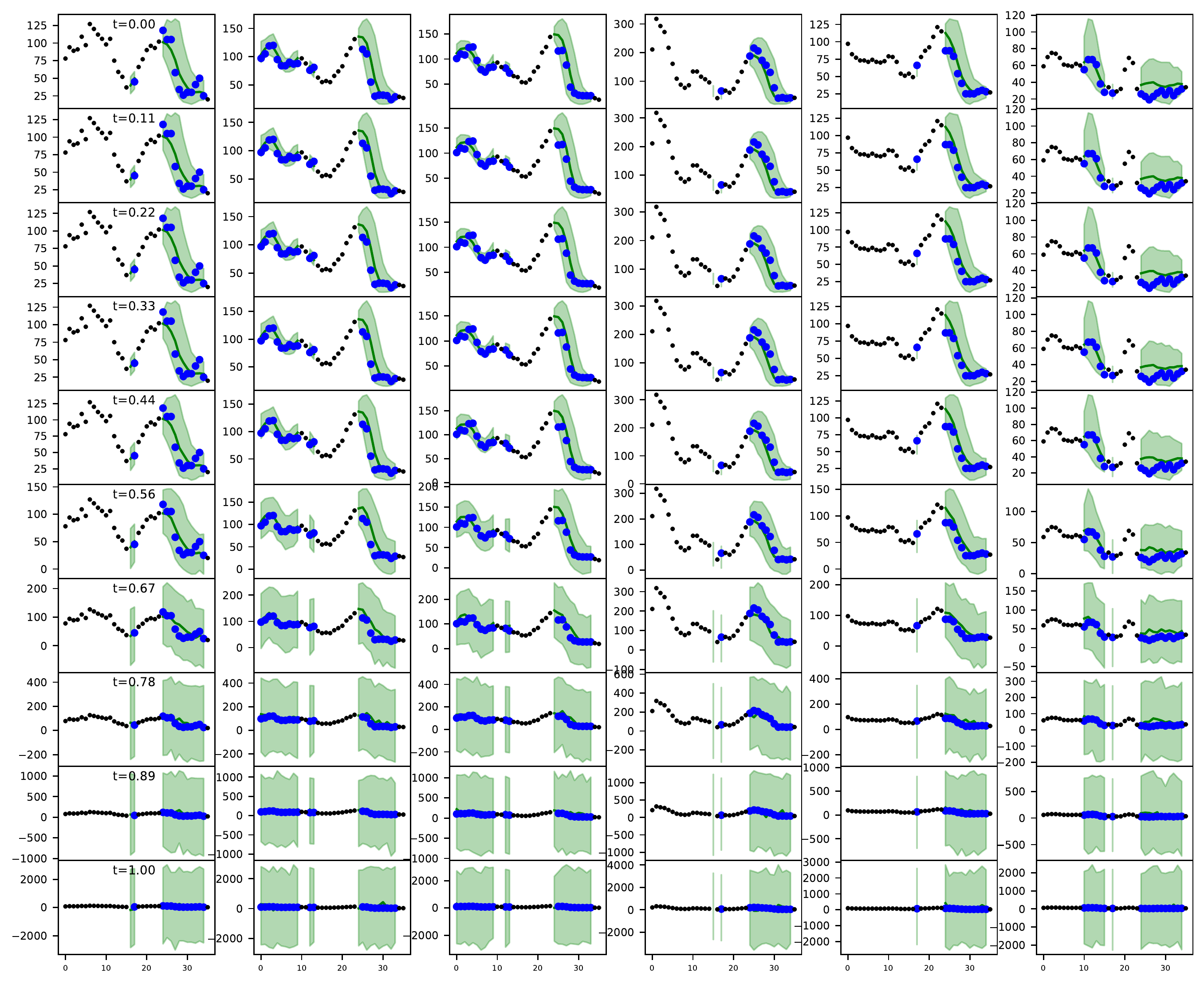}
\vskip -0.15in
\caption{Demonstration of the backward diffusion process for conditional inference using the air quality PM2.5 dataset.
The diffusion process starts from the prior distribution at the bottom ($t=1$), and the backward diffusion will converge to the data distribution at the top ($t=0$).
Each column is one feature, each row is one diffusion time.
The dark dots are $\bx_{\mathrm{obs}}$,
the blue dots are true values of $\bx_{\mathrm{target}}$,
the green band is 80\% confidence interval of imputation.
The imputation is performed using the normalized data; the figure shows the time series in the original scale.
}
\label{fig:pm25_diff_demo}
\end{figure}

\begin{figure}[H]
\centering
\includegraphics[width=\textwidth]{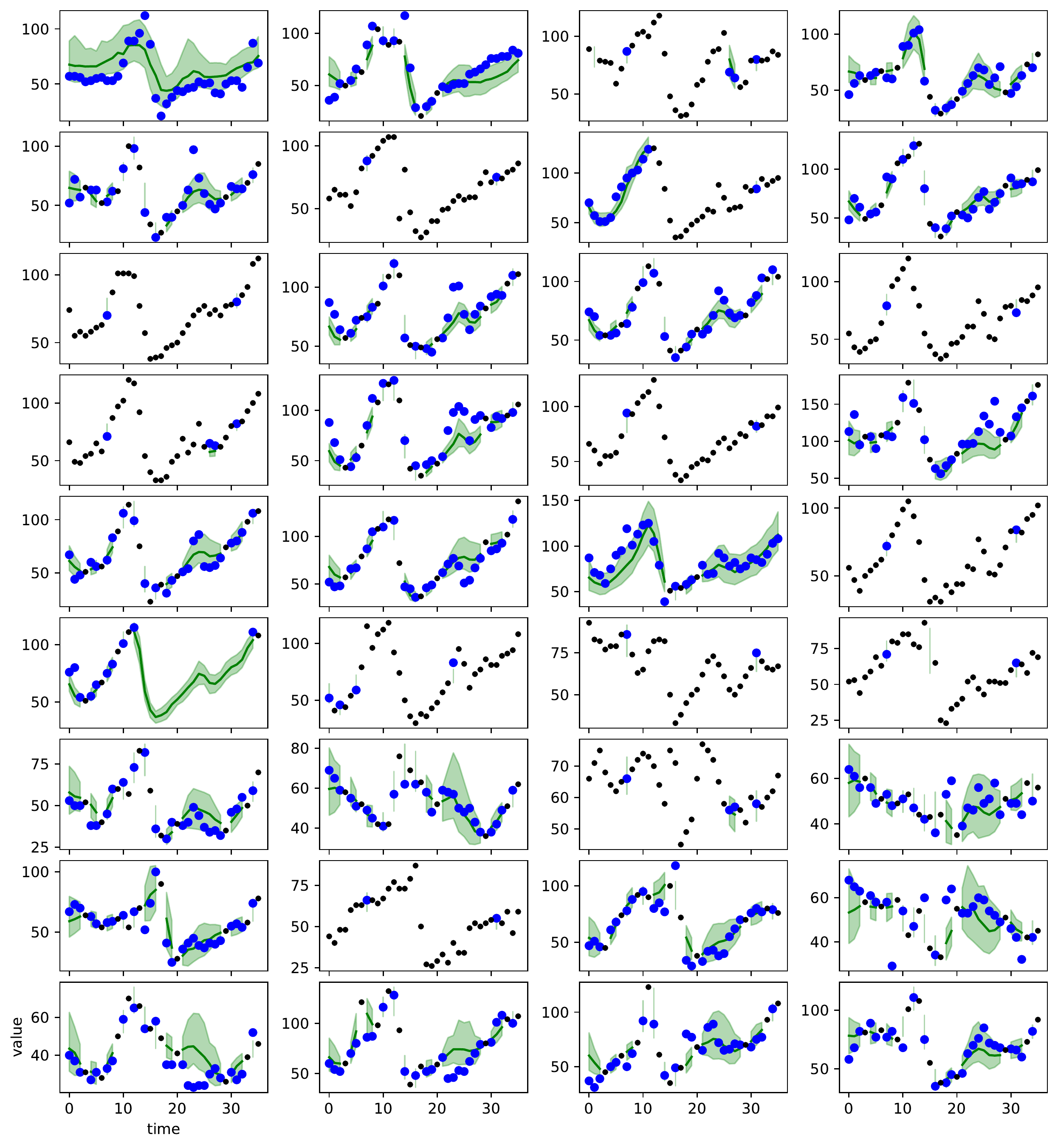}
\vskip -0.15in
\caption{Example of imputation for PM2.5 dataset.
The figure shows one sample with 36 features in each subplot and 36 time points.
Smaller dark dots are conditions, larger blue dots are ground true values of the artificial missing values, the green belt shows the 80\% confidence interval and the median curve of the imputed values. The time series data do not have observations at every time point, the missing ones are the unknowns.
}
\label{fig:pm25_eg}
\end{figure}

\begin{figure}[H]
\centering
\includegraphics[width=\textwidth]{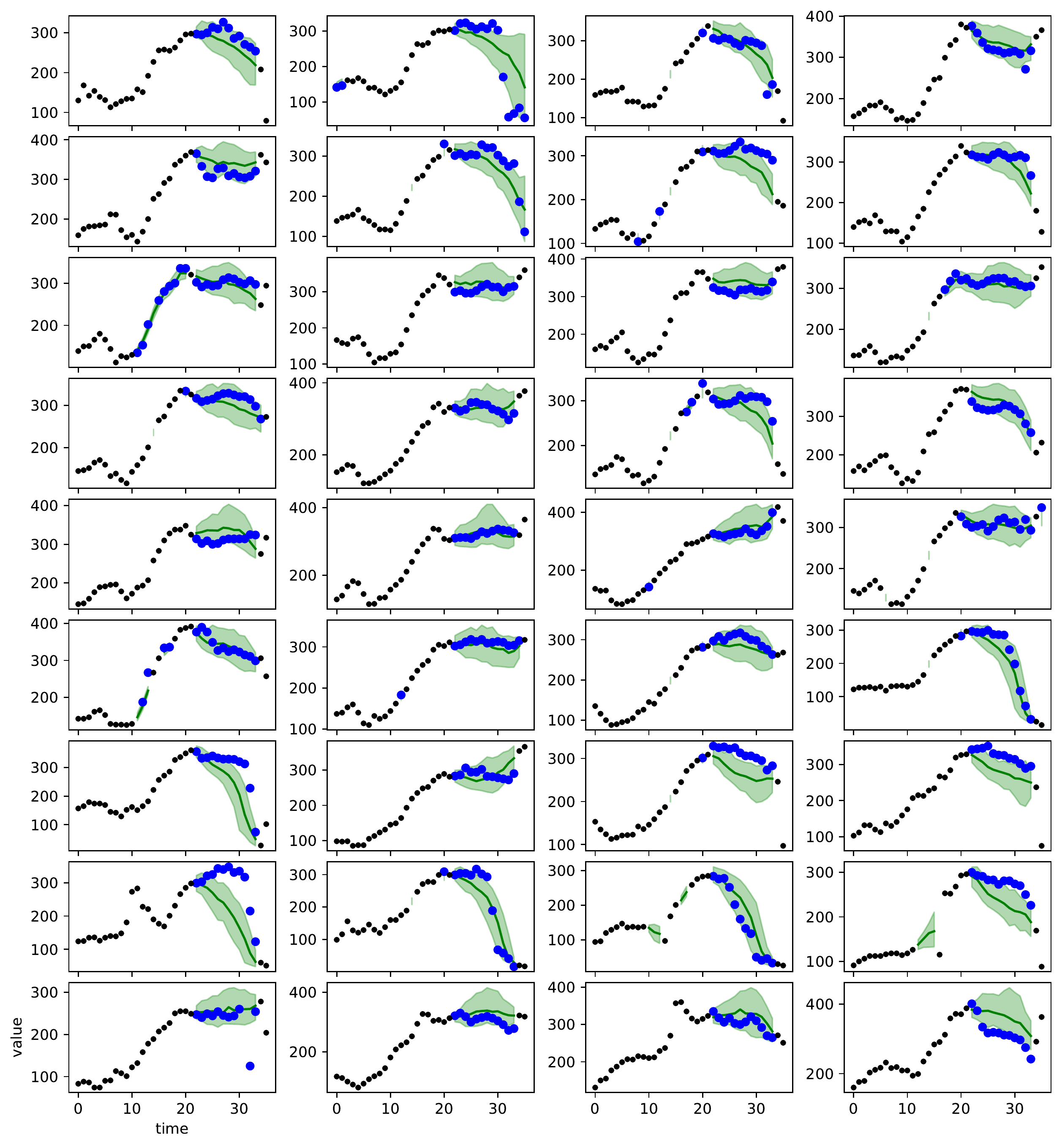}
\vskip -0.15in
\caption{Example of imputation for PM2.5 dataset. The figure shows one sample with 36 features in each subplot and 36 time points.
Smaller dark dots are conditions, larger blue dots are ground true values of the artificial missing values, the green belt shows the 80\% confidence interval and the median curve of the imputed values. The time series data do not have observations at every time point, the missing ones are the unknowns.
}
\label{fig:pm25_eg1}
\end{figure}

\begin{figure}[H]
\centering
\includegraphics[width=\textwidth]{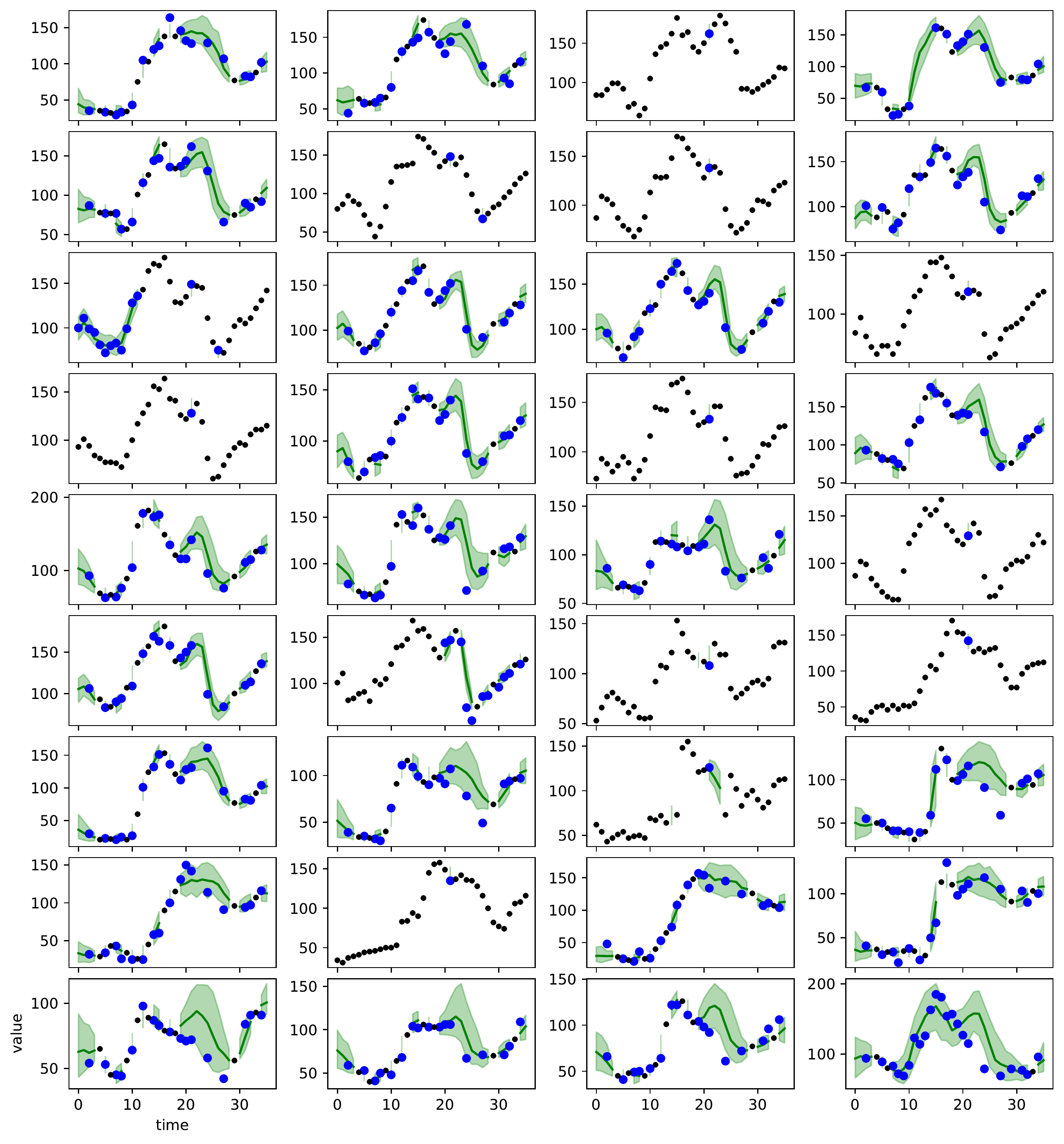}
\vskip -0.15in
\caption{Example of imputation for PM2.5 dataset. The figure shows one sample with 36 features in each subplot and 36 time points.
Smaller dark dots are conditions, larger blue dots are ground true values of the artificial missing values, the green belt shows the 80\% confidence interval and the median curve of the imputed values. The time series data do not have observations at every time point, the missing ones are the unknowns.
}
\label{fig:pm25_eg_2}
\end{figure}

%%%%%%%%%%%%%%%%%%%%%%%%%%%%%%%
\begin{figure}[H]
\centering
\vskip -0.15in
\includegraphics[scale=0.47]{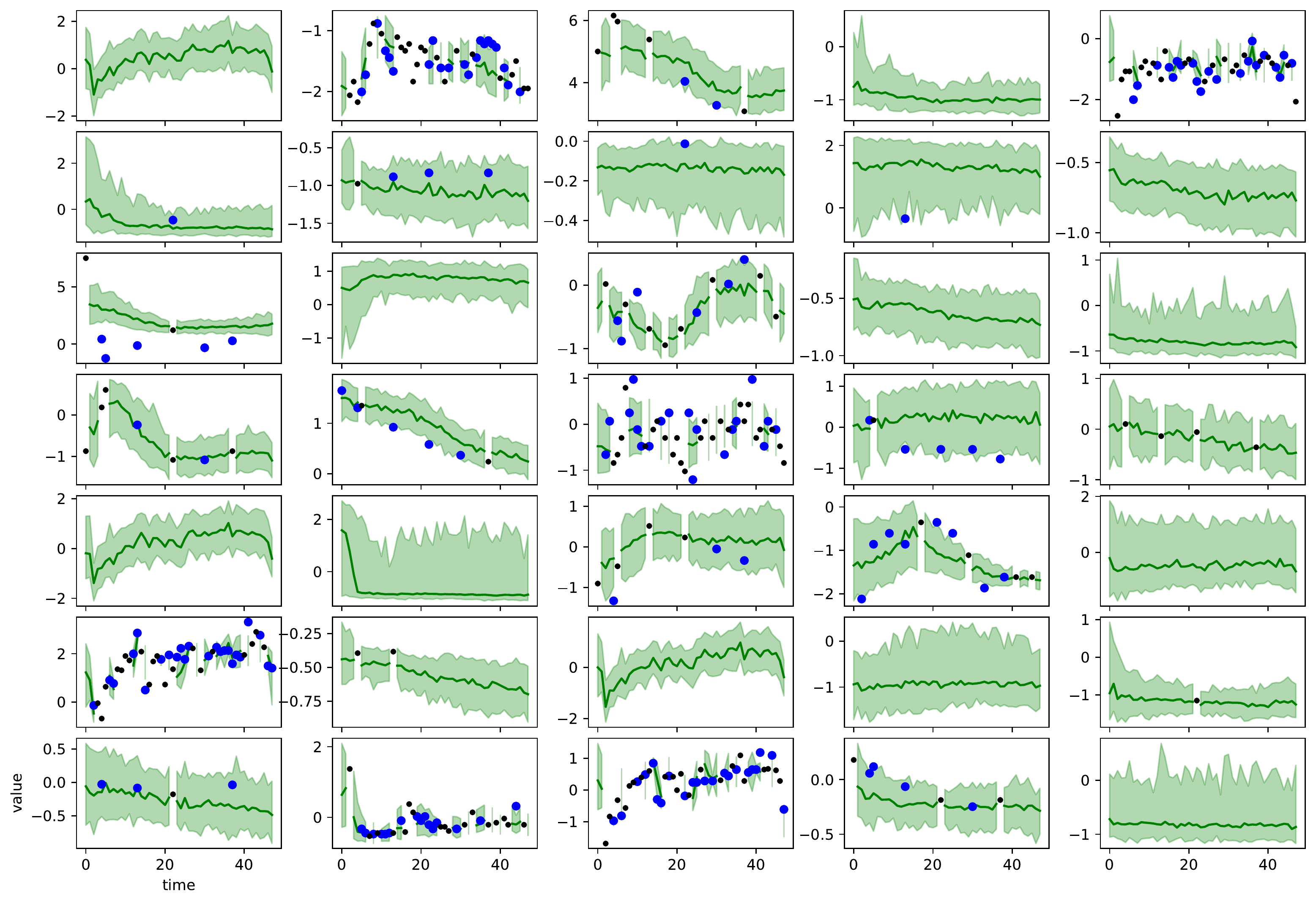}
\vskip -0.15in
\caption{Example of imputation for PhysioNet dataset with artificial missing ratio 0.1.
The missing ratio is the portion of selected artificial missing values among observations.
The figure shows one sample with 35 features in each subplot and 48 time points.
Smaller dark dots are conditions,
larger blue dots are ground true values of the artificial missing values,
the green belt shows the 80\% confidence interval and the median curve of the imputed values.
The time series data do not have observations at every time point, the missing ones are the unknowns.
As the data is sparse with 80\% of unknowns, the imputation has a very wide confidence band.
}
\label{fig:physio_eg}
\end{figure}

\begin{figure}[H]
\centering
\vskip -0.15in
\includegraphics[scale=0.47]{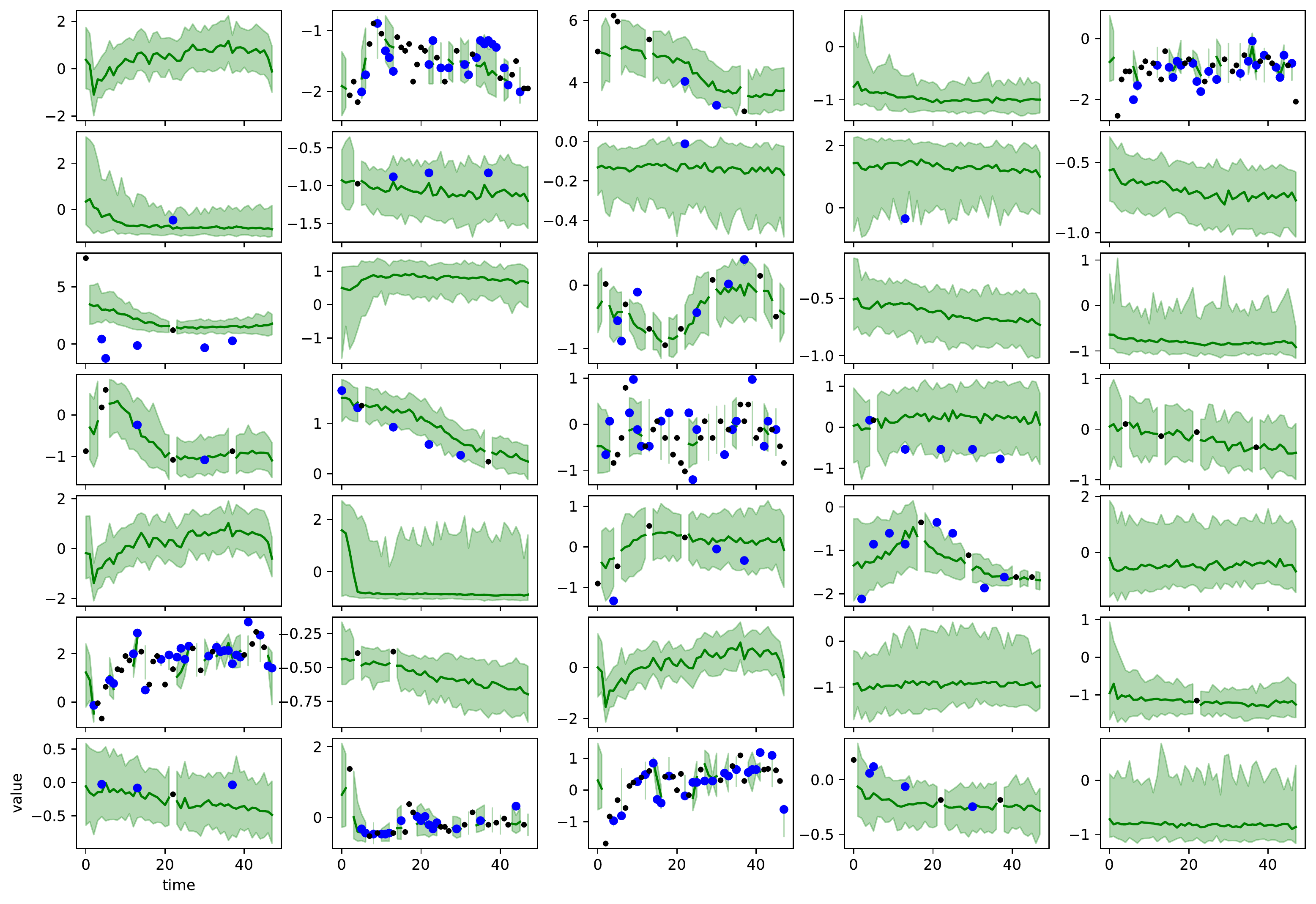}
\vskip -0.2in
\caption{Example of imputation for PhysioNet dataset with artificial missing ratio 0.5. 
}
\label{fig:physio_eg_1}
\end{figure}

% \begin{figure}[H]
% \centering
% % \includegraphics[width=\textwidth]{figures/physio0.9_batch5.pdf}
% \includegraphics[scale=0.48]{figures/physio0.9_batch5.pdf}
% \vskip -0.2in
% \caption{Example of imputation for PhysioNet dataset with artificial missing ratio 0.9. 
% }
% \label{fig:physio_eg_2}
% \end{figure}

% \begin{figure}[H]
% \centering
% \includegraphics[width=\textwidth]{figures/diff_demo.pdf}
% \caption{Demonstration of the forward and backward diffusion processes.
% Each row corresponds to a different diffusion time. Each column is one feature.
% The colored band is 80\% confidence interval.
% The data is synthetic \emph{type one}. The backward SDE samples come from unconditional sampling.
% The forward process and the backward process have the same marginal distribution at a certain diffusion time, indicating the convergence of the model training.
% }
% \end{figure}

% \subsection{Prediction task}

% \begin{table*}[ht]
% \caption{Model evaluation: 
% The metric is continuous ranked probability score (CRPS). A smaller metric is better.
% }
% \label{tab:prediction_results}
% \vspace{-0.01in}
% % \footnotesize
% \centering
% % \resizebox{\textwidth}{!}{%
% \begin{tabular}{l||c|c|c } 
% \hline
%  & Exchange & Solar  & Electricity \\ \hline
% Method 1 &  1 & 2 &  3  \\ \hline
% Method 2 &  1 & 2 &  3  \\ \hline
% %
% \textbf{Ours} &  1 & 2 &  3  \\ \hline
% \end{tabular}%
% \vspace{-0.2in}
% \end{table*}

\end{document}

%% file: package.tex
\usepackage{comment}
\usepackage{bm}
\usepackage{url}            % simple URL typesetting
\usepackage{booktabs}       % professional-quality tables
\usepackage{amsfonts}       % blackboard math symbols
\usepackage{amsmath,amssymb}
\usepackage{empheq}
\usepackage{amsthm}
\usepackage{nicefrac}       % compact symbols for 1/2, etc.
\usepackage{microtype}      % microtypography
\usepackage{enumerate,natbib,mathrsfs}
\usepackage{tablefootnote}
\usepackage{wrapfig}
\usepackage{graphicx}
\usepackage{subfigure}
\usepackage{enumitem}
\usepackage{multirow}
\usepackage{extarrows}
\usepackage[OT1]{fontenc}
\usepackage[bookmarks=false]{hyperref}
\usepackage[dvipsnames]{xcolor}

\definecolor{darkblue}{rgb}{0.0, 0.0, 0.55}
\usepackage{hyperref}
\hypersetup{
  pdffitwindow=true,
  pdfstartview={FitH},
  pdfnewwindow=true,
  colorlinks,
  linktocpage=true,
  linkcolor=red,
  urlcolor=red,
  citecolor=darkblue
}

\usepackage{hyperref}
\hypersetup{colorlinks=true,citecolor=darkblue,linkcolor=red}

\DeclareMathOperator*{\argmin}{arg\,min}

\newcommand{\bu}{{\boldsymbol u}}
\newcommand{\bx}{{\boldsymbol x}}
\newcommand{\by}{{\boldsymbol y}}

\newcommand{\bX}{{\boldsymbol X}}

\newcommand{\bbf}{{\boldsymbol f}}

\newcommand{\dd}{\mathcal{\dagger}}

\newcommand{\ea}{\end{array}}
\newcommand{\ee}{\end{equation}}
\newcommand{\bea}{\begin{eqnarray}}
\newcommand{\eea}{\end{eqnarray}}
\newcommand{\beaa}{\begin{eqnarray*}}
\newcommand{\eeaa}{\end{eqnarray*}}

\def\dbR{\mathbb{R}}

\def\E{\mathbb{E}}

\def\bx{{\bf x}}
\def\by{{\bf y}}
\def\bz{{\bf z}}

\def\qed{ \hfill\qedsymbol}

\newcommand{\basa}{\begin{assumption}}
\newcommand{\easa}{\end{assumption}}

\newcommand{\bas}{\begin{assum}}
\newcommand{\eas}{\end{assum}}

\def\dd{\mathrm{d}}

\def\limP2{\,\mathop{\buildrel \Pi_2\over\longrightarrow\,}}

\def\1{{\bf 1}}
\def\by{{\bf y}}

\def\:{\!:\!}

\newtheorem{assump}{Assumption}

\usepackage{algorithm}
\usepackage{algorithmic}

\newtheorem{theorem}{Theorem}

\newtheorem{lemma}{Lemma}
\newtheorem{proposition}{Proposition}
\newtheorem{assumption}{Assumption}

%% file: ICML23_new_score_potential.bbl
\begin{thebibliography}{61}
\providecommand{\natexlab}[1]{#1}
\providecommand{\url}[1]{\texttt{#1}}
\expandafter\ifx\csname urlstyle\endcsname\relax
  \providecommand{\doi}[1]{doi: #1}\else
  \providecommand{\doi}{doi: \begingroup \urlstyle{rm}\Url}\fi

\bibitem[Anderson(1982)]{Anderson82}
Anderson, B.~D.
\newblock Reverse-time diffusion equation models.
\newblock \emph{Stochastic Processes and Their Applications}, 12\penalty0
  (3):\penalty0 313--326, 1982.

\bibitem[Bunne et~al.(2023)Bunne, Hsieh, Cuturi, and Krause]{Gaussian_SB}
Bunne, C., Hsieh, Y.-P., Cuturi, M., and Krause, A.
\newblock {The Schr\"{o}dinger Bridge between Gaussian Measures has a Closed
  Form}.
\newblock In \emph{Proc. of the International Conference on Artificial
  Intelligence and Statistics (AISTATS)}, 2023.

\bibitem[Caluya \& Halder(2022)Caluya and Halder]{Caluya21}
Caluya, K. and Halder, A.
\newblock {Wasserstein Proximal Algorithms for the Schr\"{o}dinger Bridge
  Problem: Density Control with Nonlinear Drift}.
\newblock \emph{IEEE Transactions on Automatic Control}, 67\penalty0
  (3):\penalty0 1163--1178, 2022.

\bibitem[Cao et~al.(2018)Cao, Wang, Li, Zhou, Li, and Li]{BRITS}
Cao, W., Wang, D., Li, J., Zhou, H., Li, L., and Li, Y.
\newblock {BRITS: Bidirectional Recurrent Imputation for Time Series}.
\newblock In \emph{Advances in Neural Information Processing Systems
  (NeurIPS)}, 2018.

\bibitem[Che et~al.(2018)Che, Purushotham, Cho, Sontag, and
  Liu]{che2018recurrent}
Che, Z., Purushotham, S., Cho, K., Sontag, D., and Liu, Y.
\newblock {Recurrent Neural Networks for Multivariate Time Series with Missing
  Values}.
\newblock \emph{Scientific reports}, 8\penalty0 (1):\penalty0 1--12, 2018.

\bibitem[Chen et~al.(2023)Chen, Chewi, Li, Li, Salim, and
  Zhang]{Sitan_22_sampling_is_easy}
Chen, S., Chewi, S., Li, J., Li, Y., Salim, A., and Zhang, A.~R.
\newblock {Sampling is as Easy as Learning the Score: Theory for Diffusion
  Models with Minimal Data Assumptions}.
\newblock In \emph{Proc. of the International Conference on Learning
  Representation (ICLR)}, 2023.

\bibitem[Chen et~al.(2022)Chen, Liu, and Theodorou]{forward_backward_SDE}
Chen, T., Liu, G.-H., and Theodorou, E.~A.
\newblock {Likelihood Training of Schr\"{o}dinger Bridge using Forward-Backward
  SDEs Theory}.
\newblock In \emph{Proc. of the International Conference on Learning
  Representation (ICLR)}, 2022.

\bibitem[Chen \& Georgiou(2016)Chen and Georgiou]{Chen16}
Chen, Y. and Georgiou, T.
\newblock {Stochastic Bridges of Linear Systems}.
\newblock \emph{IEEE Transactions on Automatic Control}, 61\penalty0 (2), 2016.

\bibitem[Chen et~al.(2021{\natexlab{a}})Chen, Georgiou, and Pavon]{Chen21}
Chen, Y., Georgiou, T.~T., and Pavon, M.
\newblock {Stochastic Control Liaisons: Richard Sinkhorn Meets Gaspard Monge on
  a Schr\"{o}dinger Bridge}.
\newblock \emph{SIAM Review}, 63\penalty0 (2):\penalty0 249--313,
  2021{\natexlab{a}}.

\bibitem[Chen et~al.(2021{\natexlab{b}})Chen, Georgiou, and Pavon]{chen_21}
Chen, Y., Georgiou, T.~T., and Pavon, M.
\newblock {Optimal Transport in Systems and Control}.
\newblock \emph{Annual Review of Control, Robotics, and Autonomous Systems},
  4:\penalty0 89--113, 2021{\natexlab{b}}.

\bibitem[Christoffersen et~al.(2017)Christoffersen, Goyenko, Jacobs, and
  Karoui]{Illiquidity_Premia}
Christoffersen, P., Goyenko, R., Jacobs, K., and Karoui, M.
\newblock {Illiquidity Premia in the Equity Options Market}.
\newblock \emph{The Review of Financial Studies}, 2017.

\bibitem[Conforti et~al.(2023)Conforti, Durmus, and
  Greco]{sinkhorn_exp_general}
Conforti, G., Durmus, A., and Greco, G.
\newblock {Quantitative Contraction Rates for Sinkhorn Algorithm: Beyond
  Bounded Costs and Compact Marginals}.
\newblock \emph{arXiv:2304.04451}, 2023.

\bibitem[de~B{\'e}zenac et~al.(2020)de~B{\'e}zenac, Rangapuram, Benidis,
  Bohlke-Schneider, Kurle, Stella, Hasson, Gallinari, and
  Januschowski]{NIPS2020_Emmanuel}
de~B{\'e}zenac, E., Rangapuram, S.~S., Benidis, K., Bohlke-Schneider, M.,
  Kurle, R., Stella, L., Hasson, H., Gallinari, P., and Januschowski, T.
\newblock Normalizing {K}alman {F}ilters for {M}ultivariate {T}ime series
  {A}nalysis.
\newblock In \emph{Advances in Neural Information Processing Systems},
  volume~33. Curran Associates, Inc., 2020.

\bibitem[De~Bortoli et~al.(2021)De~Bortoli, Thornton, Heng, and Doucet]{DSB}
De~Bortoli, V., Thornton, J., Heng, J., and Doucet, A.
\newblock {Diffusion Schr\"{o}dinger Bridge with Applications to Score-Based
  Generative Modeling}.
\newblock In \emph{Advances in Neural Information Processing Systems
  (NeurIPS)}, 2021.

\bibitem[Deng et~al.(2021)Deng, , Brubaker, Mori, and
  Lehrmann]{Continuous_Latent_Flows}
Deng, R., , Brubaker, M.~A., Mori, G., and Lehrmann, A.~M.
\newblock {Continuous Latent Process Flows}.
\newblock In \emph{Advances in Neural Information Processing Systems
  (NeurIPS)}, 2021.

\bibitem[Deng et~al.(2020)Deng, Lin, and Liang]{CSGLD}
Deng, W., Lin, G., and Liang, F.
\newblock A {C}ontour {S}tochastic {G}radient {L}angevin {D}ynamics {A}lgorithm
  for {S}imulations of {M}ulti-modal {D}istributions.
\newblock In \emph{Advances in Neural Information Processing Systems
  (NeurIPS)}, 2020.

\bibitem[Deng et~al.(2022)Deng, Lin, and Liang]{AWSGLD}
Deng, W., Lin, G., and Liang, F.
\newblock {An Adaptively Weighted Stochastic Gradient MCMC Algorithm for Monte
  Carlo Simulation and Global Optimization}.
\newblock \emph{Statistics and Computing}, pp.\  32--58, 2022.

\bibitem[D{\"u}richen et~al.(2014)D{\"u}richen, Pimentel, Clifton, Schweikard,
  and Clifton]{multitask_GP}
D{\"u}richen, R., Pimentel, M.~A., Clifton, L., Schweikard, A., and Clifton,
  D.~A.
\newblock {Multitask Gaussian Processes for Multivariate Physiological
  Time-Series Analysis}.
\newblock \emph{IEEE Transactions on Biomedical Engineering}, 62\penalty0
  (1):\penalty0 314--322, 2014.

\bibitem[Fortuin et~al.(2020)Fortuin, Baranchuk, R\"{a}tsch, and Mandt]{GP_VAE}
Fortuin, V., Baranchuk, D., R\"{a}tsch, G., and Mandt, S.
\newblock {GP-VAE: Deep Probabilistic Time Series Imputation}.
\newblock In \emph{Proc. of the International Conference on Artificial
  Intelligence and Statistics (AISTATS)}, 2020.

\bibitem[Ghosal \& Nutz(2022)Ghosal and Nutz]{Nutz_sinkhorn_order2}
Ghosal, P. and Nutz, M.
\newblock {On the Convergence Rate of Sinkhorn's Algorithm}.
\newblock \emph{arXiv:2212.06000}, 2022.

\bibitem[Ghosal et~al.(2022)Ghosal, Nutz, and Bernton]{Nutz_22_func}
Ghosal, P., Nutz, M., and Bernton, E.
\newblock {Stability of Entropic Optimal Transport and Schr\"{o}dinger
  Bridges}.
\newblock \emph{Journal of Functional Analysis}, 283, 2022.

\bibitem[Grathwohl et~al.(2019)Grathwohl, Chen, Bettencourt, Sutskever, and
  Duvenaud]{FFJORD}
Grathwohl, W., Chen, R. T.~Q., Bettencourt, J., Sutskever, I., and Duvenaud, D.
\newblock {FFJORD: Free-form Continuous Dynamics for Scalable Reversible
  Generative Models}.
\newblock In \emph{Proc. of the International Conference on Learning
  Representation (ICLR)}, 2019.

\bibitem[Ho et~al.(2020)Ho, Jain, and Abbeel]{DDPM}
Ho, J., Jain, A., and Abbeel, P.
\newblock {Denoising Diffusion Probabilistic Models}.
\newblock In \emph{Advances in Neural Information Processing Systems
  (NeurIPS)}, 2020.

\bibitem[Hutchinson(1989)]{Hutchinson89}
Hutchinson, M.~F.
\newblock {A Stochastic Estimator of the Trace of the Influence Matrix for
  Laplacian Smoothing Splines}.
\newblock \emph{Communications in Statistics-Simulation and Computation},
  18\penalty0 (3):\penalty0 1059–1076, 1989.

\bibitem[Khrulkov \& Oseledets(2023)Khrulkov and
  Oseledets]{khrulkov2022understanding}
Khrulkov, V. and Oseledets, I.
\newblock {Understanding DDPM Latent Codes through Optimal Transport}.
\newblock In \emph{Proc. of the International Conference on Learning
  Representation (ICLR)}, 2023.

\bibitem[Kidger et~al.(2020)Kidger, Morrill, Foster, and
  Lyons]{control_neural_ode}
Kidger, P., Morrill, J., Foster, J., and Lyons, T.
\newblock {Neural Controlled Differential Equations for Irregular Time Series}.
\newblock In \emph{Advances in Neural Information Processing Systems
  (NeurIPS)}, 2020.

\bibitem[Koehler et~al.(2023)Koehler, Heckett, and
  Risteski]{stat_efficiency_SGM}
Koehler, F., Heckett, A., and Risteski, A.
\newblock {Statistical Efficiency of Score Matching: The View from
  Isoperimetry}.
\newblock In \emph{Proc. of the International Conference on Learning
  Representation (ICLR)}, 2023.

\bibitem[Kullback(1968)]{Kullback_68}
Kullback, S.
\newblock {Probability Densities with Given Marginals}.
\newblock \emph{Ann. Math. Statist.}, 1968.

\bibitem[Lai et~al.(2018)Lai, Chang, Yang, and Liu]{lai2018modeling}
Lai, G., Chang, W.-C., Yang, Y., and Liu, H.
\newblock {Modeling Long-and Short-term Temporal Patterns with Deep Neural
  Networks}.
\newblock In \emph{The 41st international ACM SIGIR conference on research \&
  development in information retrieval}, pp.\  95--104, 2018.

\bibitem[Lavenant \& Santambrogio(2022)Lavenant and
  Santambrogio]{Lavenant_Santambrogio_22}
Lavenant, H. and Santambrogio, F.
\newblock {The Flow Map of the Fokker–Planck Equation Does Not Provide
  Optimal Transport}.
\newblock \emph{Applied Mathematics Letters}, 133, 2022.

\bibitem[Lee et~al.(2022)Lee, Lu, and Tan]{lee2022convergence}
Lee, H., Lu, J., and Tan, Y.
\newblock {Convergence for Score-based Generative Modeling with Polynomial
  Complexity}.
\newblock \emph{Advances in Neural Information Processing Systems (NeurIPS)},
  2022.

\bibitem[L\'{e}onard(2014{\natexlab{a}})]{leonard_14}
L\'{e}onard, C.
\newblock {A Survey of the Schr\"{o}dinger Problem and Some of its Connections
  with Optimal Transport}.
\newblock \emph{Discrete \& Continuous Dynamical Systems-A}, 34\penalty0
  (4):\penalty0 1533–1574, 2014{\natexlab{a}}.

\bibitem[L\'{e}onard(2014{\natexlab{b}})]{leonard_14_chain_rule}
L\'{e}onard, C.
\newblock {Some Properties of Path Measures}.
\newblock \emph{{S\'{e}minaire de Probabilit\'{e}s XLVI}}, pp.\  207--230,
  2014{\natexlab{b}}.

\bibitem[Li et~al.(2020)Li, Wong, Chen, and Duvenaud]{scalable_SDE}
Li, X., Wong, T.-K.~L., Chen, R. T.~Q., and Duvenaud, D.
\newblock {Scalable Gradients for Stochastic Differential Equations}.
\newblock In \emph{Proc. of the International Conference on Artificial
  Intelligence and Statistics (AISTATS)}, 2020.

\bibitem[Luo et~al.(2019)Luo, Zhang, Cai, and Yuan]{luo2019e2gan}
Luo, Y., Zhang, Y., Cai, X., and Yuan, X.
\newblock {E2GAN: End-to-end Generative Adversarial Network for Multivariate
  time Series Imputation}.
\newblock In \emph{Proceedings of the 28th international joint conference on
  artificial intelligence}, pp.\  3094--3100. AAAI Press, 2019.

\bibitem[Ma \& Yong(2007)Ma and Yong]{Ma_FB_SDE}
Ma, J. and Yong, J.
\newblock \emph{{Forward-Backward Stochastic Differential Equations and their
  Applications}}.
\newblock Springer, 2007.

\bibitem[Malladi \& Sethian(1995)Malladi and Sethian]{Malladi_Sethian}
Malladi, R. and Sethian, J.~A.
\newblock {Image Processing via Level Set Curvature Flow}.
\newblock \emph{Proc. Natl. Acad. Sci.}, 92:\penalty0 7046--7050, 1995.

\bibitem[Mulyadi et~al.(2021)Mulyadi, Jun, and Suk]{v_rin}
Mulyadi, A.~W., Jun, E., and Suk, H.-I.
\newblock {Uncertainty-aware Variational-recurrent Imputation Network for
  Clinical Time Series}.
\newblock \emph{IEEE Transactions on Cybernetics}, 2021.

\bibitem[Nutz(2022)]{Nutz22_note}
Nutz, M.
\newblock {Introduction to Entropic Optimal Transport}.
\newblock \emph{Lecture Notes}, 2022.

\bibitem[Nutz \& Wiesel(2022)Nutz and Wiesel]{Nutz_22_a}
Nutz, M. and Wiesel, J.
\newblock {Stability of Schrödinger Potentials and Convergence of Sinkhorn's
  Algorithm}.
\newblock \emph{Annals of Probability}, 2022.

\bibitem[Pavon et~al.(2021)Pavon, Tabak, and Trigila]{Pavon_CPAM_21}
Pavon, M., Tabak, E.~G., and Trigila, G.
\newblock {The Data-driven Schr\"{o}dinger Bridge}.
\newblock \emph{Communications on Pure and Applied Mathematics}, 74:\penalty0
  1545--1573, 2021.

\bibitem[Peyr\'{e} \& Cuturi(2019)Peyr\'{e} and Cuturi]{Compute_OT}
Peyr\'{e}, G. and Cuturi, M.
\newblock \emph{{Computational Optimal Transport: With Applications to Data
  Science}}.
\newblock Foundations and Trends\@ in Machine Learning, 2019.

\bibitem[Raginsky et~al.(2017)Raginsky, Rakhlin, and Telgarsky]{Maxim17}
Raginsky, M., Rakhlin, A., and Telgarsky, M.
\newblock Non-convex {L}earning via {S}tochastic {G}radient {L}angevin
  {D}ynamics: a {N}onasymptotic {A}nalysis.
\newblock In \emph{Proc. of Conference on Learning Theory (COLT)}, June 2017.

\bibitem[Rasul et~al.(2021)Rasul, Seward, Schuster, and
  Vollgraf]{rasul2021autoregressive}
Rasul, K., Seward, C., Schuster, I., and Vollgraf, R.
\newblock {Autoregressive Denoising Diffusion Models for Multivariate
  Probabilistic Time Series Forecasting}.
\newblock In \emph{International Conference on Machine Learning}, pp.\
  8857--8868. PMLR, 2021.

\bibitem[Richter-Powell et~al.(2022)Richter-Powell, Lipman, and
  Chen]{Chen_divergence_free}
Richter-Powell, J., Lipman, Y., and Chen, R. T.~Q.
\newblock {Neural Conservation Laws: A Divergence-Free Perspective}.
\newblock In \emph{Advances in Neural Information Processing Systems
  (NeurIPS)}, 2022.

\bibitem[Ronneberger et~al.(2015)Ronneberger, Fischer, and
  Brox]{ronneberger2015u}
Ronneberger, O., Fischer, P., and Brox, T.
\newblock {U-Net: Convolutional Networks for Biomedical Image Segmentation}.
\newblock In \emph{International Conference on Medical image computing and
  computer-assisted intervention}. Springer, 2015.

\bibitem[Rubanova et~al.(2019)Rubanova, Chen, and Duvenaud]{latent_ode}
Rubanova, Y., Chen, R. T.~Q., and Duvenaud, D.
\newblock {Latent ODEs for Irregularly-Sampled Time Series}.
\newblock In \emph{Advances in Neural Information Processing Systems
  (NeurIPS)}, 2019.

\bibitem[Ruschendorf(1995)]{IPF_95}
Ruschendorf, L.
\newblock {Convergence of the Iterative Proportional Fitting Procedure}.
\newblock \emph{Annals of Statistics}, 1995.

\bibitem[Salinas et~al.(2019)Salinas, Bohlke-Schneider, Callot, Medico, and
  Gasthaus]{salinas2019high}
Salinas, D., Bohlke-Schneider, M., Callot, L., Medico, R., and Gasthaus, J.
\newblock {High-dimensional Multivariate Forecasting with Low Rank Gaussian
  Copula Processes}.
\newblock \emph{Advances in neural information processing systems}, 32, 2019.

\bibitem[Shi et~al.(2022)Shi, De~Bortoli, Deligiannidis, and
  Doucet]{Conditional_DSB}
Shi, Y., De~Bortoli, V., Deligiannidis, G., and Doucet, A.
\newblock {Conditional Simulation Using Diffusion Schr\"{o}dinger Bridges}.
\newblock In \emph{Proc. of the Conference on Uncertainty in Artificial
  Intelligence (UAI)}, 2022.

\bibitem[Shukla \& Marlin(2021)Shukla and Marlin]{attention_ts}
Shukla, S.~N. and Marlin, B.~M.
\newblock {Multi-time Attention Networks for Irregularly Sampled Time Series}.
\newblock In \emph{Proc. of the International Conference on Learning
  Representation (ICLR)}, 2021.

\bibitem[Silva et~al.(2012)Silva, Moody, Scott, Celi, and Mark]{phy_2012}
Silva, I., Moody, G., Scott, D.~J., Celi, L.~A., and Mark, R.~G.
\newblock {Predicting Inhospital Mortality of ICU patients: The
  Physionet/Computing in Cardiology Challenge 2012}.
\newblock In \emph{Computing in Cardiology}, pp.\  245–248. IEEE, 2012.

\bibitem[Sohl-Dickstein et~al.(2015)Sohl-Dickstein, Weiss, Maheswaranathan, and
  Ganguli]{nonequilibrium_thermodynamics_15}
Sohl-Dickstein, J., Weiss, E.~A., Maheswaranathan, N., and Ganguli, S.
\newblock {Deep Unsupervised Learning using Nonequilibrium Thermodynamics}.
\newblock In \emph{Proc. of the International Conference on Machine Learning
  (ICML)}, 2015.

\bibitem[Song \& Ermon(2019)Song and Ermon]{song2019generative}
Song, Y. and Ermon, S.
\newblock {Generative Modeling by Estimating Gradients of The Data
  Distribution}.
\newblock \emph{Advances in Neural Information Processing Systems}, 32, 2019.

\bibitem[Song et~al.(2021{\natexlab{a}})Song, Durkan, Murray, and
  Ermon]{song_likelihood_training}
Song, Y., Durkan, C., Murray, I., and Ermon, S.
\newblock {Maximum Likelihood Training of Score-Based Diffusion Models }.
\newblock In \emph{Advances in Neural Information Processing Systems
  (NeurIPS)}, 2021{\natexlab{a}}.

\bibitem[Song et~al.(2021{\natexlab{b}})Song, Sohl-Dickstein, Kingma, Kumar,
  Ermon, and Poole]{score_sde}
Song, Y., Sohl-Dickstein, J., Kingma, D.~P., Kumar, A., Ermon, S., and Poole,
  B.
\newblock {Score-Based Generative Modeling through Stochastic Differential
  Equations }.
\newblock In \emph{Proc. of the International Conference on Learning
  Representation (ICLR)}, 2021{\natexlab{b}}.

\bibitem[Tashiro et~al.(2021)Tashiro, Song, Song, and Ermon]{CSDI}
Tashiro, Y., Song, J., Song, Y., and Ermon, S.
\newblock {CSDI: Conditional Score-based Diffusion Models for Probabilistic
  Time Series Imputation}.
\newblock In \emph{Advances in Neural Information Processing Systems
  (NeurIPS)}, 2021.

\bibitem[Vargas et~al.(2021)Vargas, Thodoroff, Lamacraft, and
  Lawrence]{SBP_max_llk}
Vargas, F., Thodoroff, P., Lamacraft, A., and Lawrence, N.
\newblock {Solving Schr{\"o}dinger Bridges via Maximum Likelihood}.
\newblock \emph{Entropy}, 23\penalty0 (9):\penalty0 1134, 2021.

\bibitem[Wang et~al.(2021)Wang, Jiao, Xu, Wang, and Yang]{gefei_21}
Wang, G., Jiao, Y., Xu, Q., Wang, Y., and Yang, C.
\newblock {Deep Generative Learning via Schr\"{o}dinger Bridge}.
\newblock In \emph{Proc. of the International Conference on Machine Learning
  (ICML)}, 2021.

\bibitem[Xiong \& Pelger(2023)Xiong and
  Pelger]{Large_Dimensional_Latent_Factor}
Xiong, R. and Pelger, M.
\newblock {Large Dimensional Latent Factor Modeling with Missing Observations
  and Applications to Causal Inference}.
\newblock \emph{Journal of Econometrics}, 233:\penalty0 271--301, 2023.

\bibitem[Zheng et~al.(2013)Zheng, Liu, and Hsieh]{U_Air}
Zheng, Y., Liu, F., and Hsieh, H.-P.
\newblock {U-Air: When Urban Air Quality Inference Meets Big Data}.
\newblock In \emph{Proceedings of the 19th SIGKDD conference on Knowledge
  Discovery and Data Mining (KDD'13)}, 2013.

\end{thebibliography}
